\documentclass{article}

\PassOptionsToPackage{numbers, compress}{natbib}

\usepackage[final]{neurips_2020}

\usepackage[utf8]{inputenc} 
\usepackage[T1]{fontenc}    
\usepackage[colorlinks=true,citecolor=blue]{hyperref}
\usepackage{url}            
\usepackage{booktabs}       
\usepackage{amsfonts}       
\usepackage{nicefrac}       
\usepackage{microtype}      

\usepackage{graphicx}
\usepackage{subfigure}
\usepackage{amsthm}

\usepackage{amsmath}
\usepackage{amssymb}
\usepackage{xcolor}
\usepackage{mathtools}
\usepackage{sidecap}
\usepackage{wrapfig}
\usepackage[ruled,vlined,noend]{algorithm2e}

\usepackage{enumitem}
\setlist[itemize,enumerate]{leftmargin=*}

\usepackage{mdframed}
\definecolor{theoremcolor}{rgb}{0.94, 0.94, 0.94}
\definecolor{examplecolor}{rgb}{1, 1, 1.0}
\mdfsetup{
    backgroundcolor=theoremcolor,
    linewidth=0pt,
}

\usepackage{xfrac}
\usepackage{comment}

\newmdtheoremenv[linewidth=0pt,innerleftmargin=4pt,innerrightmargin=4pt]{definition}{Definition}
\newmdtheoremenv[linewidth=0pt,innerleftmargin=4pt,innerrightmargin=4pt]{proposition}{Proposition}
\newmdtheoremenv[linewidth=0pt,innerleftmargin=0pt,innerrightmargin=0pt,backgroundcolor=examplecolor]{example}{Example}
\newmdtheoremenv{corollary}{Corollary}
\newmdtheoremenv{theorem}{Theorem}
\newmdtheoremenv{lemma}{Lemma}

\DeclareMathOperator*{\argmax}{arg\,max}

\newcommand{\andre}[1]{{\textcolor{blue}{\bf [{\sc Andre:} #1]}}}

\newcommand{\remove}[1]{}

\title{Sparse and Continuous Attention Mechanisms}

\usepackage{soul}
\usepackage{marvosym}
\newcommand\markUnbabel{\Cancer}
\newcommand\markIT{\Leo}
\newcommand\markISR{\Jupiter}
\newcommand\markLUMLIS{\Saturn}
\newcommand\markUvA{\Virgo}

\author{%
  Andr\'e F.~T.~Martins\textsuperscript{\markIT,\markLUMLIS,\markUnbabel} 
  \And
  Ant\'onio Farinhas\textsuperscript{\markIT} 
  \And
  Marcos Treviso\textsuperscript{\markIT} 
  \AND
  Vlad Niculae\textsuperscript{\markUvA,\markIT}
  \And
  Pedro M.~Q.~Aguiar\textsuperscript{\markISR,\markLUMLIS} 
  \And
  M\'ario A.~T.~Figueiredo\textsuperscript{\markIT,\markLUMLIS} 
  \AND \\[-4ex]
  \{
  {\tt andre.t.martins}, 
  {\tt marcos.treviso}, 
  {\tt antonio.farinhas}, 
  {\tt mario.figueiredo}\}\\{\tt @tecnico.ulisboa.pt}, 
  {\tt aguiar@isr.ist.utl.pt}, 
  {\tt vlad@vene.ro}\\[2ex]
  \textsuperscript{\markIT{}}Instituto de Telecomunica\c{c}\~oes,  Instituto Superior Técnico, Lisbon, Portugal \\
  \textsuperscript{\markISR{}}Instituto de Sistemas e Rob\'otica, Instituto Superior Técnico, Lisbon, Portugal \\
  \textsuperscript{\markLUMLIS{}}LUMLIS (Lisbon ELLIS Unit), Lisbon, Portugal \\
  \textsuperscript{\markUvA{}}Informatics Institute, University of Amsterdam, The Netherlands \\
  \textsuperscript{\markUnbabel{}}Unbabel, Lisbon, Portugal
}

\begin{document}
\maketitle
\begin{abstract}
Exponential families are widely used in machine learning; they include many distributions in continuous and discrete domains (\textit{e.g.}, Gaussian, Dirichlet, Poisson, and categorical distributions via the softmax transformation). Distributions in each of these families have fixed support. In contrast, for finite domains, there has been recent work on sparse alternatives to softmax (\textit{e.g.} sparsemax and $\alpha$-entmax), which have varying support, being able to assign zero probability to irrelevant categories.
This paper expands that work in two directions: first, we extend $\alpha$-entmax\remove{and Fenchel-Young (FY) losses} to continuous domains, revealing a link with Tsallis statistics and deformed exponential families. 
Second, we introduce continuous-domain attention mechanisms, deriving efficient gradient backpropagation algorithms for $\alpha \in \{1,2\}$.
Experiments on  attention-based text classification, machine translation, and visual question answering  illustrate the use of continuous attention in 1D and 2D, showing that it allows attending to time intervals and compact regions.\remove{and FY-based interval regression on weather data.}
\end{abstract}

\section{Introduction}

Exponential families are ubiquitous in statistics and machine learning \citep{brown1986fundamentals,barndorff2014information}. 
They enjoy many useful properties, such as the existence of conjugate priors (crucial in Bayesian inference) and 
the classical Pitman-Koopman-Darmois theorem \citep{pitman1936sufficient,darmois1935lois,koopman1936distributions}, which states that,
among families with {\bf fixed support} (independent of the parameters), exponential families are the only having sufficient statistics of fixed dimension for any number of i.i.d.\ samples.

Departing from exponential families, there has been recent work on discrete, finite-domain distributions with {\bf varying and sparse support},  via the \textit{sparsemax} and the \textit{entmax} transformations \citep{Martins2016ICML,blondel2020learning,peters2019sparse}. 
Those approaches drop the link to exponential families of categorical distributions provided by the softmax transformation, which always yields dense probability mass functions. In contrast, sparsemax and entmax can lead to sparse distributions, whose support is not constant throughout the family. This property has been used \remove{to obtain new loss functions for multi-label classification and sparse sequence-to-sequence problems, as well as} to design sparse attention mechanisms with improved interpretability \citep{peters2019sparse,correia2019adaptively}.

However, {sparsemax} and {entmax} are so far limited to discrete domains. Can a similar approach be extended to continuous domains? This paper provides that extension and pinpoints a connection with ``deformed exponential families'' \citep{naudts2009q,sears2010generalized,ding2010t} and Tsallis statistics \citep{Tsallis1988}, leading to {\bf $\alpha$-sparse families} (\S\ref{sec:sparse_families}). 
We use this construction to obtain new density families with varying support, 
including the {\it truncated parabola} and {\it paraboloid} distributions (2-sparse counterpart of the Gaussian, \S\ref{sec:sparsemax} and Fig.~\ref{fig:gaussian_paraboloid}). 

Softmax and its variants are widely used in \textit{attention mechanisms}, an important component of neural networks \citep{bahdanau2014neural}. Attention-based neural networks can ``attend'' to finite sets of objects and identify relevant features. 
We use our extension above to devise new {\bf continuous attention mechanisms} (\S\ref{sec:attention}), which can attend to continuous data streams and to domains that are inherently continuous, such as images. Unlike traditional attention mechanisms, ours are suitable for selecting compact regions, such as 1D-segments or 2D-ellipses. We show that the Jacobian of these transformations are generalized covariances, and we use this fact to obtain efficient backpropagation algorithms (\S\ref{sec:jacobian}). 

As a proof of concept, we apply our models with continuous attention to text classification, machine translation, and visual question answering tasks, with encouraging results (\S\ref{sec:experiments}).

\remove{We also generalize the recent Fenchel-Young (FY) losses \citep{blondel2020learning} 
to arbitrary domains, illustrating their usefulness by estimating sparse continuous densities for regression problems under bounded noise \citep{d2013bounded}. We show that properties of FY losses hold for general $\alpha$-sparse families, including convexity and closed-form gradient w.r.t.\ their canonical parameters. We use this to perform interval regression, which returns mean estimates and intervals, based on the support of their distribution.} 

\paragraph{Notation.}
Let $(S, \mathcal{A}, \nu)$ be a measure space, 
where $S$ is a set, $\mathcal{A}$ is a $\sigma$-algebra, and $\nu$ is a measure. 
We denote by $\mathcal{M}_+^1(S)$ the set of $\nu$-absolutely continuous probability measures. From the Radon-Nikodym theorem \citep[\S31]{halmos2013measure}, each element of $\mathcal{M}_+^1(S)$ is identified (up to equivalence within measure zero) with a probability density function $p: S \rightarrow \mathbb{R}_+$, with $\int_S p(t)\, d\nu(t) = 1$. 
For convenience, we often drop $d\nu(t)$ from the integral. 
We denote the measure of $A\in  \mathcal{A}$ as
$|A| = \nu(A) = \int_{A} 1$, and  
the support of a density $p \in \mathcal{M}_+^1(S)$ as $\mathrm{supp}(p) = \{t \in S \mid p(t) > 0\}$. 
Given $\phi:S \rightarrow \mathbb{R}^m$, 
we write expectations 
as
$\mathbb{E}_p[\phi(t)] := \int_S p(t) \, \phi(t)$.  
Finally, we define $[a]_+ := \max\{a, 0\}$. 

\begin{figure*}[t]
\centering
\includegraphics[width=.27\textwidth]{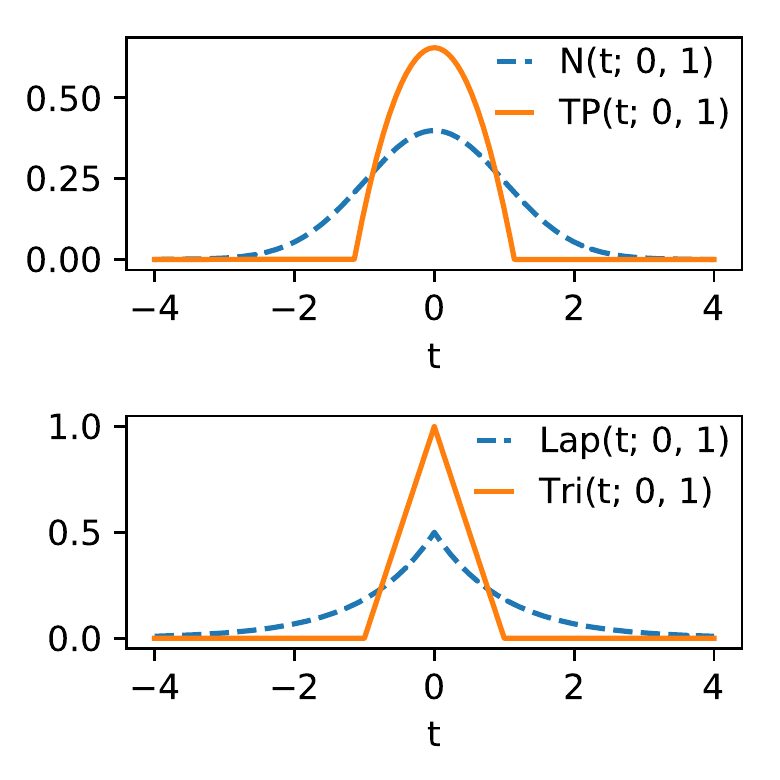}\quad
\includegraphics[width=.33\textwidth]{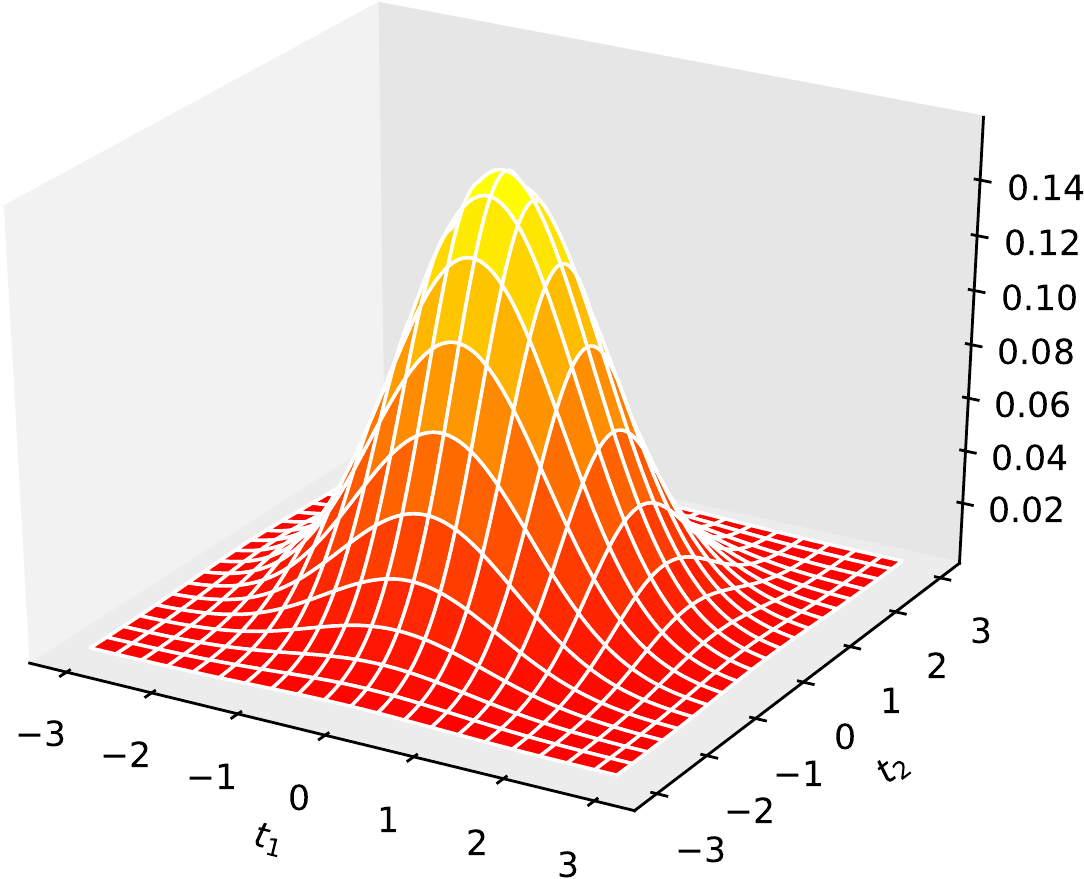}%
\includegraphics[width=.33\textwidth]{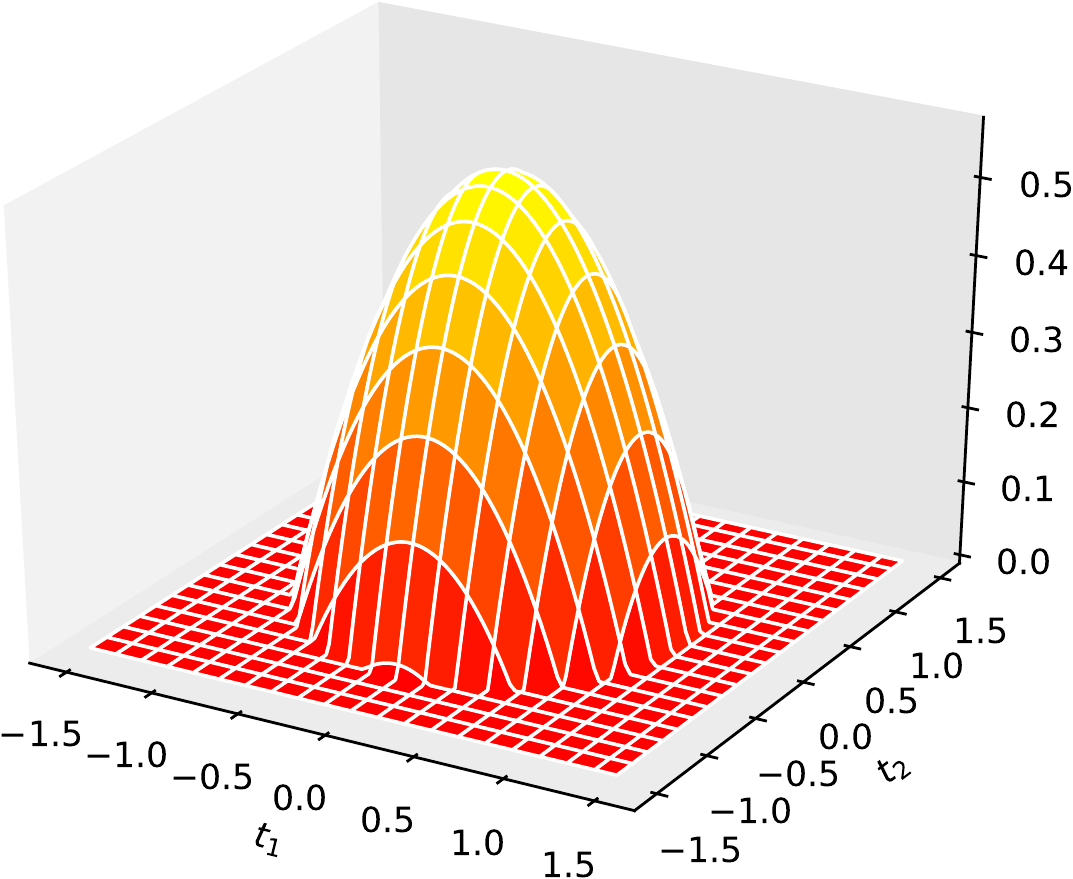}%
\caption{{\bf 1D and 2D distributions generated by the $\Omega_\alpha$-RPM for $\alpha\in \{1,2\}$.} Left: Univariate location-scale families, including Gaussian and truncated parabola (top) and Laplace and triangular (bottom). Middle and right: Bivariate Gaussian $\mathcal{N}(t; 0, I)$ and truncated paraboloid $\mathrm{TP}(t; 0, I)$. \label{fig:gaussian_paraboloid}}
\end{figure*}

\remove{
\begin{figure*}[t]
\centering \includegraphics[width=.8\textwidth]{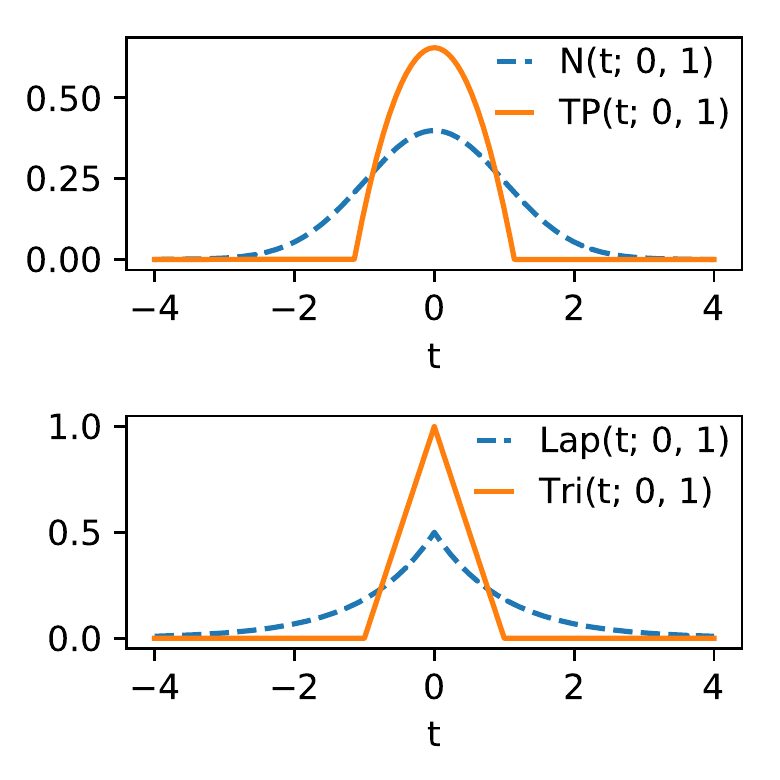}
\caption{{\bf Location-scale families with different $\alpha$-Tsallis regularizations.} Left: Gaussian and truncated parabola distributions with zero mean and different variances. Middle: same for Laplace and triangular distributions. Right: Distributions $\hat{p}_{\Omega_\alpha}[-t^2/2]$ for varying $\alpha$. \label{fig:gaussian_laplacian_distributions}}
\end{figure*}
}

\section{Sparse Families}\label{sec:sparse_families}

In this section, we provide background on exponential families and its generalization through Tsallis statistics. We link these concepts, 
studied in statistical physics, 
to sparse alternatives to softmax recently 
proposed in the machine learning literature \citep{Martins2016ICML,peters2019sparse}, extending 
the latter to continuous domains. 

\subsection{Regularized prediction maps ($\Omega$-RPM)}

Our starting point is the notion of $\Omega$-regularized prediction maps, introduced by \citet{blondel2020learning} for finite domains $S$. 
This is a general framework for mapping vectors in $\mathbb{R}^{|S|}$ 
(\textit{e.g.}, label scores computed by a neural network) into probability vectors in $\triangle^{|S|}$ (the simplex), with a regularizer $\Omega$ encouraging uniform distributions. Particular choices of $\Omega$ recover 
argmax, softmax \citep{bridle1990probabilistic}, and sparsemax \citep{Martins2016ICML}. 
Our definition below extends this framework to arbitrary measure spaces $\mathcal{M}_+^1(S)$, where we assume  $\Omega: \mathcal{M}_+^1(S) \rightarrow \mathbb{R}$ is a lower semi-continuous,  proper, and strictly convex function.

\vspace{0.1cm}
\begin{definition}
\label{def:regularized_prediction}
The $\Omega$-\textbf{regularized prediction map} ($\Omega$-RPM) $\hat{p}_{\Omega}: \mathcal{F} \rightarrow \mathcal{M}_+^1(S)$ is defined as
\begin{equation}\label{eq:reg_prediction}
\hat{p}_{\Omega}[f] = \argmax_{p \in \mathcal{M}_+^1(S)} \mathbb{E}_{p}[f(t)] - \Omega(p),
\end{equation}
where $\mathcal{F}$ is the set of functions for which the maximizer above exists and is unique.
\end{definition}

It is often convenient to consider a ``temperature parameter'' $\tau>0$, absorbed into $\Omega$ via $\Omega := \tau \tilde{\Omega}$. If $f$ has a unique global maximizer $t^\star$, the low-temperature limit yields ${ \lim_{\tau\rightarrow 0}}\, \hat{p}_{\tau \tilde{\Omega}}[f] = \delta_{t^\star}$,
a Dirac delta distribution at the maximizer of $f$. 
For finite $S$, this is the {\it argmax} transformation shown in \citep{blondel2020learning}. 
Other interesting examples of regularization functionals are shown in the next subsections.

\subsection{Shannon's negentropy and exponential families}

A natural choice of regularizer is the {Shannon's negentropy}, $\Omega(p) = \int_S p(t)\log p(t)$. 
In this case, if we interpret $-f(t)$ as an energy function, the $\Omega$-RPM corresponds to the well-known {\it free energy variational principle}, leading to Boltzmann-Gibbs distributions (\citep{cover2012elements}; see App.~\ref{sec:diff_ent_exp_family}): 
\begin{equation}\label{eq:boltzmann}
\hat{p}_\Omega[f](t) = \frac{\exp(f(t))}{\int_S \exp(f(t'))d\nu(t')} = \exp \bigl(f(t) - A(f)\bigr),
\end{equation}
where 
$A(f) := \log \int_S \exp(f(t))$
is the log-partition function. 
If $S$ is finite and $\nu$ is the counting measure, the integral in \eqref{eq:boltzmann} is a summation and we can write $f$ as a vector $[f_1, \ldots, f_{|S|}] \in \mathbb{R}^{|S|}$. In this case, the $\Omega$-RPM is the {\it softmax transformation},
\begin{equation}
\hat{p}_\Omega[f] = \mathrm{softmax}(f) = \tfrac{\exp(f)}{\sum_{k=1}^{|S|} \exp(f_k)} \in \triangle^{|S|}.
\end{equation}
If $S=\mathbb{R}^N$, $\nu$ is the Lebesgue measure, and 
$f(t) = -\sfrac{1}{2} (t-\mu)^\top\Sigma^{-1}(t-\mu)$ for $\mu \in \mathbb{R}^N$ and $\Sigma \succ 0$ (i.e., $\Sigma$ is a positive definite matrix), 
we obtain a {\it multivariate Gaussian}, $\hat{p}_{\Omega}[f](t) = \mathcal{N}(t; \mu, \Sigma)$. 
\remove{with Shannon negentropy $\Omega(p) = -\sfrac{1}{2}\log \det (2\pi e \Sigma)$.} 
This becomes a univariate Gaussian $\mathcal{N}(t; \mu, \sigma^2)$ if $N=1$. 
For $S=\mathbb{R}$ and defining $f(t) = -|t-\mu|/b$, with  $\mu \in \mathbb{R}$ and $b>0$, we get a {\it Laplace} density, $\hat{p}_{\Omega}[f](t) = \tfrac{1}{2b}\exp\left(-|t-\mu|/b\right)$. \remove{, with Shannon negentropy $\Omega(p) = -\log(2b e)$.}


\paragraph{Exponential families.} 
Let $f_{\theta}(t) = \theta^\top \phi(t)$, where $\phi(t) \in \mathbb{R}^M$ is a vector of \textit{statistics} and $\theta \in \Theta \subseteq \mathbb{R}^M$ is a vector of \textit{canonical parameters}. 
A family of the form \eqref{eq:boltzmann} parametrized by $\theta \in \Theta \subseteq \mathbb{R}^M$ is called an {\it exponential family} \citep{barndorff2014information}. 
Exponential families 
have many appealing properties, such as the existence of conjugate priors and sufficient statistics, and a dually flat geometric structure \citep{amari2016information}. Many well-known distributions are exponential families, including the categorical and Gaussian distributions above, and Laplace distributions with a fixed $\mu$. 
A key property of exponential families is that the support is constant within the same family and dictated by the base measure $\nu$: this follows immediately from the positiveness of the $\exp$ function in \eqref{eq:boltzmann}. We abandon this property in the sequel.

\subsection{Tsallis' entropies and $\alpha$-sparse families}

Motivated by applications in statistical physics,
\citet{Tsallis1988} proposed a generalization of Shannon's negentropy. 
This generalization is rooted on the notions of  $\beta$-logarithm, $\log_{\beta}: \mathbb{R}_{\ge 0} \rightarrow \mathbb{R}$ (not to be confused with base-$\beta$ logarithm), and $\beta$-exponential, $\exp_{\beta}: \mathbb{R} \rightarrow \mathbb{R}$: 
\begin{equation}\label{eq:beta_log_exp}
    \log_{\beta}(u) := \left\{
    \begin{array}{ll}
        \frac{u^{1-\beta} - 1}{1-\beta}, & \beta \ne 1 \\
        \log u, & \beta = 1;
    \end{array}
    \right.
    \qquad
    \exp_{\beta}(u) := \left\{
    \begin{array}{ll}
     	[1 + (1-\beta)u]_+^{1/(1-\beta)}, & \beta \ne 1 \\
        \exp u, & \beta = 1.
    \end{array}
    \right.    
\end{equation}
Note that ${\lim_{\beta \rightarrow 1}}\log_\beta(u) = \log u$, ${\lim_{\beta \rightarrow 1}}\exp_\beta(u) = \exp u$, and $\log_\beta(\exp_\beta(u)) = u$ for any $\beta$. 
Another important concept is that of ``$\beta$-escort distribution''  \citep{Tsallis1988}: this is the distribution $\tilde{p}^{\beta}$ given by
\begin{equation}\label{eq:escort}
\tilde{p}^{\beta}(t) := \frac{p(t)^{\beta}}{\|p\|_\beta^\beta}, \quad \text{where $\|p\|_\beta^\beta = \int_{S} p(t')^{\beta} d\nu(t')$}.
\end{equation}
Note that we have $\tilde{p}^{1}(t) = p(t)$. 

The {\bf $\alpha$-Tsallis negentropy}  \citep{havrda1967quantification,Tsallis1988} is defined as:%
\footnote{This entropy is normally defined up to a constant, often presented without the $\tfrac{1}{\alpha}$ factor. We use the same definition as \citet[\S 4.3]{blondel2020learning} for convenience.} %
\begin{equation}\label{eq:tsallis}
\Omega_{\alpha}(p) := \tfrac{1}{\alpha} \mathbb{E}_{p}[\log_{2-\alpha}(p(t))] = 
\begin{cases}
\frac{1}{\alpha(\alpha-1)}\left( \int_S  p(t)^\alpha - 1 \right), &
\alpha \ne 1,\\
\int_S p(t) \log p(t), &
\alpha = 1.
\end{cases}
\end{equation}
Note that $\lim_{\alpha \rightarrow
1}\Omega_{\alpha}(p) = \Omega_{1}(p)$, for any $p \in \mathcal{M}_+^1(S)$, 
with $\Omega_{1}(p)$ recovering Shannon's negentropy (proof in App.~\ref{sec:gini_ent_sparse_family}). 
Another notable case is $\Omega_2(p) = \sfrac{1}{2}\int_S p(t)^2  - \sfrac{1}{2}$,
the negative of which is called the Gini-Simpson index \citep{Jost2006,Rao1982}. 
We come back to the $\alpha=2$ case in \S\ref{sec:sparsemax}.

\remove{
This family is continuous in $\alpha$, \textit{i.e.}, $\lim_{\alpha \rightarrow
1}\Omega_{\alpha}(p) = \Omega_{1}(p)$, for any $p \in \mathcal{M}_+^1(S)$, with $\Omega_{1}(p)$ recovering Shannon's negentropy (proof in App.~\ref{sec:gini_ent_sparse_family}). 
Another notable case is $\Omega_2(p) = \sfrac{1}{2}\int_S p(t)^2  - \sfrac{1}{2}$,
the negative of which has several names, {\it e.g.}, Gini-Simpson index \citep{Jost2006} or Rao's quadratic entropy \citep{Rao1982}. 
We will come back to the $\alpha=2$ case in \S\ref{sec:sparsemax}.
}

For $\alpha>0$, $\Omega_\alpha$ is strictly convex, hence it can be plugged in as the regularizer in  Def.~\ref{def:regularized_prediction}. 
The next proposition (\citep{naudts2009q}; proof in  App.~\ref{sec:gini_ent_sparse_family}) provides an expression for $\Omega_{\alpha}$-RPM using the $\beta$-exponential  \eqref{eq:beta_log_exp}: 

\vspace{0.15cm}
\begin{proposition}\label{prop:solution_rpm_tsallis}
For $\alpha > 0$ and $f\in\mathcal{F}$, 
\begin{equation}\label{eq:entmax}
\hat{p}_{\Omega_\alpha}[f](t) = \exp_{2-\alpha}(f(t) - A_\alpha(f)), 
\end{equation}
where 
$A_\alpha: \mathcal{F} \rightarrow \mathbb{R}$ is a normalizing function:
$A_\alpha(f) = \frac{\frac{1}{1-\alpha} + \int_S p_\theta(t)^{2-\alpha} f(t)}{\int_S p_\theta(t)^{2-\alpha}} - \frac{1}{1-\alpha}.$
\end{proposition}

Let us contrast \eqref{eq:entmax} with Boltzmann-Gibbs distributions \eqref{eq:boltzmann}, recovered with $\alpha=1$. One key thing to note is that the $(2-\alpha)$-exponential, for $\alpha > 1$, can return zero values. Therefore, the distribution $\hat{p}_{\Omega_\alpha}[f]$ in \eqref{eq:entmax} may not have full support, \textit{i.e.}, we may have $\mathrm{supp}(\hat{p}_{\Omega_\alpha}[f]) \subsetneq S$.  
We say that $\hat{p}_{\Omega_\alpha}[f]$ has {\it sparse support} if $\nu(S \setminus \mathrm{supp}(\hat{p}_{\Omega_\alpha}[f])) > 0$.%
\footnote{This should not be confused with sparsity-inducing distributions \citep{FigueiredoNIPS2001,TippingJMLR2001}.} %
This generalizes the notion of sparse vectors.

\paragraph{Relation to sparsemax and entmax.} 
\citet{blondel2020learning}  showed that, for finite $S$,  $\Omega_2$-RPM is the {\bf sparsemax} transformation, $\hat{p}_\Omega[f] = \mathrm{sparsemax}(f) =
\arg\min_{p \in \triangle^{|S|}} \|p - f\|^2$. 
\remove{(Euclidean projection of  $f\in\mathbb{R}^{|S|}$ onto the $|S|$-dimensional probability simplex $\triangle^{|S|}$).} 
Other values of $\alpha$ were studied by \citet{peters2019sparse}, under the name {\bf $\alpha$-entmax} transformation. For $\alpha>1$, these transformations have a propensity for returning sparse distributions, where several entries have zero probability. 
Proposition~\ref{prop:solution_rpm_tsallis} shows that similar properties can be obtained when $S$ is continuous.

\paragraph{Deformed exponential families.} 
With a linear parametrization $f_\theta(t) = \theta^\top\phi(t)$, 
distributions with the form \eqref{eq:entmax} are called \textit{deformed exponential} or \textit{$q$-exponential families} \citep{naudts2009q,sears2010generalized,ding2010t,matsuzoe2012geometry}. 
\remove{
also referred to as a \textit{$t$-exponential family} \citep{ding2010t} and a \textit{$q$-exponential family} \citep{matsuzoe2012geometry}. 
}
The geometry of these families induced by the Tsallis $q$-entropy was studied by \citet[\S 4.3]{amari2016information}.%
\footnote{Unfortunately, the literature is inconsistent in defining these coefficients. Our $\alpha$ matches that of \citet{blondel2020learning}; Tsallis' $q$ equals $2-\alpha$; this family is also related to Amari's $\alpha$-divergences, but their $\alpha = 2q - 1$. Inconsistent definitions have also been proposed for $q$-exponential families regarding how they are normalized; for example, the Tsallis maxent principle leads to a different definition. See App.~\ref{sec:tsallis_maxent} for a detailed discussion.} 
Unlike those prior works, we are interested in the sparse, light tail scenario ($\alpha>1$), not in heavy tails. For $\alpha>1$, we call these {\bf $\alpha$-sparse families.} 
When $\alpha \rightarrow 1$, $\alpha$-sparse families become exponential families and they cease to be ``sparse'', in the sense that all distributions in the same family have the same support.

A relevant problem is that of characterizing $A_\alpha(\theta)$. When $\alpha=1$, 
$A_1(\theta) = {\lim_{\alpha\rightarrow 1}} A_\alpha(\theta) = \log \int_S \exp(\theta^\top \phi(t))$ is the log-partition function (see \eqref{eq:boltzmann}), and  its first and higher order derivatives are equal to the moments of the sufficient statistics. The following proposition (stated as \citet[Theorem~5]{amari2011geometry}, and proved in our App.~\ref{sec:proof_gradient_A}) 
characterizes $A_\alpha(\theta)$ for $\alpha \ne 1$ in terms of an expectation under the $\beta$-escort distribution for $\beta = 2-\alpha$ (see \eqref{eq:escort}). This proposition will be used later to derive the Jacobian of entmax attention mechanisms. \remove{and the gradient and Hessian of Fenchel-Young losses.}

\vspace{0.3cm}

\begin{proposition}\label{prop:gradient_A}
$A_\alpha(\theta)$ is a convex function and its gradient is given by
\begin{equation}\label{eq:derivative_of_partition}
    \nabla_\theta A_\alpha(\theta) = \mathbb{E}_{\tilde{p}_\theta^{2-\alpha}}[\phi(t)] = \frac{\int_{S} p_\theta(t)^{2-\alpha} \phi(t)}{\int_{S} p_\theta(t)^{2-\alpha}}.
\end{equation}
\end{proposition}


\subsection{The 2-Tsallis entropy: sparsemax}\label{sec:sparsemax}

In this paper, we focus on the case $\alpha=2$. For finite $S$, this corresponds to the sparsemax transfomation proposed by \citet{Martins2016ICML}, 
which has appealing  theoretical and computational properties. 
In the general case, 
plugging $\alpha=2$ in \eqref{eq:entmax} leads to the $\Omega_2$-RPM,
\begin{equation}\label{eq:sparsemax}
\hat{p}_{\Omega_2}[f](t) = [f(t) - \lambda]_+, 
\qquad \text{where $\lambda = A_2(f) - 1$,}
\end{equation}
\textit{i.e.}, $\hat{p}_{\Omega_2}[f]$ is obtained from $f$ by subtracting a constant (which may be negative) and truncating, where that constant $\lambda$ must be such that $ \int_S [f(t) - \lambda]_+ = 1$. 
\remove{For the discrete case, this constant has been called the \textit{threshold function} by \citet{Martins2016ICML}.}

If $S$ is continuous and $\nu$ the Lebesgue measure, we call $\Omega_2$-RPM the {\bf continuous sparsemax} transformation. Examples follow, 
some of which correspond to  novel distributions.

\paragraph{Truncated parabola.} 
If $f(t) = -\frac{(t-\mu)^2}{2\sigma^2}$, we obtain the continuous sparsemax counterpart of a Gaussian, which we dub a ``truncated parabola'':
\begin{equation}\label{eq:truncated_parabola}
\hat{p}_{\Omega_2}[f](t) 
= \left[ - \tfrac{(t-\mu)^2}{2\sigma^2} - \lambda \right]_+
\eqqcolon \mathrm{TP}(t; \mu, \sigma^2), 
\end{equation}
where $\lambda = -\tfrac{1}{2}\bigl(3/(2\sigma)\bigr)^{2/3}$ (see App.~\ref{sec:proof_truncated_parabola}). 
\remove{
, 
$\mathrm{supp}(\hat{p}_{\Omega_2}[f]) = [\mu - \tfrac{3}{-4\, \lambda}, \mu + \tfrac{3}{-4\, \lambda} ]$ and 
$\Omega_2(\hat{p}_{\Omega_2}[f]) = -\tfrac{1}{2} - \tfrac{2\lambda}{5}$.
}
This function, depicted in Fig.~\ref{fig:gaussian_paraboloid} (top left), is widely used in density estimation. For $\mu=0$ and $\sigma=\sqrt{2/3}$, it is known as the Epanechnikov kernel \citep{epanechnikov1969non}.

\paragraph{Truncated paraboloid.} 
The previous example can be generalized to $S=\mathbb{R}^N$, with $f(t)=-\frac{1}{2}(t-\mu)^\top\Sigma^{-1}(t-\mu)$, where $\Sigma \succ 0$, leading to a ``multivariate truncated paraboloid,'' the sparsemax counterpart of the multivariate Gaussian (see middle and rightmost plots in Fig.~\ref{fig:gaussian_paraboloid}):
\begin{equation}\label{eq:truncated_paraboloid}
\hat{p}_{\Omega_2}[f](t) = \bigl[-\lambda - \tfrac{1}{2}(t-\mu)\Sigma^{-1}(t-\mu)\bigr]_+, 
\qquad \text{where $\lambda = -\Bigl(\Gamma\bigl(\tfrac{N}{2} + 2\bigr) / \sqrt{\mathrm{det}(2\pi \Sigma)}\Bigr)^{\frac{2}{2+N}}$}.
\end{equation}
The expression above, derived in App.~\ref{sec:proof_truncated_paraboloid}, reduces to \eqref{eq:truncated_parabola} for $N=1$.
Notice that (unlike in the Gaussian case) a diagonal $\Sigma$ does not lead to a product of independent truncated parabolas.

\paragraph{Triangular.} 
Setting $f(t) = -|t-\mu|/b$, with $b>0$,  yields the triangular distribution 
\begin{equation}\label{eq:triangular}
\hat{p}_{\Omega_2}[f](t) = \left[ -\lambda - \tfrac{|t-\mu|}{b}\right]_+ \eqqcolon \mathrm{Tri}(t; \mu,b),
\end{equation}
where $\lambda\! =\! -1/\sqrt{b}$ (see App.~\ref{sec:proof_triangular}). 
\remove{
$\mathrm{supp}(\hat{p}_{\Omega_2}[f]) = [\mu - \sqrt{b}, \mu + \sqrt{b}]$, and $\Omega_2(\hat{p}_{\Omega_2}[f]) = -\frac{1}{2} +  \frac{1}{3\sqrt{b}}.$
}
Fig.~\ref{fig:gaussian_paraboloid} (bottom left) depicts this distribution alongside Laplace.


\paragraph{Location-scale families.} 
More generally, let
$f_{\mu, \sigma}(t) := -\frac{1}{\sigma}g'(|t-\mu|/\sigma)$ for a location  $\mu \in \mathbb{R}$ and a scale  $\sigma > 0$, where
$g:\mathbb{R}_+ \rightarrow \mathbb{R}$ is convex and continuously differentiable.
Then, we have
\begin{equation}
\hat{p}_{\Omega_2}[f](t) = \left[-\lambda - \tfrac{1}{\sigma}g'(|t-\mu|/\sigma)\right]_+,
\end{equation}
where $\lambda = -g'(a)/\sigma$
and $a$ is the solution of the equation $ag'(a) - g(a) + g(0) = \frac{1}{2}$ (a sufficient condition for such solution to exist is $g$ being strongly convex; see App.~\ref{sec:proof_location_scale} for a proof).  
\remove{
The support of this distribution is
$\mathrm{supp}(\hat{p}_{\Omega_2}[f_{\mu, \sigma}]) = [(-a+\mu)/\sigma, (a+\mu)/\sigma]$. 
}
This example subsumes the truncated parabola ($g(t) = t^3/6$) and the triangular distribution ($g(t) = t^2/2$). 

\paragraph{2-sparse families.} 
Truncated parabola and paraboloid distributions 
form a 2-sparse family, with statistics $\phi(t) = [t, \mathrm{vec}(tt^\top)]$ and  canonical parameters $\theta = [\Sigma^{-1}\mu, \mathrm{vec}(-\frac{1}{2}\Sigma^{-1})]$. 
\remove{
From (\ref{eq:sparsemax}--\ref{eq:truncated_parabola}):
\begin{equation*}
    A_2(\theta) = -\frac{1}{2}\left(\frac{3}{2\sigma}\right)^{2/3} + \frac{\mu^2}{2\sigma^2} + 1 =  -\frac{\theta_1^2}{4\theta_2} + \left(\frac{9\theta_2}{16}\right)^{1/3} + 1.
\end{equation*}
Contrast this expression with the one for Gaussian distributions, which form an exponential family with the same sufficient statistics and canonical parameters, for which $A_1(\theta) = -\frac{\theta_1^2}{4\theta_2} - \frac{1}{2}\log(-2\theta_2)$.
}
Gaussian distributions form an exponential family with the same sufficient statistics and canonical parameters. 
In 1D, truncated parabola and Gaussians are both particular cases of the so-called ``$q$-Gaussian'' \citep[\S 4.1]{naudts2009q}, for $q = 2-\alpha$. 
\remove{Fig.~\ref{fig:gaussian_laplacian_distributions} (right) shows $(2-\alpha)$-Gaussian distributions for $\alpha \in [1,2]$.}
Triangular distributions with a fixed location $\mu$ and varying scale $b$ also form a 2-sparse family (similarly to Laplace distributions with fixed location being exponential families).

\section{Continuous Attention}\label{sec:attention}


Attention mechanisms have become a key component of  neural networks 
\citep{bahdanau2014neural,sukhbaatar2015end,vaswani2017attention}. They dynamically detect and extract relevant input features (such as words in a text or regions of an image). So far, attention has only been applied to discrete domains; we generalize it to {\it continuous} spaces.

\paragraph{Discrete attention.}
Assume an input object split in $L=|S|$ pieces, \textit{e.g.}, a document with $L$ words or an image with $L$ regions. 
A vanilla attention mechanism works as follows: each piece has a $D$-dimensional representation (\textit{e.g.}, coming from an RNN or a CNN), yielding a matrix $V \in \mathbb{R}^{D\times L}$. 
These representations are compared against a query vector (\textit{e.g.}, using an additive model \citep{bahdanau2014neural}), leading to a score vector $f = [f_1, \ldots, f_L] \in \mathbb{R}^L$. 
Intuitively, the relevant pieces that need attention should be assigned high scores. Then, a transformation $\rho : \mathbb{R}^L \rightarrow \triangle^{L}$ (\textit{e.g.}, softmax or sparsemax) is applied to the score vector to produce a probability vector $p = \rho(f)$.
We may see this as an $\Omega$-RPM. The probability vector $p$ is then used to compute a weighted average of the input representations, via $c = Vp \in \mathbb{R}^D$. This context vector $c$ is finally used to produce the network's decision.

To learn via the backpropagation algorithm, the Jacobian of the transformation $\rho$, $J_{\rho} \in \mathbb{R}^{L \times L}$, is needed. 
\citet{Martins2016ICML} gave  expressions for softmax and sparsemax,
\begin{equation}\label{eq:jacobians_discrete}
J_{\mathrm{softmax}}(f) = \mathrm{Diag}(p) - pp^\top, \qquad
J_{\mathrm{sparsemax}}(f) = \mathrm{Diag}(s) - ss^\top/(1^\top s),
\end{equation}
where $p=\mathrm{softmax}(f)$, and $s$ is a binary vector whose $\ell$\textsuperscript{th} entry is 1 iff $\ell \in \mathrm{supp}(\mathrm{sparsemax}(f))$. 

\subsection{The continuous case: score and value functions}
\label{subsec:continuous_attention}

Our extension of $\Omega$-RPMs to arbitrary domains (Def.~\ref{def:regularized_prediction}) opens the door for constructing {\bf continuous attention mechanisms}. The idea is simple: instead of splitting the input object into a finite set of pieces, we assume an underlying continuous domain: \textit{e.g.}, text may be represented as a function $V:S \rightarrow \mathbb{R}^{D}$ that maps points in the real line ($S \subseteq \mathbb{R}$, continuous time) onto a $D$-dimensional vector representation, representing the ``semantics'' of the text evolving over time; images may be regarded as a smooth function in 2D ($S \subseteq \mathbb{R}^2$), instead of being split into regions in a grid. 

Instead of scores $[f_1, \ldots, f_L]$, we now have a {\bf score function} $f: S  \rightarrow
\mathbb{R}$, which we map to a probability density $p \in \mathcal{M}_+^1(S)$. 
This density is used in tandem with the  value mapping $V: S \rightarrow \mathbb{R}^D$ to obtain a context vector $c = \mathbb{E}_{p} [V(t)]  \in \mathbb{R}^D$. 
Since $\mathcal{M}_+^1(S)$ may be infinite dimensional, we need to parametrize  $f$, $p$, and $V$ to be able to compute in a finite-dimensional parametric space.

\paragraph{Building attention mechanisms.}
We represent $f$ and $V$ using basis functions, 
$\phi: S \rightarrow \mathbb{R}^M$ and $\psi: S \rightarrow \mathbb{R}^N$, 
defining $f_\theta(t) = \theta^\top \phi(t)$
and $V_B(t) = B\psi(t)$, 
where $\theta \in \mathbb{R}^M$ and $B \in \mathbb{R}^{D \times N}$. 
The score function $f_\theta$ is mapped into a probability density $p := \hat{p}_{\Omega}[f_\theta]$, 
from which we compute the context vector as $c=\mathbb{E}_{p}[V_B(t)]=Br$,  with $r =\mathbb{E}_{p}[\psi(t)]$. 
\remove{
from which the expectation
$r = \mathbb{E}_{p}[\psi(t)] \in \mathbb{R}^N$ is then obtained.
Given $r$, a context vector is computed as $c = Br$; from the definition of $V_B(t)$, this is equivalent to writing $c = \mathbb{E}_{p}[V_B(t)]$. 
}
Summing up yields the  following:
\vspace{0.1cm}
\begin{definition}\label{def:attention_mechanism}
Let $\langle S, \Omega, \phi, \psi \rangle$ be a tuple with  $\Omega: \mathcal{M}_+^1(S) \rightarrow \mathbb{R}$, $\phi:S\rightarrow \mathbb{R}^M$, and $\psi:S\rightarrow \mathbb{R}^N$.
An \textbf{attention mechanism} is a mapping $\rho: \Theta \subseteq \mathbb{R}^M \rightarrow \mathbb{R}^N$, defined as:
\begin{equation}\label{eq:attention_expectation}
    \rho(\theta) = \mathbb{E}_{p}[\psi(t)],
\end{equation}
with $p = \hat{p}_\Omega[f_\theta]$ and $f_\theta(t) = \theta^\top \phi(t)$.
If $\Omega = \Omega_\alpha$, we  call this \textbf{entmax} attention, denoted as $\rho_\alpha$. The values $\alpha=1$ and $\alpha=2$ lead to softmax and sparsemax attention, respectively.
\end{definition}

Note that, if $S = \{1, ..., L\}$ and $\phi(k) = \psi(k) = e_k$ (Euclidean canonical basis), we recover the discrete attention of  \citet{bahdanau2014neural}.  
Still in the finite case, if $\phi(k)$ and $\psi(k)$ are key and value vectors and $\theta$ is a query vector, this recovers the key-value attention of \citet{vaswani2017attention}.  

On the other hand, for $S = \mathbb{R}^D$ and $\phi(t) = [t, \mathrm{vec}(tt^\top)]$, we obtain new attention mechanisms (assessed experimentally for the 1D and 2D cases 
in \S\ref{sec:experiments}): for $\alpha=1$, the underlying density $p$ is Gaussian, and for $\alpha=2$, it is a truncated paraboloid (see \S\ref{sec:sparsemax}). In both cases, we show (App.~\ref{sec:gaussian_basis}) that the expectation \eqref{eq:attention_expectation} is tractable (1D) or simple to approximate numerically (2D) if 
$\psi$ are Gaussian RBFs, and we use this fact in \S\ref{sec:experiments} (see  Alg.~\ref{algo:forward_backward_gaussian} for pseudo-code for the case $\alpha=1$).

\begin{algorithm}[t]
\small
\SetAlgoLined
\SetKwInput{KwInput}{Parameters}
\SetKwFunction{FRegression}{Regression}
\SetKwFunction{FForward}{Forward}
\SetKwFunction{FBackward}{Backward}
\def\algspace{.5\baselineskip}
\SetKwProg{Fn}{Function}{:}{}
\KwInput{Gaussian RBFs $\psi(t) = [\mathcal{N}(t; \mu_j, \Sigma_j)]_{j=1}^N$, basis functions $\phi(t) = [t, \mathrm{vec}(tt^\top)]$, value function $V_B(t) = B\psi(t)$ with $B \in \mathbb{R}^{D \times N}$, score function $f_\theta(t) = \theta^\top \phi(t)$ with $\theta \in \mathbb{R}^M$}
\vspace{\algspace}
\Fn{\FForward{$\theta := [\Sigma^{-1}\mu, -\frac{1}{2}\Sigma^{-1}]$}}{
    $r_j \leftarrow \mathbb{E}_{\hat{p}_\Omega[f_\theta]}[\psi_j(t)] = \mathcal{N}(\mu, \mu_j, \Sigma+\Sigma_j), \quad \forall j \in [N]$\hfill%
    \tcp*[h]{Eqs.\,\ref{eq:attention_expectation}, \ref{eq:continuous_softmax_forward_pass}}\\
    \KwRet{$c \leftarrow Br$ (context vector)}%
}
\vspace{\algspace}
\vspace{\algspace}
\Fn{\FBackward{$\frac{\partial \mathcal{L}}{\partial c}, \theta := [\Sigma^{-1}\mu, -\frac{1}{2}\Sigma^{-1}]$}}{
    \For{$j\gets1$ \KwTo $N$}{
        $\tilde{s} \leftarrow \mathcal{N}(\mu, \mu_j, \Sigma+\Sigma_j), \,\,
        \tilde{\Sigma} \leftarrow (\Sigma^{-1}+\Sigma_j^{-1})^{-1}, \,\,
        \tilde{\mu} \leftarrow \tilde{\Sigma}(\Sigma^{-1}\mu + \Sigma_j^{-1}\mu_j)$\\
        $\frac{\partial r_j}{\partial \theta} \leftarrow \mathrm{cov}_{\hat{p}_\Omega[f_\theta]}(\phi(t), \psi_j(t)) = [\tilde{s}(\tilde{\mu} - \mu); \tilde{s}(\tilde{\Sigma} + \tilde{\mu}\tilde{\mu}^\top - \Sigma - \mu\mu^\top)]$\hfill%
    \tcp*[h]{Eqs.\,\ref{eq:jacob}, \ref{eq:continuous_softmax_backward_pass_01}--\ref{eq:continuous_softmax_backward_pass_02}}\\        
    }
    \KwRet{$\frac{\partial \mathcal{L}}{\partial \theta} \leftarrow \left(\frac{\partial r}{\partial \theta}\right)^\top B^\top \frac{\partial \mathcal{L}}{\partial c}$}
}
\caption{Continuous softmax attention with $S=\mathbb{R}^D$, $\Omega=\Omega_1$, and Gaussian RBFs.\label{algo:forward_backward_gaussian}}
\end{algorithm}

\paragraph{Defining the value function $V_B(t)$.}
In many problems, the input is a discrete sequence of observations (\textit{e.g.}, text) or it was discretized (\textit{e.g.}, images), at locations $\{t_\ell\}_{\ell=1}^L$. To turn it into a continuous signal, we need to smooth and interpolate these observations. If we start with a  discrete encoder representing the input as a matrix $H \in \mathbb{R}^{D \times L}$, one way of obtaining a value mapping $V_B: S \rightarrow \mathbb{R}^D$ is by ``approximating'' $H$ with {\it multivariate ridge regression}. With  $V_B(t) = B \psi(t)$, 
and packing the basis vectors $\psi(t_\ell)$ as columns of matrix $F \in \mathbb{R}^{N\times L}$, we obtain: \begin{equation}\label{eq:B_regression}
    B^\star \,\,=\,\, \arg\min_B \|BF - H\|_F^2 + \lambda \|B\|_F^2 \,\,=\,\,  HF^\top (FF^\top + \lambda I_N)^{-1} \,\,=\,\, HG,
\end{equation}
where $\|.\|_F$ is the Frobenius norm, and
the $L\times N$ matrix $G = F^\top (FF^\top + \lambda I_N)^{-1}$ depends only on the values of the basis functions at discrete time steps and can be obtained off-line for different input lenghts $L$.
The result is an expression for $V_B$ with $ND$ coefficients, cheaper than $H$ if $N \ll L$. 

\subsection{Gradient backpropagation with continuous attention}\label{sec:jacobian}

The next proposition, based on Proposition~\ref{prop:gradient_A} and proved in App.~\ref{sec:proof_jacobian_entmax}, allows backpropagating over continuous entmax attention mechanisms. We define, for $\beta \ge 0$, a {\it generalized $\beta$-covariance},
\begin{equation}\label{eq:beta_covariance}
\mathrm{cov}_{p, \beta}[\phi(t), \psi(t)] \,\,=\,\, \|p\|_\beta^\beta \times \left( \mathbb{E}_{\tilde{p}_\beta}\big[\phi(t) \psi(t)^\top\big] - \mathbb{E}_{\tilde{p}_\beta}[\phi(t)]\,  \mathbb{E}_{\tilde{p}_\beta}[\psi(t)]^\top
\right),
\end{equation}
where $\tilde{p}_\beta$ is the $\beta$-escort distribution in \eqref{eq:escort}.
For $\beta=1$, we have the usual covariance; for $\beta=0$, we get a covariance taken w.r.t.\ a uniform density on the support of $p$,  scaled by $|\mathrm{supp}(p)|$.

\vspace{0.1cm}
\begin{proposition}\label{prop:jacobian_entmax}
Let $p = \hat{p}_{\Omega_\alpha}[f_\theta]$ with $f_\theta(t) = \theta^\top \phi(t)$.
The Jacobian of the $\alpha$-entmax is:
    \begin{equation}\label{eq:jacob}
        J_{\rho_\alpha}(\theta) = \frac{\partial \rho_\alpha(\theta)}{\partial \theta} = \mathrm{cov}_{p, 2-\alpha}(\phi(t), \psi(t)).
    \end{equation}
\end{proposition}

Note that in the finite case, \eqref{eq:jacob} reduces to the expressions in \eqref{eq:jacobians_discrete} for softmax and sparsemax. 

\paragraph{Example: Gaussian RBFs.} 
As before, let $S = \mathbb{R}^D$, $\phi(t) = [t, \mathrm{vec}(tt^\top)]$, and $\psi_j(t) = \mathcal{N}(t; \mu_j, \Sigma_j)$. 
For $\alpha=1$, we obtain closed-form expressions for the expectation \eqref{eq:attention_expectation} and the Jacobian \eqref{eq:jacob}, for any $D \in \mathbb{N}$: 
$\hat{p}_\Omega[f_\theta]$ is a Gaussian, the expectation 
\eqref{eq:attention_expectation} is the integral of a product of Gaussians, and the covariance \eqref{eq:jacob} involves first- and second-order Gaussian moments. Pseudo-code for the case $\alpha=1$ is shown as Alg.~\ref{algo:forward_backward_gaussian}. 
For $\alpha=2$, $\hat{p}_\Omega[f_\theta]$ is a truncated paraboloid. 
In the 1D case, both \eqref{eq:attention_expectation} and \eqref{eq:jacob} can be expressed in closed form in terms of the $\mathrm{erf}$ function. 
In the 2D case, we can reduce the problem to 1D integration using the change of variable formula and working with polar coordinates. See App.~\ref{sec:gaussian_basis} for details. 

We use the facts above in the experimental section (\S\ref{sec:experiments}), where we experiment with continuous variants of softmax and sparsemax attentions in natural language processing and vision applications.

\remove{
\section{Continuous Fenchel-Young Losses}\label{sec:fy_losses}

We saw above how to construct distributions $p$ from $f_\theta$ via the $\Omega$-RPM (\S\ref{sec:sparse_families}), and how to used them to build attention mechanisms (\S\ref{sec:attention}). What about the reverse: can we estimate $\theta$ from an empirical data distribution $\bar{p}$? 
For finite $S$, \citet{blondel2020learning} introduced the notion of {\bf Fenchel-Young loss}. Here, we extend that notion to arbitrary domains.\footnote{The construction hinges on the notion of Fenchel dual, denoted $\Omega^*$, of an l.s.c. proper convex function $\Omega\!:\!\mathcal{M}_+^1(S) \rightarrow \mathbb{R}$ \citep{Bauschke_Combettes2011}:
\begin{equation*}
    \Omega^*(f) = \!\!\!\! \max_{p\in \mathcal{M}_+^1(S)} \! \mathbb{E}_{p}[f(t)] - \Omega({p}) 
 \! =\!  \mathbb{E}_{\hat{p}_\Omega[f]}[f(t)] - \Omega(\hat{p}_\Omega[f]).\label{eq:Omega_star2}
\end{equation*}}

\begin{definition}
Given an l.s.c., proper, strictly convex function $\Omega:\mathcal{M}_+^1(S) \rightarrow \mathbb{R}$, the \textbf{Fenchel-Young loss} (FY loss) $L_{\Omega} : \mathcal{F} \times \mathcal{M}_+^1(S) \rightarrow \mathbb{R}$, defined as 
\begin{equation}
L_{\Omega}(f; p) = \Omega^*(f) + \Omega(p) - \mathbb{E}_{p}[f(t)].
\end{equation}
\end{definition}

One target use of FY losses is to estimate $f$, given some empirical data distribution $\bar p$, by minimizing $L_{\Omega}(f; \bar{p})$. The following proposition (that stems directly from the definition of $\Omega^*$ in \eqref{eq:Omega_star2}) motivates this use of the FY loss.
\vspace{0.1cm}
\begin{proposition}\label{prop:fy_properties1}
For any l.s.c., proper, strictly convex function $\Omega:\mathcal{M}_+^1(S) \rightarrow \mathbb{R}$, we have that $L_{\Omega}(f;p) \geq 0$ and  $L_{\Omega}(f;p) = 0 \Leftrightarrow p = \hat{p}_{\Omega}[f]$ almost everywhere.
\end{proposition}

For the parametric case, $f_{\theta}(t) = \theta^{\top}\phi(t)$, the following proposition (proved in App.~\ref{sec:proof_fy_properties}) extends the result of \citet{blondel2020learning} 
to arbitrary domains, creating the path to estimating $\theta$ from data and showing the existence of fixed-dimensional sufficient statistics for $\alpha$-sparse families. 
\vspace{0.05cm}
\begin{proposition}\label{prop:fy_properties}
Let $f_{\theta}(t) = \theta^\top \phi(t)$. Then, 
\vspace{-0.2cm}
\begin{enumerate}[nosep]
\item $\nabla_{\theta} L_{\Omega}(f_\theta; p) = \mathbb{E}_{\hat{p}_{\Omega}[f_\theta]}[\phi(t)] - \mathbb{E}_{p}[\phi(t)]$.
    \item $L_{\Omega_{\alpha}}(f_\theta; p)$ is convex w.r.t.\ $\theta$. 
\item $\hat{\theta} \in \arg\min_{\theta} L_{\Omega_{\alpha}}(f_\theta; p) \Leftrightarrow \mathbb{E}_{\hat{p}_{\Omega}[f_{\hat{\theta}}]}[\phi(t)] = v$,
    where $v = \mathbb{E}_{p}[\phi(t)]$.
\end{enumerate}
\end{proposition}
In point 3, if $p=\bar{p}$, an empirical data distribution, $v$ is the empirical mean of the statistics;  the statement shows that estimating $\theta$ only depends on 
$\bar{p}$ through $v$, which generalizes the concept of sufficient statistics from exponential families.

\begin{example}\label{ex:fy_gaussians} For $\alpha= 1$ (Shannon entropy), the Fenchel dual is the log-partition function  \eqref{eq:boltzmann}, $\Omega_1^*(f) = A_1(f)$, and the FY loss recovers the  Kullback-Leibler divergence (KLD).  
For $p(t) = \mathcal{N}(t; \mu_p, \sigma_p^2)$ and $f(t) = -\frac{(t-\mu)^2}{2\sigma^2}$, 
\begin{equation}\label{eq:kl_gaussians}
L_{\Omega_1}(f; p) = \log\left(\tfrac{\sigma}{\sigma_p}\right) + \tfrac{\sigma_p^2 + (\mu_p - \mu)^2}{2\sigma^2} - \tfrac{1}{2},
\end{equation}
since the negentropy of a Gaussian is as given in Example~\ref{ex:Gaussian}. This is the well-known KLD between two Gaussians.
\end{example}

As shown in App.~\ref{sec:gini_ent_sparse_family}, 
the conjugate of $\Omega_\alpha$ is
\begin{equation}\label{eq:omega_conjugate}
\Omega_\alpha^*(f) = (\alpha-1)\Omega_{\alpha}(\hat{p}_{\Omega_\alpha}[f]) + \lambda + (\alpha-1)^{-1},
\end{equation}
from which the following examples of FY losses can be obtained (full derivations are in App.~\ref{sec:fy_example_derivations}).

\begin{example} For $\alpha =2$ and $f(t) = -\frac{(t-\mu)^2}{2\sigma^2}$, we have (see \eqref{eq:truncated_parabola}) $\hat{p}_{\Omega_\alpha}[f] = \mathrm{TP}(t; \mu, \sigma^2)$ and $\lambda$ as given by \eqref{eq:lambda_gaussian}. Plugging these in \eqref{eq:omega_conjugate} yields $\Omega_2^*(f) = \tfrac{1}{2} -\tfrac{3}{10} \left(\tfrac{3}{2\sigma} \right)^{2/3}.$
If $p(t) = \mathrm{TP}(t; \mu_p, \sigma_p^2)$ is another truncated parabola, the FY loss between $f$ and $p$ becomes (see App.~\ref{sec:fy_example_derivations}):
\begin{equation*}
L_{\Omega_2}(f; p) = \tfrac{ \sqrt[3]{\tfrac{9}{4}}\Bigl( \sigma_p^{\tfrac{4}{3}} - 3 \sigma^{\tfrac{4}{3}} + 2\sigma^2\sigma_p^{-\frac{2}{3}} \Bigr)  }{10 \sigma^2} +  \frac{(\mu-\mu_p)^2}{2\sigma^2}.
\end{equation*}
If $\sigma=\sigma_p$,  $L_{\Omega_2}$ and \eqref{eq:kl_gaussians} are both equal to $(\mu-\mu_p)^2/(2\sigma^2)$.

\end{example}

\begin{example}\label{ex:fy_triangular}
If $f(t) = -\frac{|t-\mu|}{b}$  and 
$p(t) = \mathrm{Tri}(t; 0, b_p)$
\eqref{eq:triangular}, we have
\begin{equation*}
L_{\Omega_2}(f; p) =
\tfrac{\sqrt{b} - 2\sqrt{b_p}}{3 \sqrt{b\, b_p}} 
+ \max \Bigl\{\tfrac{|\mu|}{b}, \,\, -\tfrac{|\mu|^3}{3 b_p b} +\tfrac{\mu^2}{b_p^{1/2} b} + \tfrac{b_p^{1/2}}{3b}\Bigr\}.
\end{equation*}
\end{example}

We use the FY expressions from examples~\ref{ex:fy_gaussians}-\ref{ex:fy_triangular} for our interval regression experiments in \S\ref{sec:experiments}.
}

\section{Experiments}\label{sec:experiments}

As proof of concept, we test our continuous attention mechanisms on three tasks: document classification, machine translation, and visual question answering (more experimental details in App.~\ref{sec:model_hyperparams}). 

\remove{
\subsection{Continuous attention mechanisms}
}

\paragraph{Document classification.} 

Although textual data is fundamentally discrete, 
modeling long documents as a continuous signal may be advantageous, due to smoothness and independence of length. 
To test this hypothesis, we use the IMDB movie review dataset \citep{maas2011learning}, 
whose inputs are documents (280 words on average) and  outputs are sentiment labels (positive/negative). Our baseline is a biLSTM with discrete attention. 
For our continuous attention models, we normalize the document length $L$ into the unit interval $[0,1]$, and use $f(t)\! =\! \sfrac{-(t-\mu)^2}{2\sigma^2}$ as the score function, leading to a 1D Gaussian ($\alpha=1$) or  truncated parabola ($\alpha=2$) as the attention density.  We compare three attention variants: 
{\bf discrete attention} with softmax \citep{bahdanau2014neural} and sparsemax \citep{Martins2016ICML};  
{\bf continuous attention},  
where a CNN and max-pooling  
yield a document representation $v$ from which 
we compute $\mu = \mathrm{sigmoid}(w_1^\top v)$ and $\sigma^2 = \mathrm{softplus}(w_2^\top v)$;   
and {\bf combined attention}, which obtains $p \in \triangle^L$ from discrete attention, computes $\mu = \mathbb{E}_{p}[\ell/L]$ and $\sigma^2 = \mathbb{E}_{p}[(\ell/L)^2] - \mu^2$, applies the continuous attention, and sums the two context vectors (this model has the same number of parameters as the discrete attention baseline). 

\begin{table}[t]
    \caption{Results on IMDB in terms of accuracy (\%). 
    For the continuous attentions, we used $N \in \{32, 64, 128\}$ Gaussian RBFs 
$\mathcal{N}(t, \tilde{\mu}, \tilde{\sigma}^2)$, with $\tilde{\mu}$ linearly spaced in $[0,1]$ and $\tilde{\sigma} \in \{.1, .5\}$. 
    } 
    \label{table:results_doc_classification}
    \vspace{-0.1cm}
    \begin{scriptsize}
    \begin{center}
\begin{tabular}{lc}
        \toprule
        \sc Attention & $\bar{L} \approx 280$ \\
        \midrule 
        Discrete softmax    		& 90.78 	\\
        Discrete sparsemax	 		& 90.58 	\\
       \bottomrule
 \end{tabular}
 \qquad
\begin{tabular}{lccc}
        \toprule
        \sc Attention & $N=32$ & $N=64$ & $N=128$ \\
       \midrule
        Continuous softmax			& 90.20		& 90.68		& 90.52    	\\
        Continuous sparsemax	 	& 90.52		& 89.63		& 90.90   	\\
        Disc. + Cont. softmax		& 90.98		& 90.69		& 89.62   	\\
        Disc. + Cont. sparsemax	 	& \bf 91.10		& \bf 91.18		& \bf 90.98    	\\
        \bottomrule
    \end{tabular}
\end{center}
    \end{scriptsize}
\end{table}

Table~\ref{table:results_doc_classification} 
shows accuracies for different numbers $N$ of Gaussian RBFs. The accuracies of the individual models are similar, suggesting that continuous attention is as effective as its discrete counterpart, despite having fewer basis functions than words, \textit{i.e.}, $N \ll L$. Among the continuous variants, the sparsemax outperforms the softmax, except for $N=64$. We also see that a large $N$ is not necessary to obtain good results, which is encouraging for tasks with long sequences. Finally, combining discrete and continuous sparsemax produced the best results, without increasing the number of parameters.

\begin{figure*}[t]
\centering
\includegraphics[width=0.3\textwidth]{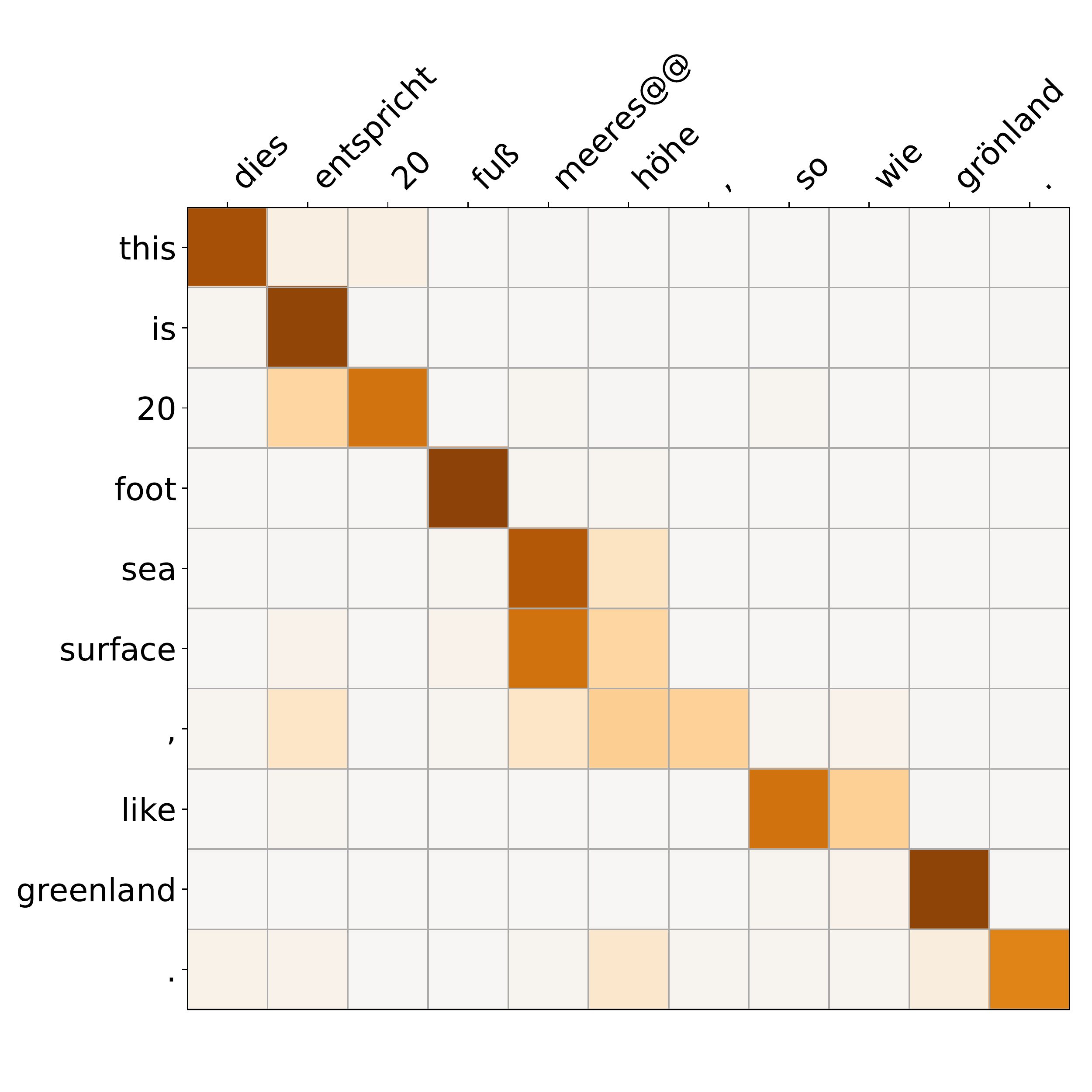}
\includegraphics[width=0.3\textwidth]{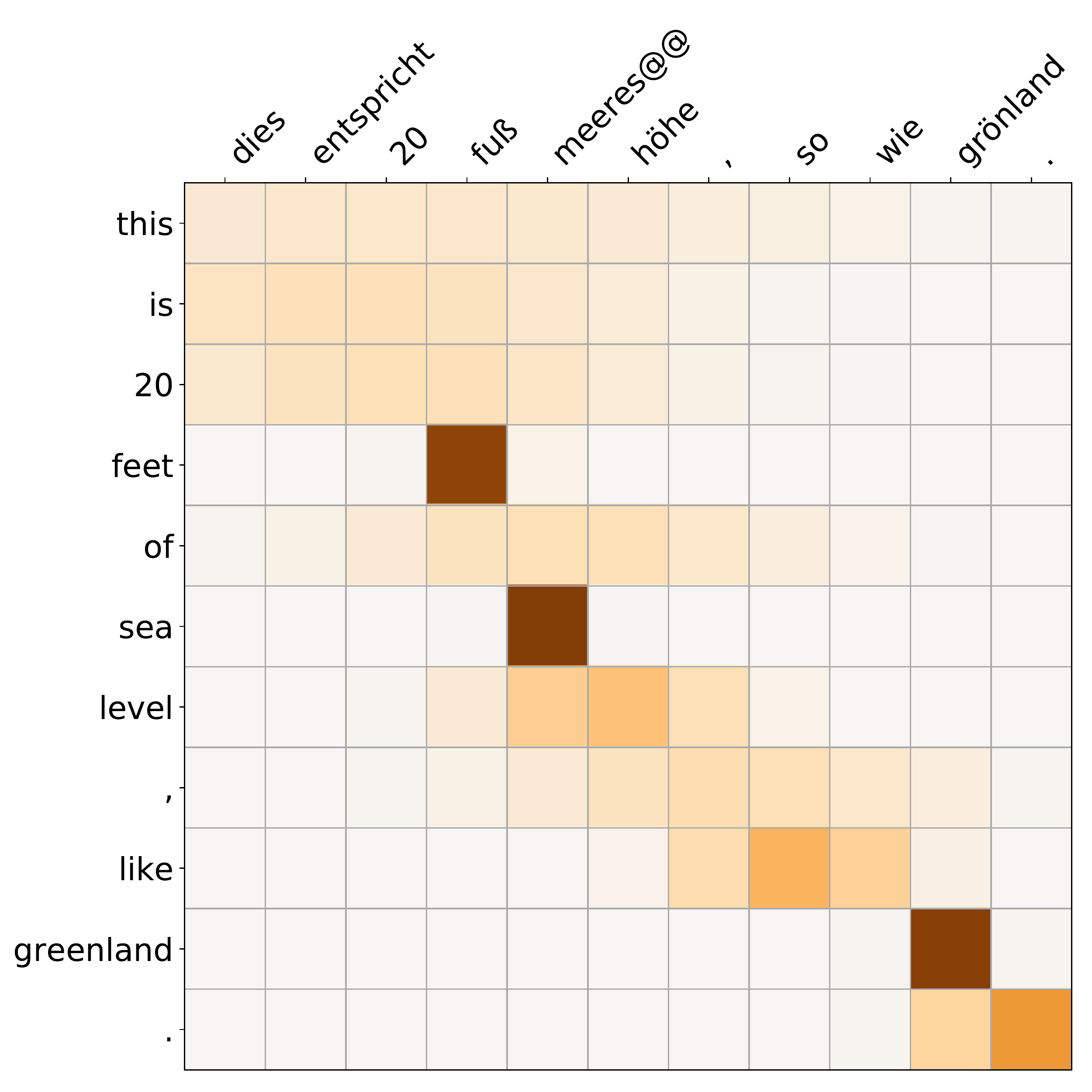}
\includegraphics[width=0.3\textwidth]{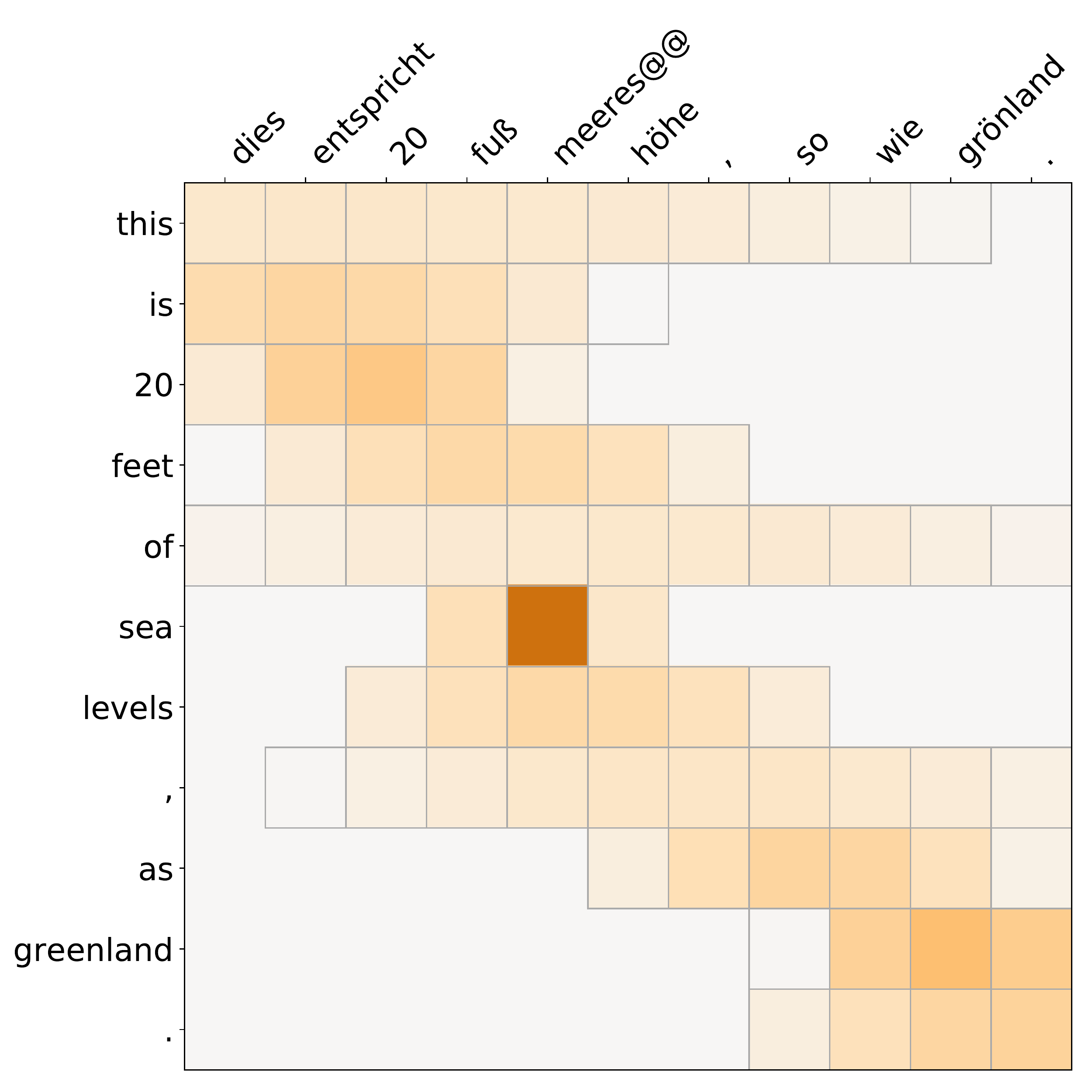}
\caption{\label{fig:attention_maps}Attention maps in machine translation: discrete (left), continuous softmax (middle), and continuous sparsemax (right), for a sentence in the De-En IWSLT17 validation set. In the rightmost plot, the selected words are the ones with positive density. In the test set, these models attained BLEU scores of 23.92 (discrete), 24.00 (continuous softmax), and 24.25 (continuous sparsemax).}
\end{figure*}

\paragraph{Machine translation.} 

We use the De$\to$En IWSLT 2017 dataset \citep{cettolo2017overview}, and a biLSTM model with discrete softmax attention as a baseline. 
For the continuous attention models, we use the combined attention setting described above, with 30 Gaussian RBFs and $\tilde{\mu}$ linearly spaced in $[0,1]$ and $\tilde{\sigma} \in \{.03, .1, .3\}$. 
The results (caption of Fig.~\ref{fig:attention_maps}) 
show a slight benefit in the combined  attention over discrete attention only, without any additional parameters. 
Fig.~\ref{fig:attention_maps} shows heatmaps for the different attention mechanisms on a De$\to$En sentence. The continuous mechanism tends to have attention means close to the diagonal,   adjusting the variances based on alignment confidence or when a larger context is needed ({\it e.g.}, a peaked density for the target word ``sea'', and a flat one for ``of'').

\paragraph{Visual QA.} 


Finally, we report experiments with 2D continuous attention on visual question answering, using the VQA-v2 dataset \cite{Goyal2019} and a modular co-attention network as a baseline \cite{Yu2019}.%
\footnote{Software code is available at \url{https://github.com/deep-spin/mcan-vqa-continuous-attention}.} %
The discrete attention model attends over a  14$\times$14 grid.%
\footnote{An alternative would be bounding box features from an external object detector \citep{anderson2018bottom}. We opted for grid regions to check if continuous attention has the ability to detect relevant objects on its own. However, our method can handle bounding boxes too, if  the $\{t_\ell\}_{\ell=1}^L$ coordinates in the  regression
\eqref{eq:B_regression} are placed on those regions.} %
For continuous attention, we normalize the image size into the unit square $[0,1]^2$. We fit a 2D Gaussian ($\alpha=1$) or truncated paraboloid ($\alpha=2$) as the attention density; both correspond  to  $f(t)=-\frac{1}{2}(t-\mu)^\top\Sigma^{-1}(t-\mu)$, with $\Sigma \succ 0$. We use the mean and variance according to the discrete attention probabilities and obtain $\mu$ and $\Sigma$ with moment matching. We use $N = 100 \ll 14^2$ Gaussian RBFs, with $\tilde{\mu}$ linearly spaced in  $[0,1]^2$ and $\tilde{\Sigma}=0.001\cdot \mathrm{I}$. Overall, the number of neural network parameters is the same as in discrete attention. 

The results in Table~\ref{table:results_vqa} show similar accuracies for all attention models, with a slight advantage for continuous softmax. Figure~\ref{fig:examples_vqa} shows an example (see App.~\ref{sec:model_hyperparams} for more examples and some failure cases): in the baseline model, the discrete attention is too scattered, possibly mistaking the lamp with a TV screen.  The continuous attention models focus on the right region and answer the question correctly, with continuous sparsemax enclosing all the relevant information in its supporting ellipse.

\begin{figure*}[t]
\centering
\includegraphics[width=0.21\textwidth]{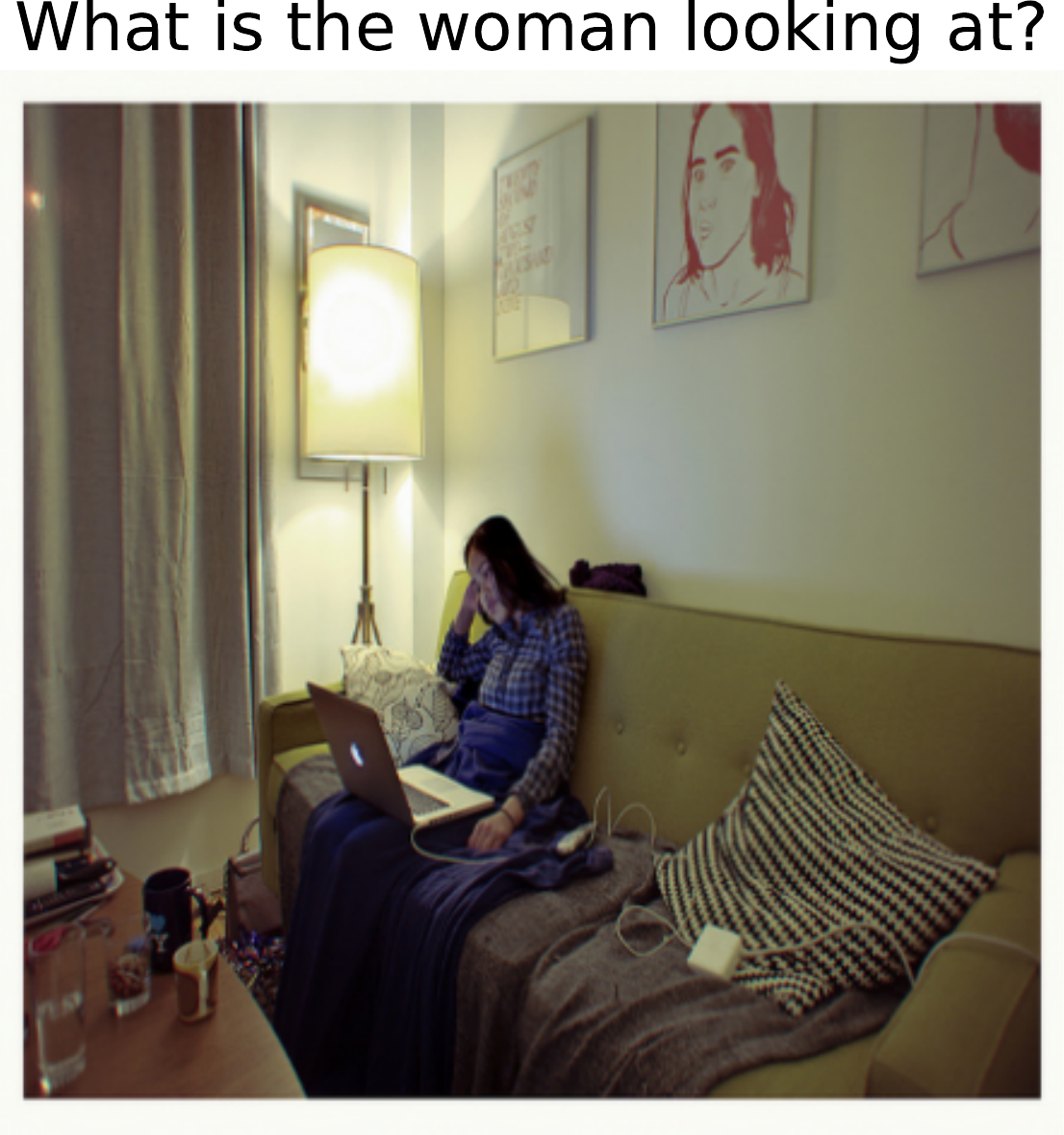}
\includegraphics[width=0.21\textwidth]{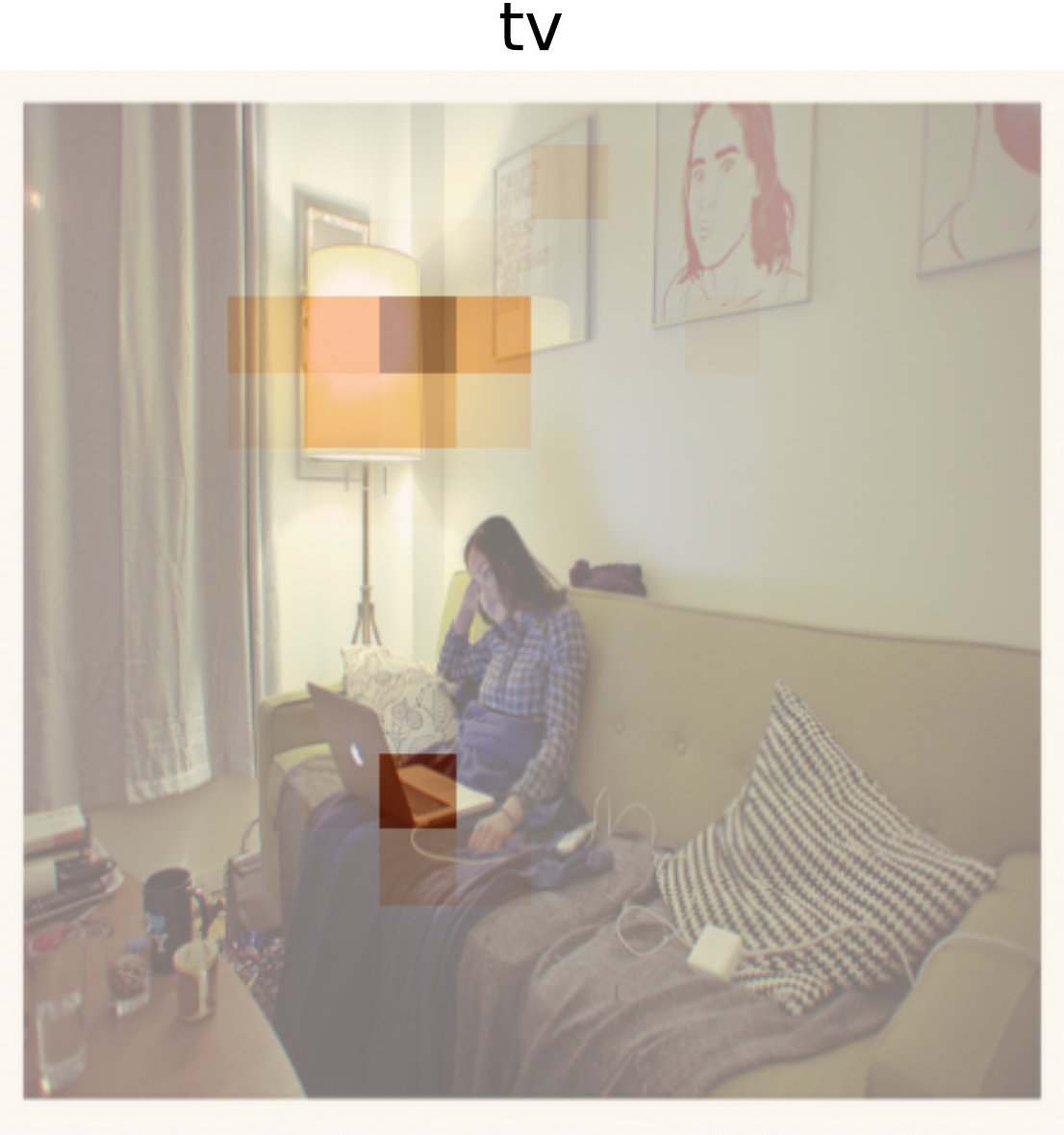}
\includegraphics[width=0.21\textwidth]{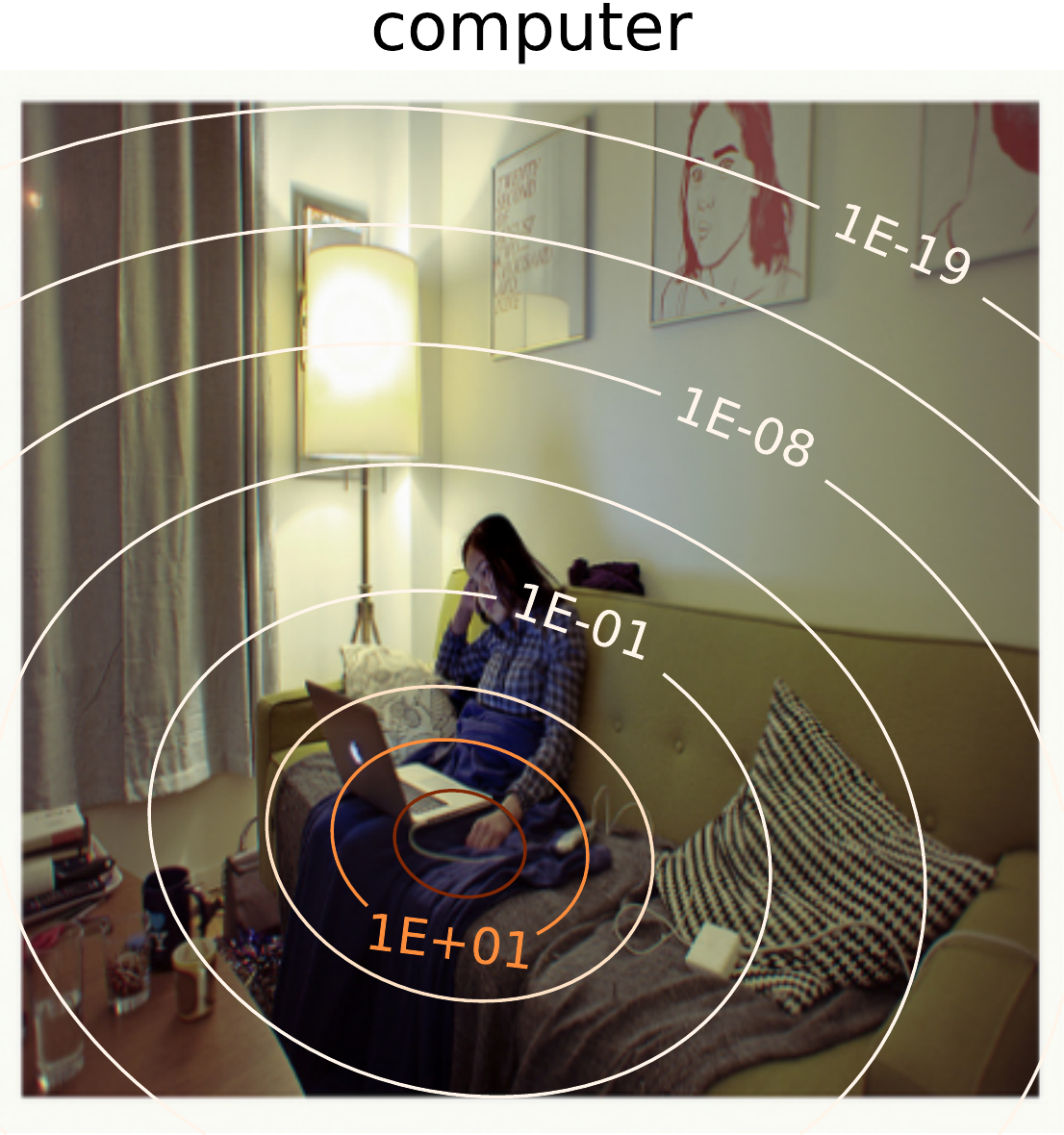}
\includegraphics[width=0.21\textwidth]{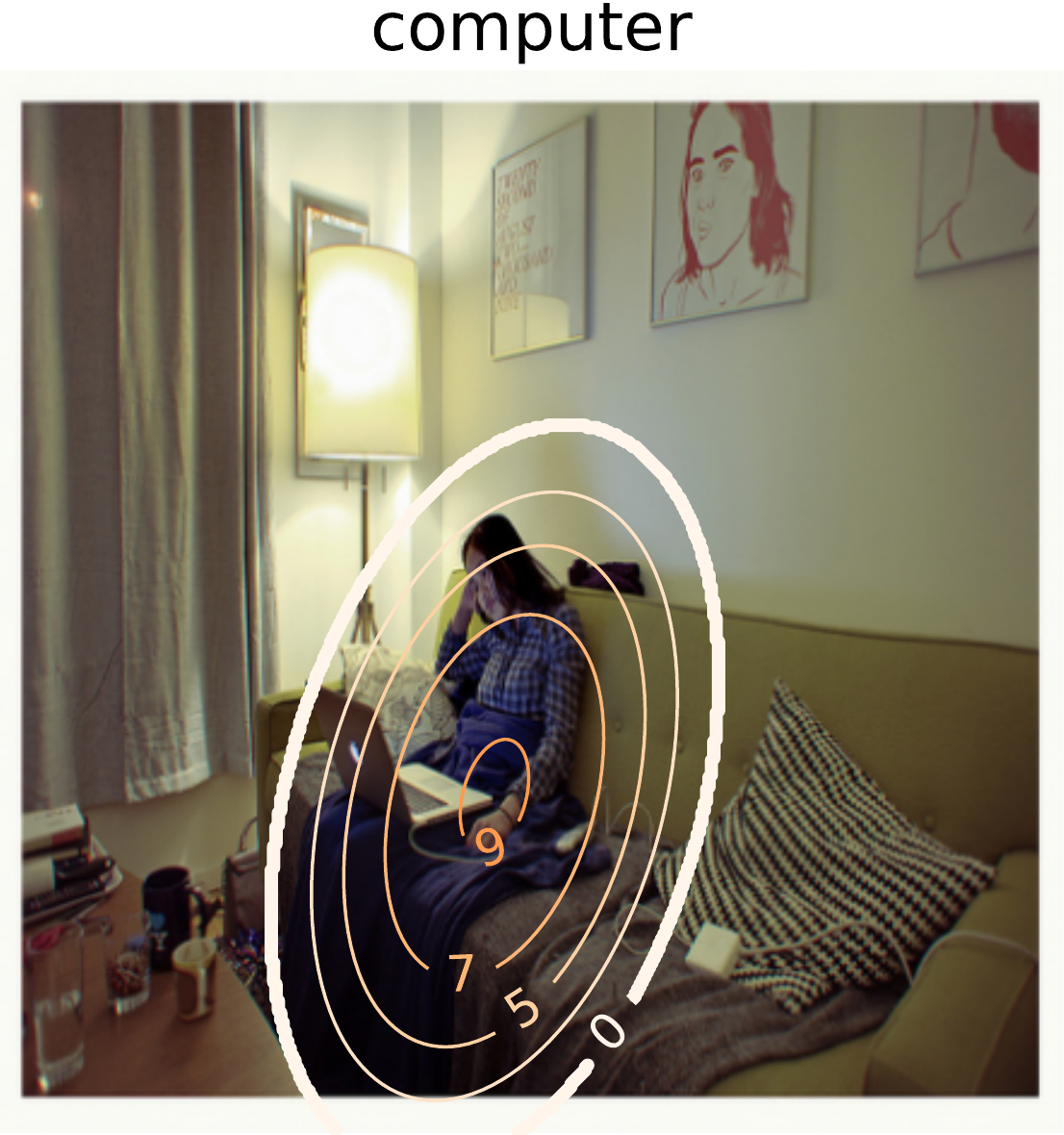}
\caption{\label{fig:examples_vqa}Attention maps for an example in VQA-v2: original image, discrete attention, continuous softmax, and continuous sparsemax. The latter encloses all probability mass within the outer ellipse.}
\end{figure*}

\begin{table}[t]
    \caption{Accuracies of different models on the \textit{test-dev} and \textit{test-standard} splits of VQA-v2.} 
    \label{table:results_vqa}
    \vspace{-.1cm}
    \begin{small}
    \begin{center}
\begin{tabular}{l@{\hspace{10pt}}c@{\hspace{10pt}}c@{\hspace{10pt}}c@{\hspace{10pt}}c@{\hspace{10pt}}c@{\hspace{10pt}}c@{\hspace{10pt}}c@{\hspace{10pt}}c}
        \toprule
        \sc Attention & \multicolumn{4}{c}{Test-Dev}  & \multicolumn{4}{c}{Test-Standard} \\
        {} & Yes/No & Number & Other & Overall & Yes/No & Number & Other & Overall \\
        \midrule 
        Discrete softmax    		& 83.40 & 43.59	& 55.91 & 65.83 & 83.47 & 42.99 & 56.33 & 66.13    	\\
        \midrule
        2D continuous softmax			& 83.40	& 44.80	& 55.88 & \textbf{65.96} & 83.79 & 44.33 & 56.04 & \textbf{66.27}    	\\
        2D continuous sparsemax	 	& 83.10 & 44.12 & 55.95 & 65.79 & 83.38 & 43.91 & 56.14 & 66.10  	\\
        \bottomrule
    \end{tabular}
\end{center}
    \end{small}
\end{table}

\remove{
\subsection{Interval regression}

\andre{Indeed, this is not a censored regression task with prescribed intervals. Interval-valued regression is studied in symbolic data analysis (see \citet{billard2000regression} and \citet{maia2008forecasting}. It is the regression counterpart of multi-label classification. Using two networks is an interesting suggestion (similar to the midpoint and range method of \citet{neto2010constrained}, we will add a comparison.}

Finally, we perform
regression
experiments to assess the continuous FY losses from \S\ref{sec:fy_losses}. Regression typically predicts a scalar outcome; however, when predicting stochastic phenomena, one is often interested in more information than just a point prediction. \textit{E.g.}, interval regression predicts an interval $[y_\text{min}, y_\text{max}]$; intervals are intuitively interpretable and, in some applications, they are available as supervision.

We address temperature prediction as interval
regression, using 2015--2019 data from the U.S.~Climate Reference Network \citep{diamond2013us} for Ithaca, NY, 
predicting minimum and maximum daily temperatures. Let $f_{\theta}$ denote the prediction function, computed in some  parametric form by an NN. Given an interval
$y = [y_\text{min}, y_\text{max}]$, we posit a  target distribution $p_y$,
and minimize the total loss
$
\sum_{(x, y) \in \mathcal{D}} L_{\Omega_\alpha}(f_{\theta(x)}; p_y).
$
The input features $x$ are minimum and maximum temperatures in the previous 2 days, and an index $x_0 \in [0, 1]$ of the target day ($x_0=0$ corresponds
to 2015-01-01 and $x_0=1$ to 2019-12-31). We remove windows with missing data and leave out 99
validation and 99 test days with non-overlapping windows.
We consider the following models for $p_y$: \textbf{(i)} Gaussian with $[\mu-\sigma, \mu+\sigma] = y$;
\textbf{(ii)} Gaussian with 95\% CI equal to $y$; \textbf{(iii)} truncated parabola of support $y$; \textbf{(iv)} triangular of support $y$. These choices correspond to assumptions about daily temperature distributions, given
min and max values; in all cases, we fit a distribution of the same type as $p_y$.

We train a 2-hidden-layer NN, outputting a location
parameter via an affine map, and a scale using a softplus-affine map. We tune the hidden dimension in $\{2^5, ..., 2^9\}$, the Adam learning rate in $.001 \times \{0.5, 1, 2\}$, and the dropout
rate in $\{0.05, 0.1, 0.2\}$. 
The baseline predicts the mid-point of the interval using mean
squared error and sets the interval width to the mean interval length in the training data. Tab.~\ref{tab:weather} shows that a truncated parabola leads to the best fit
in terms of both mean and support, suggesting sparse
distributions as a promising tool for interval regression.

\begin{figure}[h]\centering
\includegraphics[width=.5\textwidth]{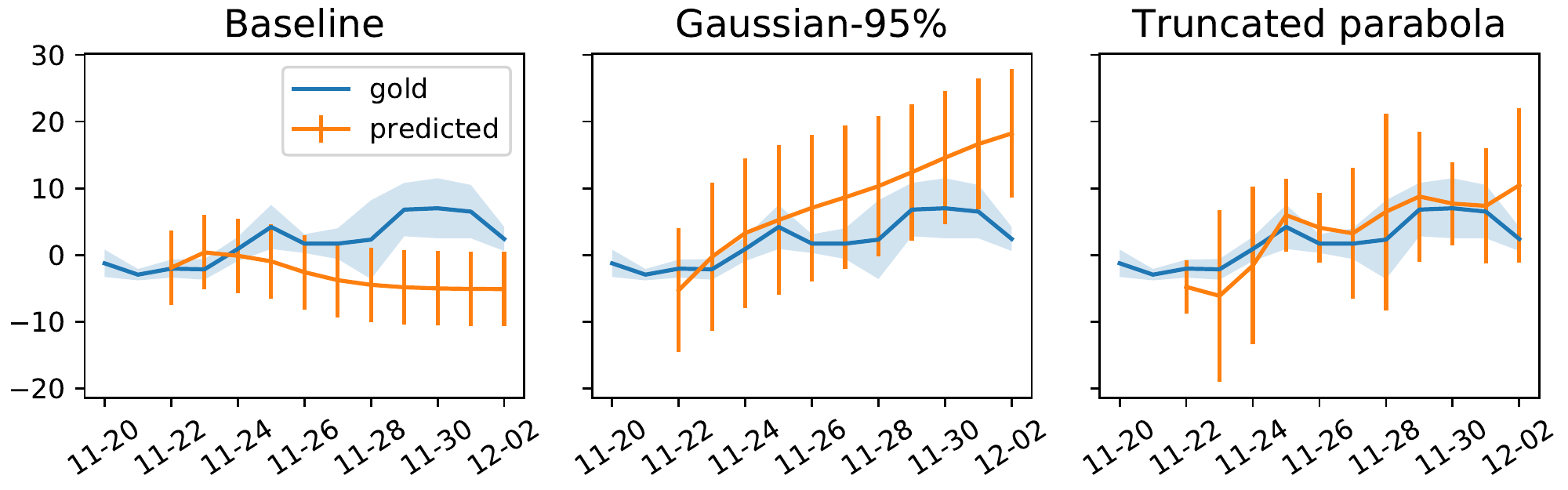}
\vspace{-0.5cm}
\caption{Auto-regressive temperature interval predictions ($^{\circ}$C).}
\end{figure}

\begin{table}[t]
    \caption{\label{tab:weather}Temperature interval prediction: mean squared error (w.r.t.\ the interval midpoint) and Jaccard similarity w.r.t.\ the gold interval.}
    \vskip 0.15in
    \begin{small}
    \begin{center}
    \begin{tabular}{lcc}
    \toprule
    & MSE & {\sc Jaccard} \\
    \midrule
    Baseline
    & 0.0112 & 0.4701 \\
    Gaussian-$\sigma$
    & 0.0111 & 0.4679 \\
    Gaussian-$95\%$
    & 0.0109 & 0.4709 \\
    Truncated parabola
    & \textbf{0.0108} & \textbf{0.4807} \\
    Triangular
    & 0.0110 & 0.4734 \\
    \bottomrule
    \end{tabular}
    \end{center}
    \end{small}
    \vskip -0.2in
\end{table}
}

\section{Related Work}\label{sec:related}
\paragraph{Relation to the Tsallis maxent principle.}
Our paper unifies two lines of work: deformed exponential families from statistical physics \citep{Tsallis1988,naudts2009q,amari2011geometry}, and sparse alternatives to softmax recently proposed in the machine learning literature \citep{Martins2016ICML,peters2019sparse,blondel2020learning}, herein extended to continuous domains. This link may be fruitful for future research in both fields. 
While most prior work is focused on heavy-tailed distributions ($\alpha < 1$), we focus instead on light-tailed, sparse distributions, the other side of the spectrum ($\alpha > 1$). See App.~\ref{sec:tsallis_maxent} for the relation to the Tsallis maxent principle.  

\paragraph{Continuity in other architectures and dimensions.} 
In our paper, we consider attention networks exhibiting temporal/spatial continuity in the input data, be it text (1D) or images (2D). Recent work propose continuous-domain CNNs for 3D structures like point clouds and molecules \citep{wang2018deep,schutt2017schnet}. The dynamics of continuous-time RNNs have been studied in \citep{funahashi1993approximation}, and similar ideas have been applied to irregularly sampled time series \citep{rubanova2019latent}. 
Other recently proposed frameworks produce continuous variants in other dimensions, such as network depth \citep{chen2018neural}, or in the target domain for machine translation tasks \citep{kumar2018mises}.
Our continuous attention networks can be used in tandem with these frameworks.

\paragraph{Gaussian attention probabilities.} 
\citet{cordonnier2020relationship} analyze the relationship between (discrete) attention and convolutional layers, and consider spherical Gaussian attention probabilities as relative positional encodings. By contrast, our approach removes the need for positional encodings: by converting the input to a function on a predefined continuous space, positions are encoded {\it implicitly}, not requiring explicit positional encoding. 
Gaussian attention has also been hard-coded as input-agnostic self-attention layers in transformers for machine translation tasks by 
\citet{you-etal-2020-hard}. 
Finally, in their DRAW architecture for image generation, \citet[\S 3.1]{gregor2015draw} propose a selective attention component which is parametrized by a spherical Gaussian distribution.

\section{Conclusions and Future Work}

We proposed extensions to regularized prediction maps,\remove{and FY losses,} originally defined on finite domains, to arbitrary measure spaces (\S\ref{sec:sparse_families}). 
With Tsallis $\alpha$-entropies for $\alpha>1$, we obtain sparse families, whose members can have zero tails, such as triangular or truncated parabola distributions. 
We then used these distributions to construct continuous attention mechanisms (\S\ref{sec:attention}). We derived their Jacobians in terms of generalized covariances (Proposition~\ref{prop:jacobian_entmax}), allowing for efficient forward and backward propagation.  Experiments for 1D and 2D cases were shown on attention-based text classification, machine translation, and visual question answering (\S\ref{sec:experiments}), with encouraging results. \remove{and FY-based interval regression.}

There are many avenues for future work. 
As a first step, we  considered unimodal distributions only (Gaussian, truncated paraboloid), for which we show that the forward and backpropagation steps have closed form or can be reduced to 1D integration. However, there are applications in which  multiple attention modes are desirable. This can be done by considering mixtures of distributions,  multiple attention heads, or sequential attention steps. 
Another direction concerns combining our continuous attention models with 
other spatial/temporal continuous architectures for CNNs and RNNs \citep{wang2018deep,schutt2017schnet,funahashi1993approximation} or 
with continuity in other dimensions, such as depth \citep{chen2018neural} or output space \citep{kumar2018mises}.



\section*{Broader Impact}

We discuss the broader impact of our work, including ethical aspects and future societal consequences. Given the early stage of our work and its predominantly theoretical nature, the discussion is mostly speculative.

The continuous attention models developed in our work can be used in a very wide range of applications, including natural language processing, computer vision, and others. For many of these applications, current state-of-the-art models use discrete softmax attention, whose interpretation capabilities have been questioned in prior work \citep{jain2019attention,serrano2019attention,wiegreffe2019attention}. Our models can potentially lead to more interpretable decisions, since they lead to less scattered attention maps (as shown in our Figures~\ref{fig:attention_maps}--\ref{fig:examples_vqa}) and are able to select contiguous text segments or image regions. 
As such, they may provide better inductive bias for interpretation.

In addition, our attention densities using Gaussian and truncated paraboloids include a variance term, being potentially useful as a measure of confidence---for example, a large ellipse in an image may indicate that the model had little confidence about where it should attend to answer a question, while a small ellipse may denote high confidence on a particular object. 

We also see opportunities for research connecting our work with other continuous models \citep{wang2018deep,schutt2017schnet,chen2018neural} leading to end-to-end continuous models which, by avoiding discretization, have the potential to be less susceptible to adversarial attacks via input perturbations. 
Outside the machine learning field, the links drawn in \S\ref{sec:sparse_families} between sparse alternatives to softmax and models used in non-extensive (Tsallis) statistical physics suggest a potential benefit in that field too. 


Note, however, that our work is a first step into all these directions, and as such further investigation will be needed to better understand the potential benefits. We strongly recommend carrying out user studies before deploying any such system, to better understand the benefits and risks. Some of the examples in App.~\ref{sec:model_hyperparams} may help understand potential failure modes. 


We should also take into account that, for any computer vision model, there are important societal risks related to privacy-violating surveillance applications. Continuous attention holds the promise to scale to larger and multi-resolution images, which may, in the longer term, be deployed in such undesirable domains. 
Ethical concerns hold for natural language applications such as machine translation, where 
biases present in data can be arbitrarily augmented or hidden by machine learning systems.
For example, our natural language processing experiments mostly use English datasets (as a
target language in machine translation, and in document classification). 
Further work is needed to understand if our findings generalize to other languages. Likewise, in the vision experiments, the VQA-v2 dataset uses COCO images, which have documented biases \citep{wang2019balanced}. 
In line with the fundamental scope and early stage of this line of research, we deliberately choose applications on standard benchmark datasets, in an attempt to put as much distance as possible from malevolent applications. 
Finally, although we chose the most widely used evaluation metrics for each task (accuracy for document classification and visual question answering, BLEU for machine translation), these metrics do not always capture performance quality---for example, BLEU in machine translation is far from being a perfect metric.

The data, memory, and computation requirements for training systems with continuous attention do not seem considerably higher than the ones which use discrete attention. On the other hand, for NLP applications, our approach has the potential to better compress sequential data, by using fewer basis functions than the sequence length (as suggested by our document classification experiments). While there is nothing specific about our research that poses environmental concerns  or that promises to alleviate such concerns, our models share the same problematic property as other neural network models in terms of their energy consumption to train models and tune hyperparameters \citep{strubell2019energy}.

\begin{ack}
This work was %
supported by the European Research Council (ERC StG DeepSPIN 758969),
by the P2020 program MAIA (contract 045909), and by the Funda\c{c}\~ao para a Ci\^encia e Tecnologia 
through contract UIDB/50008/2020. 
We would like to thank Pedro Martins, Zita Marinho, and the anonymous reviewers for their helpful feedback.


\end{ack}


\bibliography{neurips_2020}

\begin{thebibliography}{57}
\providecommand{\natexlab}[1]{#1}
\providecommand{\url}[1]{\texttt{#1}}
\expandafter\ifx\csname urlstyle\endcsname\relax
  \providecommand{\doi}[1]{doi: #1}\else
  \providecommand{\doi}{doi: \begingroup \urlstyle{rm}\Url}\fi

\bibitem[Brown(1986)]{brown1986fundamentals}
Lawrence~D Brown.
\newblock \emph{Fundamentals of Statistical Exponential Families with
  Applications in Statistical Decision Theory}.
\newblock Institute of Mathematical Statistics, 1986.

\bibitem[Barndorff-Nielsen(2014)]{barndorff2014information}
Ole Barndorff-Nielsen.
\newblock \emph{Information and Exponential Families in Statistical Theory}.
\newblock John Wiley \& Sons, 2014.

\bibitem[Pitman(1936)]{pitman1936sufficient}
Edwin James~George Pitman.
\newblock Sufficient statistics and intrinsic accuracy.
\newblock In \emph{Mathematical Proceedings of the Cambridge Philosophical
  Society}, volume~32, pages 567--579. Cambridge University Press, 1936.

\bibitem[Darmois(1935)]{darmois1935lois}
Georges Darmois.
\newblock Sur les lois de probabilit{\'e}a estimation exhaustive.
\newblock \emph{CR Acad. Sci. Paris}, 260\penalty0 (1265):\penalty0 85, 1935.

\bibitem[Koopman(1936)]{koopman1936distributions}
Bernard~Osgood Koopman.
\newblock On distributions admitting a sufficient statistic.
\newblock \emph{Transactions of the American Mathematical society}, 39\penalty0
  (3):\penalty0 399--409, 1936.

\bibitem[Martins and Astudillo(2016)]{Martins2016ICML}
Andr{\'e}~FT Martins and Ram{\'o}n~Fernandez Astudillo.
\newblock {From softmax to sparsemax: A sparse model of attention and
  multi-label classification}.
\newblock In \emph{Proc. of ICML}, 2016.

\bibitem[Blondel et~al.(2020)Blondel, Martins, and
  Niculae]{blondel2020learning}
Mathieu Blondel, Andr{\'e}~FT Martins, and Vlad Niculae.
\newblock Learning with fenchel-young losses.
\newblock \emph{Journal of Machine Learning Research}, 21\penalty0
  (35):\penalty0 1--69, 2020.

\bibitem[Peters et~al.(2019)Peters, Niculae, and Martins]{peters2019sparse}
Ben Peters, Vlad Niculae, and Andr{\'e}~F.T. Martins.
\newblock Sparse sequence-to-sequence models.
\newblock In \emph{Proc. of ACL}, 2019.

\bibitem[Correia et~al.(2019)Correia, Niculae, and
  Martins]{correia2019adaptively}
Gon{\c{c}}alo~M Correia, Vlad Niculae, and Andr{\'e}~FT Martins.
\newblock Adaptively sparse transformers.
\newblock In \emph{Proceedings of the 2019 Conference on Empirical Methods in
  Natural Language Processing and the 9th International Joint Conference on
  Natural Language Processing (EMNLP-IJCNLP)}, pages 2174--2184, 2019.

\bibitem[Naudts(2009)]{naudts2009q}
Jan Naudts.
\newblock The q-exponential family in statistical physics.
\newblock \emph{Central European Journal of Physics}, 7\penalty0 (3):\penalty0
  405--413, 2009.

\bibitem[Sears(2008)]{sears2010generalized}
Timothy Sears.
\newblock \emph{Generalized Maximum Entropy, Convexity and Machine Learning}.
\newblock PhD thesis, The Australian National University, 2008.

\bibitem[Ding and Vishwanathan(2010)]{ding2010t}
Nan Ding and S.V.N. Vishwanathan.
\newblock t-logistic regression.
\newblock In J.~D. Lafferty, C.~K.~I. Williams, J.~Shawe-Taylor, R.~S. Zemel,
  and A.~Culotta, editors, \emph{Advances in Neural Information Processing
  Systems 23}, pages 514--522. Curran Associates, Inc., 2010.

\bibitem[Tsallis(1988)]{Tsallis1988}
Constantino Tsallis.
\newblock {Possible generalization of Boltzmann-Gibbs statistics}.
\newblock \emph{Journal of Statistical Physics}, 52:\penalty0 479--487, 1988.

\bibitem[Bahdanau et~al.(2015)Bahdanau, Cho, and Bengio]{bahdanau2014neural}
Dzmitry Bahdanau, Kyunghyun Cho, and Yoshua Bengio.
\newblock {Neural machine translation by jointly learning to align and
  translate}.
\newblock In \emph{Proc. of ICLR}, 2015.

\bibitem[Halmos(2013)]{halmos2013measure}
Paul~R Halmos.
\newblock \emph{Measure Theory}, volume~18.
\newblock Springer, 2013.

\bibitem[Bridle(1990)]{bridle1990probabilistic}
John~S. Bridle.
\newblock Probabilistic interpretation of feedforward classification network
  outputs, with relationships to statistical pattern recognition.
\newblock In Françoise Fogelman-Soulié and Jeanny Hérault, editors,
  \emph{Neurocomputing}, pages 227--236. Springer, 1990.

\bibitem[Cover and Thomas(2012)]{cover2012elements}
Thomas~M Cover and Joy~A Thomas.
\newblock \emph{Elements of Information Theory}.
\newblock John Wiley \& Sons, 2012.

\bibitem[Amari(2016)]{amari2016information}
Shun-ichi Amari.
\newblock \emph{Information geometry and its applications}, volume 194.
\newblock Springer, 2016.

\bibitem[Havrda and Charv{\'a}t(1967)]{havrda1967quantification}
Jan Havrda and Franti{\v{s}}ek Charv{\'a}t.
\newblock Quantification method of classification processes. concept of
  structural $ a $-entropy.
\newblock \emph{Kybernetika}, 3\penalty0 (1):\penalty0 30--35, 1967.

\bibitem[Jost(2006)]{Jost2006}
L.~Jost.
\newblock Entropy and diversity.
\newblock \emph{Oikos}, 113:\penalty0 363--–375, 2006.

\bibitem[Rao(1982)]{Rao1982}
R.A. Rao.
\newblock Gini-{Simpson} index of diversity: a characterization,
  generalization, and applications.
\newblock \emph{Utilitas Mathematics}, 21:\penalty0 273--282, 1982.

\bibitem[Figueiredo(2001)]{FigueiredoNIPS2001}
M.~Figueiredo.
\newblock Adaptive sparseness using {Jeffreys} prior.
\newblock In \emph{Proc. of NeurIPS}, pages 697--704, 2001.

\bibitem[Tipping(2001)]{TippingJMLR2001}
M.~Tipping.
\newblock Sparse {Bayesian} learning and the relevance vector machine.
\newblock \emph{Journal of Machine Learning Research}, 1:\penalty0 211--244,
  2001.

\bibitem[Matsuzoe and Ohara(2012)]{matsuzoe2012geometry}
Hiroshi Matsuzoe and Atsumi Ohara.
\newblock Geometry for q-exponential families.
\newblock In \emph{Recent Progress in Differential Geometry and its Related
  Fields}, pages 55--71. World Scientific, 2012.

\bibitem[Amari and Ohara(2011)]{amari2011geometry}
Shun-ichi Amari and Atsumi Ohara.
\newblock Geometry of q-exponential family of probability distributions.
\newblock \emph{Entropy}, 13\penalty0 (6):\penalty0 1170--1185, 2011.

\bibitem[Epanechnikov(1969)]{epanechnikov1969non}
Vassiliy~A Epanechnikov.
\newblock Non-parametric estimation of a multivariate probability density.
\newblock \emph{Theory of Probability \& Its Applications}, 14\penalty0
  (1):\penalty0 153--158, 1969.

\bibitem[Sukhbaatar et~al.(2015)Sukhbaatar, Weston, Fergus,
  et~al.]{sukhbaatar2015end}
Sainbayar Sukhbaatar, Jason Weston, Rob Fergus, et~al.
\newblock End-to-end memory networks.
\newblock In \emph{Advances in Neural Information Processing Systems}, pages
  2440--2448, 2015.

\bibitem[Vaswani et~al.(2017)Vaswani, Shazeer, Parmar, Uszkoreit, Jones, Gomez,
  Kaiser, and Polosukhin]{vaswani2017attention}
Ashish Vaswani, Noam Shazeer, Niki Parmar, Jakob Uszkoreit, Llion Jones,
  Aidan~N Gomez, {\L}ukasz Kaiser, and Illia Polosukhin.
\newblock {Attention is all you need}.
\newblock In \emph{Proc. of NeurIPS}, 2017.

\bibitem[Maas et~al.(2011)Maas, Daly, Pham, Huang, Ng, and
  Potts]{maas2011learning}
Andrew~L Maas, Raymond~E Daly, Peter~T Pham, Dan Huang, Andrew~Y Ng, and
  Christopher Potts.
\newblock Learning word vectors for sentiment analysis.
\newblock In \emph{Proc. of NAACL-HLT}, 2011.

\bibitem[Cettolo et~al.(2017)Cettolo, Federico, Bentivogli, Jan, Sebastian,
  Katsuitho, Koichiro, and Christian]{cettolo2017overview}
Mauro Cettolo, Marcello Federico, Luisa Bentivogli, Niehues Jan, St{\"u}ker
  Sebastian, Sudoh Katsuitho, Yoshino Koichiro, and Federmann Christian.
\newblock Overview of the {IWSLT} 2017 evaluation campaign.
\newblock In \emph{Proc. of IWSLT}, pages 2--14, 2017.

\bibitem[Goyal et~al.(2019)Goyal, Khot, Agrawal, Summers-Stay, Batra, and
  Parikh]{Goyal2019}
Yash Goyal, Tejas Khot, Aishwarya Agrawal, Douglas Summers-Stay, Dhruv Batra,
  and Devi Parikh.
\newblock {Making the V in VQA Matter: Elevating the Role of Image
  Understanding in Visual Question Answering}.
\newblock \emph{International Journal of Computer Vision}, 127\penalty0
  (4):\penalty0 398--414, 2019.

\bibitem[Yu et~al.(2019)Yu, Yu, Cui, Tao, and Tian]{Yu2019}
Zhou Yu, Jun Yu, Yuhao Cui, Dacheng Tao, and Qi~Tian.
\newblock {Deep modular co-attention networks for visual question answering}.
\newblock \emph{Proceedings of the IEEE Computer Society Conference on Computer
  Vision and Pattern Recognition}, pages 6274--6283, 2019.

\bibitem[Anderson et~al.(2018{\natexlab{a}})Anderson, He, Buehler, Teney,
  Johnson, Gould, and Zhang]{anderson2018bottom}
Peter Anderson, Xiaodong He, Chris Buehler, Damien Teney, Mark Johnson, Stephen
  Gould, and Lei Zhang.
\newblock Bottom-up and top-down attention for image captioning and visual
  question answering.
\newblock In \emph{Proceedings of the IEEE conference on computer vision and
  pattern recognition}, pages 6077--6086, 2018{\natexlab{a}}.

\bibitem[Wang et~al.(2018)Wang, Suo, Ma, Pokrovsky, and Urtasun]{wang2018deep}
Shenlong Wang, Simon Suo, Wei-Chiu Ma, Andrei Pokrovsky, and Raquel Urtasun.
\newblock Deep parametric continuous convolutional neural networks.
\newblock In \emph{Proceedings of the IEEE Conference on Computer Vision and
  Pattern Recognition}, pages 2589--2597, 2018.

\bibitem[Sch{\"u}tt et~al.(2017)Sch{\"u}tt, Kindermans, Felix, Chmiela,
  Tkatchenko, and M{\"u}ller]{schutt2017schnet}
Kristof Sch{\"u}tt, Pieter-Jan Kindermans, Huziel Enoc~Sauceda Felix, Stefan
  Chmiela, Alexandre Tkatchenko, and Klaus-Robert M{\"u}ller.
\newblock Schnet: A continuous-filter convolutional neural network for modeling
  quantum interactions.
\newblock In \emph{Advances in neural information processing systems}, pages
  991--1001, 2017.

\bibitem[Funahashi and Nakamura(1993)]{funahashi1993approximation}
Ken-ichi Funahashi and Yuichi Nakamura.
\newblock Approximation of dynamical systems by continuous time recurrent
  neural networks.
\newblock \emph{Neural networks}, 6\penalty0 (6):\penalty0 801--806, 1993.

\bibitem[Rubanova et~al.(2019)Rubanova, Chen, and Duvenaud]{rubanova2019latent}
Yulia Rubanova, Tian~Qi Chen, and David~K Duvenaud.
\newblock Latent ordinary differential equations for irregularly-sampled time
  series.
\newblock In \emph{Advances in Neural Information Processing Systems}, pages
  5321--5331, 2019.

\bibitem[Chen et~al.(2018)Chen, Rubanova, Bettencourt, and
  Duvenaud]{chen2018neural}
Tian~Qi Chen, Yulia Rubanova, Jesse Bettencourt, and David~K Duvenaud.
\newblock Neural ordinary differential equations.
\newblock In \emph{Advances in neural information processing systems}, pages
  6571--6583, 2018.

\bibitem[Kumar and Tsvetkov(2018)]{kumar2018mises}
Sachin Kumar and Yulia Tsvetkov.
\newblock Von mises-fisher loss for training sequence to sequence models with
  continuous outputs.
\newblock In \emph{Proc. of ICLR}, 2018.

\bibitem[Cordonnier et~al.(2019)Cordonnier, Loukas, and
  Jaggi]{cordonnier2020relationship}
Jean-Baptiste Cordonnier, Andreas Loukas, and Martin Jaggi.
\newblock On the relationship between self-attention and convolutional layers.
\newblock In \emph{International Conference on Learning Representations}, 2019.

\bibitem[You et~al.(2020)You, Sun, and Iyyer]{you-etal-2020-hard}
Weiqiu You, Simeng Sun, and Mohit Iyyer.
\newblock Hard-coded {G}aussian attention for neural machine translation.
\newblock In \emph{Proceedings of the 58th Annual Meeting of the Association
  for Computational Linguistics}, pages 7689--7700, Online, July 2020.
  Association for Computational Linguistics.
\newblock \doi{10.18653/v1/2020.acl-main.687}.
\newblock URL \url{https://www.aclweb.org/anthology/2020.acl-main.687}.

\bibitem[Gregor et~al.(2015)Gregor, Danihelka, Graves, Rezende, and
  Wierstra]{gregor2015draw}
Karol Gregor, Ivo Danihelka, Alex Graves, Danilo Rezende, and Daan Wierstra.
\newblock Draw: A recurrent neural network for image generation.
\newblock In \emph{International Conference on Machine Learning}, pages
  1462--1471, 2015.

\bibitem[Jain and Wallace(2019)]{jain2019attention}
Sarthak Jain and Byron~C Wallace.
\newblock Attention is not explanation.
\newblock In \emph{Proc. NAACL-HLT}, 2019.

\bibitem[Serrano and Smith(2019)]{serrano2019attention}
Sofia Serrano and Noah~A Smith.
\newblock Is attention interpretable?
\newblock In \emph{Proc. ACL}, 2019.

\bibitem[Wiegreffe and Pinter(2019)]{wiegreffe2019attention}
Sarah Wiegreffe and Yuval Pinter.
\newblock Attention is not not explanation.
\newblock In \emph{Proc. EMNLP-IJCNLP}, 2019.

\bibitem[Wang et~al.(2019)Wang, Zhao, Yatskar, Chang, and
  Ordonez]{wang2019balanced}
Tianlu Wang, Jieyu Zhao, Mark Yatskar, Kai-Wei Chang, and Vicente Ordonez.
\newblock Balanced datasets are not enough: Estimating and mitigating gender
  bias in deep image representations.
\newblock In \emph{Proceedings of the IEEE International Conference on Computer
  Vision}, pages 5310--5319, 2019.

\bibitem[Strubell et~al.(2019)Strubell, Ganesh, and
  McCallum]{strubell2019energy}
Emma Strubell, Ananya Ganesh, and Andrew McCallum.
\newblock Energy and policy considerations for deep learning in nlp.
\newblock In \emph{Proceedings of the 57th Annual Meeting of the Association
  for Computational Linguistics}, pages 3645--3650, 2019.

\bibitem[Bauschke and Combettes(2011)]{Bauschke_Combettes2011}
Heinz Bauschke and Patrick Combettes.
\newblock \emph{Convex Analysis and Monotone Operator Theory in Hilbert
  Spaces}.
\newblock Springer, 2011.

\bibitem[Bregman(1967)]{bregman1967relaxation}
Lev~M Bregman.
\newblock The relaxation method of finding the common point of convex sets and
  its application to the solution of problems in convex programming.
\newblock \emph{USSR Computational Mathematics and Mathematical Physics},
  7\penalty0 (3):\penalty0 200--217, 1967.

\bibitem[Abe(2003)]{abe2003geometry}
Sumiyoshi Abe.
\newblock Geometry of escort distributions.
\newblock \emph{Physical Review E}, 68\penalty0 (3):\penalty0 031101, 2003.

\bibitem[Jaynes(1957)]{jaynes1957information}
Edwin~T Jaynes.
\newblock Information theory and statistical mechanics.
\newblock \emph{Physical review}, 106\penalty0 (4):\penalty0 620, 1957.

\bibitem[Sennrich et~al.(2016)Sennrich, Haddow, and Birch]{sennrich2016neural}
Rico Sennrich, Barry Haddow, and Alexandra Birch.
\newblock Neural machine translation of rare words with subword units.
\newblock In \emph{Proc. of ACL}, 2016.

\bibitem[Kreutzer et~al.(2019)Kreutzer, Bastings, and
  Riezler]{kreutzer2019joey}
Julia Kreutzer, Joost Bastings, and Stefan Riezler.
\newblock Joey nmt: A minimalist nmt toolkit for novices.
\newblock In \emph{Proc. of EMNLP-IJCNLP: System Demonstrations}, pages
  109--114, 2019.

\bibitem[He et~al.(2016)He, Zhang, Ren, and Sun]{He2016}
Kaiming He, Xiangyu Zhang, Shaoqing Ren, and Jian Sun.
\newblock {Deep residual learning for image recognition}.
\newblock \emph{Proceedings of the IEEE Computer Society Conference on Computer
  Vision and Pattern Recognition}, 2016-Decem:\penalty0 770--778, 2016.

\bibitem[Russakovsky et~al.(2015)Russakovsky, Deng, Su, Krause, Satheesh, Ma,
  Huang, Karpathy, Khosla, Bernstein, Berg, and Fei-Fei]{Russakovsky2015}
Olga Russakovsky, Jia Deng, Hao Su, Jonathan Krause, Sanjeev Satheesh, Sean Ma,
  Zhiheng Huang, Andrej Karpathy, Aditya Khosla, Michael Bernstein,
  Alexander~C. Berg, and Li~Fei-Fei.
\newblock {ImageNet Large Scale Visual Recognition Challenge}.
\newblock \emph{International Journal of Computer Vision}, 115\penalty0
  (3):\penalty0 211--252, 2015.

\bibitem[Anderson et~al.(2018{\natexlab{b}})Anderson, He, Buehler, Teney,
  Johnson, Gould, and Zhang]{Anderson2018}
Peter Anderson, Xiaodong He, Chris Buehler, Damien Teney, Mark Johnson, Stephen
  Gould, and Lei Zhang.
\newblock {Bottom-Up and Top-Down Attention for Image Captioning and Visual
  Question Answering}.
\newblock \emph{Proceedings of the IEEE Computer Society Conference on Computer
  Vision and Pattern Recognition}, pages 6077--6086, 2018{\natexlab{b}}.

\bibitem[Pennington et~al.(2014)Pennington, Socher, and
  Manning]{pennington2014glove}
Jeffrey Pennington, Richard Socher, and Christopher~D. Manning.
\newblock Glove: Global vectors for word representation.
\newblock In \emph{Proc. of EMNLP}, pages 1532--1543, 2014.

\end{thebibliography}
\bibliographystyle{unsrtnat}

\newpage
\onecolumn

\appendix

\bigskip

\begin{center}
\LARGE{\bf Supplemental Material}
\end{center}

\section{Differential Negentropy and Boltzmann-Gibbs distributions}\label{sec:diff_ent_exp_family}

We adapt a proof from \citet{cover2012elements}.
Let $\Omega$ be the Shannon negentropy, which is proper, lower semi-continuous, and strictly convex \citep[example~9.41]{Bauschke_Combettes2011}, and let  $$\mathrm{KL}(p \| q) := \int_S p(t) \log \frac{p(t)}{q(t)}$$ be the Kullback-Leibler divergence between distributions $p$ and $q$ (which is always non-negative and equals $0$ iff $p=q$).
Take $q(t) = \frac{\exp(f(t))}{\int_S \exp(f(t'))d\nu(t')} = \exp(f(t) - A(f))$ as in \eqref{eq:boltzmann}, where $A(f)$ is the log-partition function.

We have, for any $p \in \mathcal{M}_+^1(S)$:
\begin{eqnarray}
0 &\le& \mathrm{KL}(p \| q) 
= \int_S p(t) \log \frac{p(t)}{q(t)} 
= \Omega(p) - \int_S p(t) \log q(t) 
= \Omega(p) - \int_S p(t) (f(t) - A(f)) \nonumber\\
&=& \Omega(p) - \mathbb{E}_p[f(t)]  + A(f).
\end{eqnarray}
Therefore, we have, for any  $p \in \mathcal{M}_+^1(S)$, that
\begin{equation}
\mathbb{E}_p[f(t)]  - \Omega(p) \le A(f),
\end{equation}
with equality if and only if $p=q$. Since the right hand side is constant with respect to $p$, we have that the posited $q$ must be the maximizer of \eqref{eq:reg_prediction}.

\section{Tsallis Negentropy and  Sparse Distributions}\label{sec:gini_ent_sparse_family}

\subsection{Shannon as a limit case of Tsallis when $\alpha\rightarrow 1$}

We show that $\lim_{\alpha \rightarrow 1} \Omega_{\alpha}(p) = \Omega_1(p)$ for any $p\in \mathcal{M}_+^1(S)$. 
From \eqref{eq:tsallis}, it suffices to show that $\lim_{\beta\rightarrow 1} \log_\beta (u) = \log(u)$ for any $u \ge 0$. 
Let $g(\beta) \coloneqq u^{1-\beta} - 1$, and
$h(\beta)\coloneqq 1 - \beta$. Observe that $$\lim_{\beta\rightarrow 1}\log_\beta (u) = \lim_{\beta\rightarrow 1} \frac{g(\beta)}{h(\beta)} =
\frac{g(1)}{h(1)}
= \frac{0}{0},$$
so we are in an indeterminate case.
We take the derivatives of $g$ and $h$: 
\begin{equation}
g'(\beta) = \left(\exp (\log u^{1-\beta}) \right)'
= \exp (\log u^{1-\beta}) \cdot ((1-\beta) \log u)'
= -u^{1-\beta} \log u,
\end{equation}
and $h'(\beta) = -1$. 
From l'H\^{o}pital's rule,
\begin{equation}
\lim_{\beta \rightarrow 1} \frac{g(\beta)}{h(\beta)}
=\lim_{\beta \rightarrow 1} \frac{g'(\beta)}{h'(\beta)}
= \log u.
\end{equation}

\subsection{Proof of Proposition~\ref{prop:solution_rpm_tsallis}}

The proof of Proposition~\ref{prop:solution_rpm_tsallis} is  similar to the one in \S\ref{sec:diff_ent_exp_family}, replacing the KL divergence by the Bregman divergence induced by $\Omega_\alpha$, and using an additional bound.
Let $$B_{\Omega_\alpha}(p, q) := \Omega_\alpha(p) - \Omega_\alpha(q) - \langle \nabla \Omega_\alpha(q), p-q \rangle$$ be the (functional) Bregman divergence between distributions $p$ and $q$ induced by $\Omega_\alpha$, and let $$q(t) = \exp_{2-\alpha}(f(t) - A_\alpha(f)) = [1 + (\alpha-1)(f(t) - A_\alpha(f))]_+^{\frac{1}{\alpha-1}}.$$ 
Note that, from \eqref{eq:tsallis}, $$\left(\nabla_q \Omega_{\alpha}(q)\right)(t) = \frac{q(t)^{\alpha-1}}{\alpha-1}.$$ 
From the non-negativity of the Bregman divergence \cite{bregman1967relaxation}, we have, for any $p \in M_+^1(S)$:
\begin{eqnarray}
0 &\le^{(a)}& B_{\Omega_\alpha}(p, q) \nonumber\\
&=& \Omega_\alpha(p) - \Omega_\alpha(q) - \langle \nabla \Omega_\alpha(q), p-q \rangle \nonumber\\
&=& \Omega_\alpha(p) - \Omega_\alpha(q) - \int_S \frac{q(t)^{\alpha-1}}{\alpha-1} (p(t) - q(t)) \nonumber\\
&=& \Omega_\alpha(p) - \Omega_\alpha(q) - \underbrace{\mathbb{E}_p[[f(t) - A_\alpha(f) + (\alpha-1)^{-1}]_+]}_{\ge \mathbb{E}_p[f(t) - A_\alpha(f) + (\alpha-1)^{-1}]} + \frac{1}{\alpha-1}\int_S q(t)^{\alpha} \nonumber\\
&\le^{(b)}& \Omega_\alpha(p) - \Omega_\alpha(q) - \mathbb{E}_p[f(t) - A_\alpha(f) + (\alpha-1)^{-1}] + \frac{1}{\alpha-1}\int_S q(t)^{\alpha} \nonumber\\
&=& \Omega_\alpha(p) - \mathbb{E}_p[f(t)]  - \Omega_\alpha(q) + \underbrace{\frac{1}{\alpha-1}\left(\int_S q(t)^{\alpha} - 1\right)}_{=\alpha \Omega_\alpha(q)} +A_\alpha(f) \nonumber\\
&=& \Omega_\alpha(p) - \mathbb{E}_p[f(t)]  + (\alpha-1)\Omega_\alpha(q) + A_\alpha(f).
\end{eqnarray}

Therefore, we have, for any $p \in M_+^1(S)$,
\begin{equation}\label{eq:variational_proof_tsallis}
\mathbb{E}_p[f(t)] - \Omega_\alpha(p) \le (\alpha-1)\Omega_\alpha(q) +A_\alpha(f),
\end{equation}
with equality iff $p = q$, which leads to zero Bregman divergence ({\it i.e.}, a tight inequality $(a)$) and to 
$\mathbb{E}_p[[f(t) - A_\alpha(f) + (\alpha-1)^{-1}]_+] = \mathbb{E}_p[f(t) - A_\alpha(f) + (\alpha-1)^{-1}]$ ({\it i.e.}, a tight inequality $(b)$).

We can use the equality above to obtain an expression for the Fenchel conjugate $\Omega_\alpha^*(f) = \mathbb{E}_q[f(t)] - \Omega_\alpha(q)$ ({\it i.e.}, the value of the maximum in \eqref{eq:reg_prediction} and the right hand side in \eqref{eq:variational_proof_tsallis}):
\begin{equation}
\Omega_\alpha^*(f) = (\alpha-1)\Omega_{\alpha}(q) + A_\alpha(f).
\end{equation}

\subsection{Normalizing function $A_\alpha(f)$}\label{sec:A_alpha}

Let $p = \hat{p}_{\Omega_\alpha}[f]$. 
The expression for $A_\alpha$ in Prop.~\ref{prop:solution_rpm_tsallis} is
obtained by inverting \eqref{eq:entmax}, yielding
$A_\alpha(f) = f(t) - \log_{2-\alpha}(p(t))$,  and integrating with respect to  $p(t)^{2-\alpha} d\nu(t)$, leading to:
\begin{eqnarray}
\int_S p_\theta(t)^{2-\alpha} A_\alpha(f) &=& \int_S p(t)^{2-\alpha} f(t) - \int_S p(t)^{2-\alpha} \log_{2-\alpha}(p(t))\nonumber\\
&=& \int_S p_\theta(t)^{2-\alpha} f(t) - \frac{\int_S  (p(t) - p(t)^{2-\alpha})}{\alpha-1} \nonumber\\
&=& \int_S p(t)^{2-\alpha} f(t) - \frac{1}{\alpha-1} + \frac{\int_S  p(t)^{2-\alpha}}{\alpha-1},
\end{eqnarray}
from which the desired expression follows.

\section{Relation to the Tsallis Maxent Principle}\label{sec:tsallis_maxent}

We discuss here the relation between the $(2-\alpha)$-exponential family of distributions as presented in Prop.~\ref{prop:solution_rpm_tsallis} and the distributions arising from the Tsallis maxent principle \citep{Tsallis1988}. We put in perspective the related work in statistical physics \citep{abe2003geometry,naudts2009q}, information geometry \citep{amari2011geometry,amari2016information}, and the discrete case presented in the machine learning literature \citep{blondel2020learning,peters2019sparse}. 

We start by noting that our $\alpha$ parameter matches the $\alpha$ used in prior machine learning literature related to the ``$\alpha$-entmax transformation'' \citep{blondel2020learning,peters2019sparse}. In the definition of Tsallis entropies  \eqref{eq:tsallis}, our $\alpha$ corresponds to the entropic index $q$ defined by \citet{Tsallis1988}. 
However, our $(2-\alpha)$-exponential families correspond to the $q$-exponential families as defined by \citet{naudts2009q}, and to the $t$-exponential families described by \citet{ding2010t} (which include the $t$-Student distribution). The family of Amari's $\alpha$-divergences relates to this $q$ as $\alpha=2q-1$ \citep[\S4.3]{amari2016information}. 

These differences in notation have historical reasons, and they are explained by the different ways in which Tsallis entropies relate to $q$-exponential families. 
In fact, the physics literature has defined $q$-exponential distributions in two distinct ways, as we next describe. 

Note first that the $\Omega$-RPM in our  Def.~\ref{def:regularized_prediction} is a generalization of the free energy variational principle, if we see $-f_\theta(t) = -\theta^\top \phi(t)$ as an energy function and $\Omega$ the entropy scaled by a temperature. 
Let $\Omega = \Omega_\alpha$ be the Tsallis $\alpha$-entropy. 
An equivalent constrained version of this problem is the maximum entropy ({\it maxent}) principle \citep{jaynes1957information}:
\begin{equation}\label{eq:constrained_maxent}
\max_{p \in \mathcal{M}_+^1(S)}  -\Omega_\alpha(p), \hspace{0.5cm} \mathrm{s.t.}  \hspace{0.35cm} \mathbb{E}_{p}[\phi(t)] = b.
\end{equation}
The solution of this problem corresponds to a distribution in the $(2-\alpha)$-exponential family \eqref{eq:entmax}:
\begin{equation}\label{eq:sol_maxent}
    p^\star(t) = \exp_{2-\alpha}(\theta^\top \phi(t) - A_\alpha(\theta)),
\end{equation}
for some Lagrange multiplier $\theta$. 

However, this construction differs from the one by \citet{Tsallis1988} and others, who use {\it escort distributions} (Eq.~\ref{eq:escort}) in the expectation constraints. Namely, instead of \eqref{eq:constrained_maxent}, they consider the problem:
\begin{equation}\label{eq:constrained_maxent_tsallis}
\max_{p \in \mathcal{M}_+^1(S)}  -\Omega_\alpha(p), \hspace{0.5cm} \mathrm{s.t.}  \hspace{0.35cm} \mathbb{E}_{\textcolor{blue}{\tilde{p}^{\alpha}}}[\phi(t)] = b.
\end{equation}
The solution of \eqref{eq:constrained_maxent_tsallis} is of the form
\begin{equation}\label{eq:sol_maxent_tsallis}
    p^\star(t) = B_{\alpha}(\theta)\exp_{\alpha}(\theta^\top (\phi(t) - b)),
\end{equation}
where $\theta$ is again a Lagrange multiplier. This is derived, for example, in \citep[Eq.~15]{abe2003geometry}. 
There are two main differences between \eqref{eq:sol_maxent} and \eqref{eq:sol_maxent_tsallis}:
\begin{itemize}
    \item While \eqref{eq:sol_maxent} involves the $(2-\alpha)$-exponential, \eqref{eq:sol_maxent_tsallis} involves the $\alpha$-exponential.
    \item In \eqref{eq:sol_maxent}, the normalizing term $A_\alpha(\theta)$ is {\it inside} the $(2-\alpha)$-exponential. In \eqref{eq:sol_maxent_tsallis}, there is an normalizing factor $B_{\alpha}(\theta)$ {\it outside} the $\alpha$-exponential. 
\end{itemize}
Naturally, when $\alpha=1$, these two problems become equivalent, since an additive term inside the exponential is equivalent to a multiplicative term outside. However, this does {\it not} happen with $\beta$-exponentials ($\exp_\beta(u+v) \ne \exp_\beta(u) \exp_\beta(v)$ in general, for $\beta \ne 1$), and therefore these two alternative paths lead to two different definitions of $q$-exponential families. Unfortunately, both have been considered in the physics literature, under the same name, and  
this has been subject of debate. Quoting \citet[\S 1]{naudts2009q}: 

\begin{quote}
{\it ``An important question is then whether in the modification the normalization should stand in front of the deformed exponential function, or whether it should be included as $\ln Z(\beta)$ inside. From the general formalism mentioned above it follows
that the latter is the right way to go.''}
\end{quote}

Throughout our paper, we use the definition of \citep{naudts2009q,amari2011geometry}, equivalent to the maxent problem \eqref{eq:constrained_maxent}. 

\section{Proof of Proposition~\ref{prop:gradient_A}}\label{sec:proof_gradient_A}

We adapt the proof from \citet[Theorem 5]{amari2011geometry}. 
Note first that, for $t \in \mathrm{supp}(p_\theta)$,
\begin{eqnarray}
\nabla_\theta p_\theta(t) &=& \nabla_\theta [(\alpha-1)(\theta^\top \phi(t) - A_\alpha(\theta))+1]^{1/(\alpha-1)} \nonumber\\
&=&  [(\alpha-1)(\theta^\top \phi(t) - A_\alpha(\theta))+1]^{(2-\alpha)/(\alpha-1)} (\phi(t) - \nabla_\theta A_\alpha(\theta)) \nonumber\\
&=& p_\theta(t)^{2-\alpha} (\phi(t) - \nabla_\theta A_\alpha(\theta)),
\end{eqnarray}
and
\begin{eqnarray}
\nabla^2_\theta p_\theta(t) &=& 
\nabla_\theta p_\theta^{2-\alpha}(t) (\phi(t) - \nabla_\theta A_\alpha(\theta))^\top - p_\theta^{2-\alpha}(t) \nabla^2_\theta A_\alpha(\theta) \nonumber\\
&=& (2-\alpha) p_\theta^{1-\alpha}(t) \nabla_\theta  p_\theta(t) (\phi(t) - \nabla_\theta A_\alpha(\theta))^\top - p_\theta^{2-\alpha}(t) \nabla^2_\theta A_\alpha(\theta) \nonumber\\
&=& (2-\alpha)p_\theta(t)^{3-2\alpha}\bigl(\phi(t) - \nabla_\theta A_\alpha(\theta)\bigr) \bigl(\phi(t) - \nabla_\theta A_\alpha(\theta)\bigr)^\top \nonumber\\
&& - p_\theta(t)^{2-\alpha} \nabla^2_\theta A_\alpha(\theta).
\end{eqnarray}

Therefore we have:
\begin{equation}
0 = \nabla_\theta \underbrace{\int_S  p_\theta(t)}_{=1} = \int_S \nabla_\theta p_\theta(t) = \int_S p_\theta(t)^{2-\alpha} (\phi(t) - \nabla_\theta A_\alpha(\theta)),
\end{equation}
from which we obtain
\begin{equation}
    \nabla_\theta A_\alpha(\theta) = \frac{\int_S p_\theta(t)^{2-\alpha} \phi(t)}{\int_S p_\theta(t)^{2-\alpha}}.
\end{equation}

To prove that $A_\alpha(\theta)$ is convex, we will show that its Hessian is positive semidefinite. Note that
\begin{eqnarray}
0 &=& \nabla^2_\theta \underbrace{\int_S  p_\theta(t)}_{=1} = \int_S \nabla^2_\theta p_\theta(t) \nonumber\\
&=& \int_S (2-\alpha)p_\theta(t)^{3-2\alpha}\bigl(\phi(t) - \nabla_\theta A_\alpha(\theta)\bigr) \bigl(\phi(t) - \nabla_\theta A_\alpha(\theta)\bigr)^\top  - p_\theta(t)^{2-\alpha} \nabla^2_\theta A_\alpha(\theta)\nonumber\\
&=& (2-\alpha) \int_S p_\theta(t)^{3-2\alpha}\bigl(\phi(t) - \nabla_\theta A_\alpha(\theta)\bigr) \bigl(\phi(t) - \nabla_\theta A_\alpha(\theta)\bigr)^\top  \nonumber\\
&& - \nabla^2_\theta A_\alpha(\theta) \int_S p_\theta(t)^{2-\alpha},
\end{eqnarray}
hence, for $\alpha \le 2$,
\begin{equation}
    \nabla^2_\theta A_\alpha(\theta) = \frac{(2-\alpha) \int_S p_\theta(t)^{3-2\alpha}\overbrace{\bigl(\phi(t) - \nabla_\theta A_\alpha(\theta)\bigr) \bigl(\phi(t) - \nabla_\theta A_\alpha(\theta)\bigr)^\top}^{\succeq 0}}{\int_S p_\theta(t)^{2-\alpha}} \succeq 0,
\end{equation}
where we used the fact that $p_\theta(t) \ge 0$ for $t \in S$ and that integrals of positive semidefinite functions and positive semidefinite.

\section{Normalization Constants for Continuous Sparsemax Distributions}\label{sec:normalization_constants}

\subsection{Truncated parabola}\label{sec:proof_truncated_parabola}

Let $p(t) = \left[-\lambda - \frac{(t-\mu)^2}{2\sigma^2}\right]_+$ as in \eqref{eq:truncated_parabola}.
Let us determine the constant $\lambda$ that ensures this distribution normalizes to 1. Note that $\lambda$ does not depend on the location parameter $\mu$, hence we can assume $\mu=0$ without loss of generality. We must have
$\lambda = -\frac{a^2}{2\sigma^2}$ and $1 = \int_{-a}^{a} \left(-\lambda - \frac{x^2}{2\sigma^2}\right) = -2\lambda a - \frac{a^3}{3\sigma^2} = \frac{2a^3}{3\sigma^2}$, hence
$a = \left(\frac{3}{2}\sigma^2\right)^{1/3}$, which finally gives:
\begin{equation}\label{eq:lambda_gaussian_proof}
\lambda = -\frac{1}{2}\left(\frac{3}{2\sigma}\right)^{2/3}.
\end{equation}

\remove{
The Gini negentropy of this distribution is
\begin{eqnarray}
\Omega_2(\hat{p}_{\Omega_2}[f]) &=& -\frac{1}{2} + \frac{1}{2}\int \hat{p}^2_{\Omega_2}[f](x)\nonumber\\
&=& -\frac{1}{2} + \frac{1}{2}\int_{-a}^a \left(-\lambda - \frac{x^2}{2\sigma^2}\right)^2\nonumber\\
&=& -\frac{1}{2} - \lambda^2 a + \frac{\lambda a^3}{3\sigma^2} + \frac{a^5}{20\sigma^4}\nonumber\\
&=& -\frac{1}{2} + \frac{a^5}{4\sigma^4} - \frac{a^5}{6\sigma^4} + \frac{a^5}{20\sigma^4}\nonumber\\
&=& -\frac{1}{2} + \frac{2a^5}{15\sigma^4}\nonumber\\
&=& -\frac{1}{2} + \frac{1}{5}\left(\frac{3}{2\sigma}\right)^{2/3}.
\end{eqnarray}
}

\subsection{Multivariate truncated paraboloid}\label{sec:proof_truncated_paraboloid}

Let $p(t) = \left[-\lambda - \frac{1}{2}(t-\mu)\Sigma^{-1}(t-\mu)\right]_+$ as in \eqref{eq:truncated_paraboloid}.
Let us determine the constant $\lambda$ that ensures this distribution normalizes to 1, where we assume again $\mu=0$ without loss of generality. To obtain $\lambda$, we start by invoking the formula for computing the volume of an ellipsoid defined
by the equation $x^\top \Sigma^{-1} x \le 1$:
\begin{equation}
V_{\mathrm{ell}}(\Sigma) = \frac{\pi^{n/2}}{\Gamma(n/2 + 1)} \mathrm{det}(\Sigma)^{1/2},
\end{equation}
where $\Gamma(t)$ is the Gamma function.
Since each slice of a paraboloid is an ellipsoid, we can apply Cavalieri's principle to obtain the volume of a paraboloid $y=\frac{1}{2} x^\top \Sigma^{-1} x$ of height $h = -\lambda$ as follows:
\begin{eqnarray}
V_{\mathrm{par}}(h) &=& \int_{0}^{h} V_{\mathrm{ell}}(2 \Sigma y)dy 
\,\,=\,\, \frac{(2\pi)^{n/2}\mathrm{det}(\Sigma)^{1/2}}{\Gamma(\frac{n}{2} + 1)}  \int_{0}^{h} y^{\frac{n}{2}}dy\nonumber\\ 
&=& \frac{(2\pi)^{n/2}\mathrm{det}(\Sigma)^{1/2}}{(\frac{n}{2} + 1)\Gamma(\frac{n}{2} + 1)}  h^{\frac{n}{2}+1}\nonumber\\
&=& \frac{\sqrt{(2\pi)^{n}\mathrm{det}(\Sigma)}}{\Gamma(\frac{n}{2} + 2)}  h^{\frac{n}{2}+1}.
\end{eqnarray}
Equating the volume to 1, we obtain $\lambda = -h$ as:
\begin{equation}
\lambda = -\left(\frac{\Gamma(\frac{n}{2} + 2)}{\sqrt{(2\pi)^{n}\mathrm{det}(\Sigma)}}\right)^{\frac{2}{2+n}}.
\end{equation}

\subsection{Triangular}\label{sec:proof_triangular}

Let $p(t) = \left[-\lambda - \frac{|t-\mu|}{b}\right]_+$ as in \eqref{eq:triangular}.
Let us determine the constant $\lambda$ that ensures this distribution normalizes to 1.
Assuming again $\mu=0$ without loss of generality, we must have
$\lambda = -\frac{a}{b}$ and $1 = \int_{-a}^{a} \left(-\lambda - \frac{|x|}{b}\right) = -2\lambda a - \frac{a^2}{b} = \frac{a^2}{b}$, hence
$a = \sqrt{b}$, which finally gives
$\lambda = -b^{-1/2}$.

\remove{
The negentropy of this distribution is
\begin{eqnarray}
\Omega_2(\hat{p}_{\Omega_2}[f]) &=& -\frac{1}{2} + \frac{1}{2}\int \hat{p}^2_{\Omega_2}[f](x)\nonumber\\
&=& -\frac{1}{2} + \frac{1}{2}\int_{-a}^a \left(-\lambda - \frac{|x|}{b}\right)^2\nonumber\\
&=& -\frac{1}{2} + \frac{1}{2}\int_{-a}^a \left(\lambda^2 +\frac{2\lambda|x|}{b} + \frac{x^2}{b^2} \right)\nonumber\\
&=& -\frac{1}{2} + \lambda^2 a + \frac{\lambda a^2}{b} + \frac{\lambda a^3}{3b^2} \nonumber\\
&=& -\frac{1}{2} +  \frac{a^3}{b^2} - \frac{a^3}{b^2} + \frac{a^3}{3b^2} \nonumber\\
&=& -\frac{1}{2} +  \frac{1}{3\sqrt{b}}.
\end{eqnarray}
}

\subsection{Location-scale families}\label{sec:proof_location_scale}

We first show that $a$ is the solution of the equation $ag'(a) - g(a) + g(0) = \frac{1}{2}$.
From symmetry around $\mu$, we must have
\begin{eqnarray}
\frac{1}{2} = \int_{\mu}^{\mu + a\sigma} \left(\frac{1}{\sigma}g'(a) - \frac{1}{\sigma}g'\left( \frac{t-\mu}{\sigma}\right)\right) dt = \int_{0}^{a} \left(g'(a) - g'(s)\right)  ds = ag'(a) - g(a) + g(0),
\end{eqnarray}
where we made a variable substitution $s = (t-\mu)/\sigma$,
which proves the desired result.
Now we show that a solution always exists if $g$ is strongly convex, {\it i.e.}, if there is some $\gamma>0$ such that
$g(0) \ge g(s) - sg'(s) + \frac{\gamma}{2}s^2$ for any $s \ge 0$.
Let $F(s) := sg'(s) - g(s) + g(0)$. We want to show that the equation $F(a) = \frac{1}{2}$ has a solution. Since $g$ is continuously differentiable, $F$ is continuous.
From the strong convexity of $g$, we have that $F(s) \ge \frac{\gamma}{2}s^2$ for any $s \ge 0$, which implies that $\lim_{s\rightarrow +\infty} F(s) = +\infty$.
Therefore, since $F(0) = 0$, we have by the intermediate value theorem that there must be some $a$ such that $F(a) = \frac{1}{2}$.

\section{Proof of Proposition~\ref{prop:jacobian_entmax}}\label{sec:proof_jacobian_entmax}

We have
\begin{eqnarray}
\nabla_\theta \mathbb{E}_{p}[\psi_i(t)] &=& \nabla_\theta \int_S p_\theta(t) \psi_i(t) 
= \int_S \nabla_\theta p_\theta(t) \psi_i(t) \nonumber\\
&=& \int_S p_\theta^{2-\alpha}(t) \nabla_\theta \log_{2-\alpha}(p_\theta(t)) \psi_i(t) \nonumber\\
&=& \int_S p_\theta^{2-\alpha}(t) \nabla_\theta (\theta^\top \phi(t) - A_\alpha(\theta)) \psi_i(t) \nonumber\\
&=& \int_S p_\theta^{2-\alpha}(t) (\phi(t) - \nabla_\theta A_\alpha(\theta)) \psi_i(t).
\end{eqnarray}
Using the expression for $\nabla_\theta A_\alpha(\theta)$ from Proposition~\ref{prop:gradient_A} yields the desired result.

\section{Continuous Attention with Gaussian RBFs}\label{sec:gaussian_basis}

We derive expressions for the evaluation and gradient computation of  continuous attention mechanisms where $\psi(t)$ are Gaussian radial basis functions, both for the softmax ($\alpha=1$) and sparsemax ($\alpha=2$) cases. 
For softmax, we show closed-form expressions for any number of dimensions (including the 1D and 2D cases). 
For sparsemax, we derive closed-form expressions for the 1D case, and we reduce the 2D case to a univariate integral on an interval, easy to compute numerically. 

This makes it possible to plug both continuous attention mechanisms in neural networks and learn them end-to-end with the gradient backpropagation algorithm.

\subsection{Continuous softmax ($\alpha=1$)}

We derive expressions for continuous softmax for multivariate Gaussians in $\mathbb{R}^D$.  
This includes the 1D and 2D cases, where $D \in \{1,2\}$. 

If $S=\mathbb{R}^D$, for $\phi(t)=[t,tt^\top]$, the distribution $p=\hat{p}_{\Omega_1}[f_\theta]$, with $f_\theta(t)=\theta^\top \phi(t)$, is a multivariate Gaussian where the mean $\mu$ and the covariance matrix $\Sigma$ are related to the canonical parameters as $\theta=[\Sigma^{-1}\mu,-\frac{1}{2}\Sigma^{-1}]$. 

We derive closed form expressions for the attention mechanism output 
$\rho_1(\theta)=\mathbb{E}_p[\psi(t)]$ in \eqref{eq:attention_expectation} and for its Jacobian $J_{\rho_1}(\theta)= \mathrm{cov}_{p,1}(\phi(t), \psi(t))$ in \eqref{eq:jacob},  when $\psi(t)$ are Gaussian RBFs, i.e., each $\psi_j$ is of the form $\psi_j(t)=\mathcal{N}(t;\mu_j, \Sigma_j)$. 

\paragraph{Forward pass.}

Each coordinate of the attention mechanism output becomes the integral of a product of Gaussians,
\begin{equation}
    \mathbb{E}_p[\psi_j(t)]=\int_{\mathbb{R}^D}\mathcal{N}(t;\mu, \Sigma)\mathcal{N}(t;\mu_j, \Sigma_j).
\end{equation}

We use the fact that the product of two Gaussians is a scaled Gaussian:
\begin{equation}
    \mathcal{N}(t;\mu, \Sigma)\mathcal{N}(t;\mu_j, \Sigma_j)=\Tilde{s}\mathcal{N}(t;\Tilde{\mu}, \Tilde{\Sigma}),
\end{equation}
where 
\begin{equation}
    \Tilde{s}=\mathcal{N}(\mu;\mu_j, \Sigma + \Sigma_j), \qquad \Tilde{\Sigma}=(\Sigma^{-1}+\Sigma_j^{-1})^{-1}, \qquad \Tilde{\mu}=\Tilde{\Sigma}(\Sigma^{-1}\mu+\Sigma_j^{-1}\mu_j).
\end{equation}
Therefore, the forward pass can be computed as:
\begin{equation}\label{eq:continuous_softmax_forward_pass}
\begin{split}
\mathbb{E}_p[\psi_j(t)] & =\Tilde{s}\int_{\mathbb{R}^D}\mathcal{N}(t;\Tilde{\mu}, \Tilde{\Sigma})=\Tilde{s}\\
 & =\mathcal{N}(\mu;\mu_j, \Sigma + \Sigma_j).
\end{split}
\end{equation}

\paragraph{Backward pass.}
To compute the backward pass, we have that each row of the Jacobian $J_{\rho_1}(\theta)$ becomes a first or second moment under the resulting Gaussian:
\begin{equation}\label{eq:continuous_softmax_backward_pass_01}
\begin{split}
\mathrm{cov}_{p,1}(t, \psi_j(t)) & = \mathbb{E}_p[t\psi_j(t)]-\mathbb{E}_p[t]\mathbb{E}_p[\psi_j(t)]\\
 & =\int_{\mathbb{R}^D}t\mathcal{N}(t;\mu, \Sigma)\mathcal{N}(t;\mu_j, \Sigma_j)-\Tilde{s}\mu\\
 & =\Tilde{s}\int_{\mathbb{R}^D}t\mathcal{N}(t;\Tilde{\mu}, \Tilde{\Sigma})-\Tilde{s}\mu \\
 & =\Tilde{s}(\Tilde{\mu}-\mu),
\end{split}
\end{equation} and, noting that $\Sigma=\mathbb{E}[(t-\mu)(t-\mu)^\top]=\mathbb{E}[tt^\top]-\mu\mu^\top$,
\begin{equation}\label{eq:continuous_softmax_backward_pass_02}
\begin{split}
\mathrm{cov}_{p,1}(tt^\top, \psi_j(t)) & = \mathbb{E}_p[tt^\top\psi_j(t)]-\mathbb{E}_p[tt^\top]\mathbb{E}_p[\psi_j(t)]\\
 & =\int_{\mathbb{R}^D}tt^\top\mathcal{N}(t;\mu, \Sigma)\mathcal{N}(t;\mu_j, \Sigma_j)-\Tilde{s}(\Sigma+\mu\mu^\top)\\
 &=\Tilde{s}\int_{\mathbb{R}^D}tt^\top\mathcal{N}(t;\Tilde{\mu}, \Tilde{\Sigma})-\Tilde{s}(\Sigma+\mu\mu^\top)\\
 &=\Tilde{s}(\Tilde{\Sigma}+\Tilde{\mu}\Tilde{\mu}^\top)-\Tilde{s}(\Sigma+\mu\mu^\top)\\
 &=\Tilde{s}(\Tilde{\Sigma}+\Tilde{\mu}\Tilde{\mu}^\top-\Sigma-\mu\mu^\top).
\end{split}
\end{equation}

\subsection{Continuous sparsemax in 1D ($\alpha=2$, $D=1$)}

With $\phi(t) = [t, t^2]$, the distribution $p = \hat{p}_{\Omega_2}[f_\theta]$, with $f_\theta(t) = \theta^\top \phi(t)$, becomes a truncated parabola where  $\mu$ and  $\sigma^2$ are related to the canonical parameters as above, i.e., $\theta = [\frac{\mu}{\sigma^2}, -\frac{1}{2\sigma^2}]$.

We  derive closed form expressions for the attention mechanism output $\rho_2(\theta) = \mathbb{E}_{p}[\psi(t)]$ in \eqref{eq:attention_expectation} and its Jacobian $J_{\rho_2}(\theta) = \frac{\partial \rho_2(\theta)}{\partial \theta} = \mathrm{cov}_{p, 2}(\phi(t), \psi(t))$ in \eqref{eq:jacob} when $\psi(t)$ and Gaussian RBFs, {\it i.e.}, each $\psi_j$ is of the form $\psi_j(t) = \mathcal{N}(t; \mu_j, \sigma_j^2)$.

\paragraph{Forward pass.}
Each coordinate of the attention mechanism output becomes:
\begin{eqnarray}\label{eq:expectation_continuous_sparsemax_rbf}
\mathbb{E}_{p}[\psi_j(t)] &=& \int_{\mu-a}^{\mu+a} \left(-\lambda - \frac{(t-\mu)^2}{2\sigma^2}\right)  \mathcal{N}(t; \mu_j, \sigma_j^2)\nonumber\\
&=& \int_{\frac{\mu-\mu_j-a}{\sigma_j}}^{\frac{\mu-\mu_j+a}{\sigma_j}} \frac{1}{\sigma_j} \left(-\lambda - \frac{(\sigma_j t + \mu_j - \mu)^2}{2\sigma^2}\right)  \mathcal{N}(s; 0, 1) ds,
\end{eqnarray}
where $a=(\frac{3}{2}\sigma^2)^{1/3}$ and $\lambda=-\frac{a^2}{2\sigma^2} = -\frac{1}{2}(\frac{3}{2\sigma})^{2/3}$, as stated in \eqref{eq:lambda_gaussian_proof}, and we made the substitution
$s = \frac{t-\mu_j}{\sigma_j}$. 
We use the fact that, for any $u, v \in \mathbb{R}$ such that $u \le v$:
\begin{eqnarray}\label{eq:erf_expr123}
\int_{u}^{v} \mathcal{N}(t; 0, 1) &=&  \frac{1}{2}\left( \mathrm{erf}\left(\frac{v}{\sqrt{2}}\right) - \mathrm{erf}\left(\frac{u}{\sqrt{2}}\right) \right),\nonumber\\
\int_{u}^{v} t\mathcal{N}(t; 0, 1) &=&  -\mathcal{N}(v; 0, 1) + \mathcal{N}(u; 0, 1),\nonumber\\
\int_{u}^{v} t^2\mathcal{N}(t; 0, 1) &=&
\frac{1}{2}\left( \mathrm{erf}\left(\frac{v}{\sqrt{2}}\right) - \mathrm{erf}\left(\frac{u}{\sqrt{2}}\right) \right) - v\mathcal{N}(v; 0, 1) + u\mathcal{N}(u; 0, 1),
\end{eqnarray}
from which the expectation \eqref{eq:expectation_continuous_sparsemax_rbf} can be computed directly.

\paragraph{Backward pass.}
Since $|\mathrm{supp}(p)| = 2a$, we have from \eqref{eq:beta_covariance} and \eqref{eq:erf_expr123} that each row of the Jacobian $J_{\rho_2}(\theta)$ becomes:
\begin{eqnarray}
\lefteqn{\mathrm{cov}_{p, 2}(t, \psi_j(t)) =} \nonumber\\
&& \int_{\mu-a}^{\mu+a} t\mathcal{N}(t; \mu_j, \sigma_j^2)
- \frac{1}{2a}\left(\int_{\mu-a}^{\mu+a} t\right)\left(\int_{\mu-a}^{\mu+a} \mathcal{N}(t; \mu_j, \sigma_j^2)\right)\nonumber\\
&=&
\int_{\frac{\mu-\mu_j-a}{\sigma_j}}^{\frac{\mu-\mu_j+a}{\sigma_j}} (\mu_j + \sigma_j s)\mathcal{N}(s; 0, 1)
- \underbrace{\frac{1}{2a}\left( \frac{(\mu+a)^2}{2} - \frac{(\mu-a)^2}{2} \right)}_{=\mu} 
\left( \int_{\frac{\mu-\mu_j-a}{\sigma_j}}^{\frac{\mu-\mu_j+a}{\sigma_j}} \mathcal{N}(s; 0, 1) \right)\nonumber\\
&=&
(\mu_j-\mu)\int_{\frac{\mu-\mu_j-a}{\sigma_j}}^{\frac{\mu-\mu_j+a}{\sigma_j}} \mathcal{N}(s; 0, 1)
+ \sigma_j \int_{\frac{\mu-\mu_j-a}{\sigma_j}}^{\frac{\mu-\mu_j+a}{\sigma_j}} s\mathcal{N}(s; 0, 1)
\nonumber\\
&=&
\frac{\mu_j-\mu}{2}\left( \mathrm{erf}\left( \frac{\mu-\mu_j+a}{\sqrt{2}\sigma_j} \right) - \mathrm{erf}\left( \frac{\mu-\mu_j-a}{\sqrt{2}\sigma_j} \right)\right)\nonumber\\
&& - \sigma_j \left( \mathcal{N}\left(\frac{\mu-\mu_j+a}{\sigma_j}; 0, 1\right) - \mathcal{N}\left(\frac{\mu-\mu_j-a}{\sigma_j}; 0, 1\right) \right),
\end{eqnarray}
and
\begin{eqnarray}
\lefteqn{\mathrm{cov}_{p, 2}(t^2, \psi_j(t)) =} \nonumber\\
&& \int_{\mu-a}^{\mu+a} t^2\mathcal{N}(t; \mu_j, \sigma_j^2)
- \frac{1}{2a}\left(\int_{\mu-a}^{\mu+a} t^2\right)\left(\int_{\mu-a}^{\mu+a} \mathcal{N}(t; \mu_j, \sigma_j^2)\right)
\nonumber\\
&=&
\int_{\frac{\mu-\mu_j-a}{\sigma_j}}^{\frac{\mu-\mu_j+a}{\sigma_j}} (\mu_j + \sigma_j s)^2\mathcal{N}(s; 0, 1)
- \underbrace{\frac{1}{2a}\left( \frac{(\mu+a)^3}{3} - \frac{(\mu-a)^3}{3} \right)}_{=\frac{a^2}{3} + \mu^2} 
\left( \int_{\frac{\mu-\mu_j-a}{\sigma_j}}^{\frac{\mu-\mu_j+a}{\sigma_j}} \mathcal{N}(s; 0, 1) \right)\nonumber\\
&=&
\left(\mu_j^2-\mu^2 -\frac{a^2}{3}\right)\int_{\frac{\mu-\mu_j-a}{\sigma_j}}^{\frac{\mu-\mu_j+a}{\sigma_j}} \mathcal{N}(s; 0, 1)
+ 2\mu_j\sigma_j \int_{\frac{\mu-\mu_j-a}{\sigma_j}}^{\frac{\mu-\mu_j+a}{\sigma_j}} s\mathcal{N}(s; 0, 1)
+ \sigma_j^2 \int_{\frac{\mu-\mu_j-a}{\sigma_j}}^{\frac{\mu-\mu_j+a}{\sigma_j}} s^2\mathcal{N}(s; 0, 1)
\nonumber\\
&=&
\left(\mu_j^2-\mu^2 +\sigma_j^2 -\frac{a^2}{3}\right)\left( \mathrm{erf}\left( \frac{\mu-\mu_j+a}{\sqrt{2}\sigma_j} \right) - \mathrm{erf}\left( \frac{\mu-\mu_j-a}{\sqrt{2}\sigma_j} \right)\right)\nonumber\\
&& -\sigma_j(\mu+\mu_j+a) \mathcal{N}\left(\frac{\mu-\mu_j+a}{\sigma_j}; 0, 1\right) + \sigma_j(\mu+\mu_j-a) \mathcal{N}\left(\frac{\mu-\mu_j-a}{\sigma_j}; 0, 1\right).
\end{eqnarray}

\subsection{Continuous sparsemax in 2D ($\alpha=2$, $D=2$)}

Let us now consider the case where $D=2$. 
For $\phi(t)=[t,tt^\top]$, the distribution $p=\hat{p}_{\Omega_2}[f_\theta]$, with $f_\theta(t)=\theta^\top \phi(t)$, becomes a bivariate truncated paraboloid where $\mu$ and  $\Sigma$ are related to the canonical parameters as before, $\theta=[\Sigma^{-1}\mu,-\frac{1}{2}\Sigma^{-1}]$. We obtain expressions for the attention mechanism output $\rho_2(\theta)=\mathbb{E}_p[\psi(t)]$ and its Jacobian $J_{\rho_2}(\theta)= \mathrm{cov}_{p,2}(\phi(t), \psi(t))$ that include 1D integrals (simple to integrate numerically), when $\psi(t)$ are Gaussian RBFs, {\it i.e.}, when each $\psi_j$ is of the form $\psi_j(t)=\mathcal{N}(t;\mu_j, \Sigma_j)$.

We start with the following lemma:

\smallskip

\begin{lemma}\label{lemma:affine_transform_gaussian}
Let $\mathcal{N}(t, \mu, \Sigma)$ be a $D$-dimensional multivariate Gaussian, 
Let $A \in \mathbb{R}^{D \times R}$ be a full column rank matrix (with $R \le D$), and $b \in \mathrm{R}^D$. 
Then we have $\mathcal{N}(Au + b; \mu, \Sigma) = \tilde{s} \mathcal{N}(u; \tilde{\mu}, \tilde{\Sigma})$ with:
\begin{eqnarray*}
\tilde{\Sigma} &=& (A^\top \Sigma^{-1} A)^{-1}\\
\tilde{\mu} &=& \tilde{\Sigma}A^\top \Sigma^{-1} (\mu - b)\\
\tilde{s} &=& (2\pi)^{\frac{R-D}{2}} \frac{|\tilde{\Sigma}|^{1/2}}{|\Sigma|^{1/2}} \exp\left(-\frac{1}{2}(\mu-b)^\top P(\mu-b)\right), \quad P = \Sigma^{-1} - \Sigma^{-1} A \tilde{\Sigma} A^\top\Sigma^{-1}.
\end{eqnarray*}
If $R=D$, then $A$ is invertible and the expressions above can be simplified to:
\begin{eqnarray*}
\tilde{\Sigma} &=& A^{-1} \Sigma A^{-\top}\\
\tilde{\mu} &=& A^{-1} (\mu - b)\\
\tilde{s} &=& |A|^{-1}.
\end{eqnarray*}
\end{lemma}

\begin{proof}
The result can be derived by writing $\mathcal{N}(Au+b; \mu, \Sigma) = (2\pi)^{-R/2} |\Sigma|^{-1/2} \exp(-\tfrac{1}{2}(Au+b-\mu)^\top\Sigma^{-1}(Au+b-\mu))$ and splitting the exponential of the sum as a product of exponentials.
\end{proof}

\paragraph{Forward pass.}
For the forward pass, we need to compute
\begin{equation}
\mathbb{E}_p[\psi_j(t)] = \iint_{\mathbb{R}^2} \left[-\lambda - \frac{1}{2}(t-\mu)^\top\Sigma^{-1}(t-\mu)\right]_+ \mathcal{N}(t; \mu_j, \Sigma_j) dt,
\end{equation}
with
\begin{equation}
    \mathcal{N}(t;\mu_j, \Sigma_j)=\frac{1}{2\pi\left| \Sigma_j\right|^{\frac{1}{2}}} \exp{\left(-\frac{1}{2}(t-\mu_j)^\top \Sigma_j^{-1}(t-\mu_j)\right)},
\end{equation}
and (from \eqref{eq:truncated_paraboloid})
\begin{equation}
    \lambda=-\left(\frac{1}{\pi\sqrt{\det(\Sigma)}}\right)^{\frac{1}{2}}.
\end{equation}
Using Lemma~\ref{lemma:affine_transform_gaussian} and  the change of variable formula (which makes the determinants cancel), we can reparametrize $u = (-2\lambda)^{-\frac{1}{2}} \Sigma^{-\frac{1}{2}} (t - \mu)$ and write this as an integral over the unit circle:
\begin{equation}
\mathbb{E}_p[\psi_j(t)] = \iint_{\|u\|\le 1} 
-\lambda (1-\|u\|^2) \mathcal{N}(u; \tilde{\mu}, \tilde{\Sigma}) du,
\end{equation}
with $\tilde{\mu} = (-2\lambda)^{-\frac{1}{2}} \Sigma^{-\frac{1}{2}}(\mu_j - \mu)$, 
$\tilde{\Sigma} = (-2\lambda)^{-1} \Sigma^{-\frac{1}{2}} \Sigma_j \Sigma^{-\frac{1}{2}}$. 
We now do a change to polar coordinates, $u = (r\cos\theta, r\sin\theta) = ar$, where $a = [\cos \theta, \sin \theta]^\top \in \mathbb{R}^{2\times 1}$. The integral becomes:
\begin{eqnarray}
\mathbb{E}_p[\psi_j(t)] &=& \int_{0}^{2\pi} \int_{0}^1 
-\lambda (1-r^2) \mathcal{N}(ar; \tilde{\mu}, \tilde{\Sigma}) r \, dr \, d\theta\nonumber\\
&=& \int_{0}^{2\pi} \int_{0}^1 
-\lambda r(1-r^2) \tilde{s} \mathcal{N}(r; r_0, \sigma^2) \, dr \, d\theta,
\end{eqnarray}
where in the second line we applied again Lemma~\ref{lemma:affine_transform_gaussian}, resulting in 
\begin{eqnarray*}
\sigma^2(\theta) \equiv \sigma^2 &=& (a^\top \tilde{\Sigma}^{-1} a)^{-1}\\
r_0(\theta) \equiv r_0 &=& \sigma^2 a^\top \tilde{\Sigma}^{-1} \tilde{\mu}\\
\tilde{s}(\theta) \equiv \tilde{s} &=& \frac{1}{\sqrt{2\pi}}  \frac{\sigma}{|\tilde{\Sigma}|^{1/2}} \exp\left(-\frac{1}{2}\tilde{\mu}^\top P\tilde{\mu}\right), \quad P = \tilde{\Sigma}^{-1} - \sigma^2 \tilde{\Sigma}^{-1} a a^\top\tilde{\Sigma}^{-1}.
\end{eqnarray*}
Applying Fubini's theorem, 
we fix $\theta$ and integrate with respect to $r$. We use the formulas \eqref{eq:erf_expr123} and the fact that, for any $u, v \in \mathbb{R}$ such that $u \le v$:
\begin{equation}\label{eq:erf_expr4}
\int_{u}^v t^3 \mathcal{N}(t; 0, 1) = -\mathcal{N}(v; 0, 1)(2+v^2) + \mathcal{N}(u; 0, 1)(2+u^2).
\end{equation}
We obtain a closed from expression for the inner integral:
\begin{eqnarray}
 F(\theta) &=& \int_{0}^1 
r(1-r^2)  \mathcal{N}(r; r_0, \sigma^2) \, dr\nonumber\\
&=& (2\sigma^3 + r_0^2 \sigma + r_0\sigma) \mathcal{N}\left(\frac{1-r_0}{\sigma}; 0, 1\right)
- (2\sigma^3 + r_0^2 \sigma -\sigma) \mathcal{N}\left(-\frac{r_0}{\sigma}; 0, 1\right)\nonumber\\
&& 
-\frac{r_0^3 + (3\sigma^2 - 1)r_0}{2} \left[ \mathrm{erf}\left(\frac{1-r_0}{\sqrt{2}\sigma}\right)- \mathrm{erf}\left(-\frac{r_0}{\sqrt{2}\sigma}\right)\right].
\end{eqnarray}
The desired integral can then be expressed in a single dimension as 
\begin{eqnarray}
\mathbb{E}_p[\psi_j(t)] &=& -\lambda \int_{0}^{2\pi} \tilde{s}(\theta) F(\theta),
\end{eqnarray}
which may be integrated numerically.

\paragraph{Backward pass.}

For the backward pass we need to solve
\begin{equation}
\label{eq:1-2_row_J}
\mathrm{cov}_{p, 2}(t, \psi_j(t)) =
\iint_{E} t\mathcal{N}(t; \mu_j, \Sigma_j)
- \frac{1}{|E|}\left(\iint_{E} t\right)\left(\iint_{E} \mathcal{N}(t; \mu_j, \Sigma_j)\right)
\end{equation}and 
\begin{equation}
\label{eq:3-6_row_J}
\mathrm{cov}_{p, 2}(tt^\top, \psi_j(t)) =
\iint_{E} tt^\top\mathcal{N}(t; \mu_j, \Sigma_j)
- \frac{1}{|E|}\left(\iint_{E} tt^\top\right)\left(\iint_{E} \mathcal{N}(t; \mu_j, \Sigma_j)\right)
\end{equation} where $E = \mathrm{supp}(p) = \{t \in \mathbb{R}^2 \mid \frac{1}{2}(t-\mu)^\top \Sigma^{-1} (t-\mu) \le -\lambda\}$ denotes the support of the density $p$, a region bounded by an ellipse. 
Note that these expressions include integrals of vector-valued functions and that \eqref{eq:1-2_row_J} and \eqref{eq:3-6_row_J} correspond to the first to second and the third to sixth row of the Jacobian, respectively. The integrals that do not include Gaussians have closed form expressions and can be computed as
\begin{equation}
    \frac{1}{|E|}\left(\iint_{E} t\right)=\mu
\end{equation}
and
\begin{equation}
    \frac{1}{|E|}\left(\iint_{E} tt^\top\right)=\mu\mu^\top+\frac{\Sigma}{|E|},
\end{equation} where $|E|$ is the area of the region $E$ given by
\begin{equation}
    |E|=\frac{\pi}{\sqrt{ \det \left( \frac{1}{-2\lambda}\,\Sigma^{-1}\right)}}.
\end{equation}

All the other integrals are solved using the same affine transformation and change to polar coordinates as in the forward pass. Given this, $\tilde{\mu}$, $\tilde{\Sigma}$, $a$, $\sigma^2, r_0$ and $\tilde{s}$ are the same as before. 
To solve \eqref{eq:1-2_row_J} we write
\begin{equation}
    \iint_{E} t\mathcal{N}(t; \mu_j, \Sigma_j) = \iint_{\|u\|\le 1} \left((-2\lambda)^{\frac{1}{2}} \Sigma^{\frac{1}{2}}u+\mu\right) \mathcal{N}(u; \tilde{\mu}, \tilde{\Sigma}) du
\end{equation}
in polar coordinates,
\begin{equation}
    \int_{0}^{2\pi} \int_{0}^1 
 r\left((-2\lambda)^{\frac{1}{2}} \Sigma^{\frac{1}{2}}ar+\mu\right) \tilde{s} \, \mathcal{N}(r; r_0, \sigma^2) dr \, d\theta,
\end{equation}
which can be then expressed in a single dimension as 
\begin{eqnarray}
\iint_{E} t\mathcal{N}(t; \mu_j, \Sigma_j) &=& \int_{0}^{2\pi} \tilde{s}(\theta) G(\theta)d\theta,
\end{eqnarray}
with
\begin{eqnarray}
 G(\theta) &=& \int_{0}^1 
r\left((-2\lambda)^{\frac{1}{2}} \Sigma^{\frac{1}{2}}ar+\mu\right)  \mathcal{N}(r; r_0, \sigma^2) \, dr \nonumber\\
&=& \int_{-\frac{r_0}{\sigma}}^{\frac{1-r_0}{\sigma}}
(s\sigma+r_0)\left((-2\lambda)^{\frac{1}{2}} \Sigma^{\frac{1}{2}}a(s\sigma+r_0)+\mu\right)  \mathcal{N}(r; r_0, \sigma^2) \, ds \nonumber\\
&=&
\left((-2\lambda)^{\frac{1}{2}} \Sigma^{\frac{1}{2}}a\sigma(r_0)+\mu\sigma\right)\mathcal{N}\left(-\frac{r_0}{\sigma};0,1\right) \nonumber\\
&&-\left((-2\lambda)^{\frac{1}{2}} \Sigma^{\frac{1}{2}}a\sigma(1+r_0)+\mu\sigma\right) \mathcal{N}\left(\frac{1-r_0}{\sigma};0,1\right) \nonumber\\
&&+\frac{1}{2}\left((-2\lambda)^{\frac{1}{2}} \Sigma^{\frac{1}{2}}a(\sigma^2+r_0^2)+\mu r_0\right)\left[ \mathrm{erf}\left(\frac{1-r_0}{\sqrt{2}\sigma}\right)- \mathrm{erf}\left(-\frac{r_0}{\sqrt{2}\sigma}\right)\right]. 
\end{eqnarray}
We do the same for
\begin{equation}
    \iint_{E} \mathcal{N}(t; \mu_j, \Sigma_j)=\iint_{\|u\|\le 1} \mathcal{N}(u; \tilde{\mu}, \tilde{\Sigma}) du =     \int_{0}^{2\pi} \int_{0}^1 
 r\tilde{s} \, \mathcal{N}(r; r_0, \sigma^2) dr \, d\theta,
\end{equation}
which can then be expressed in a single dimension as 
\begin{eqnarray}
\iint_{E} \mathcal{N}(t; \mu_j, \Sigma_j) &=& \int_{0}^{2\pi} \tilde{s}(\theta) H(\theta)d\theta,
\end{eqnarray}with
\begin{eqnarray}
 H(\theta) &=& \int_{0}^1 
r\mathcal{N}(r; r_0, \sigma^2) \, dr \nonumber = \int_{-\frac{r_0}{\sigma}}^{\frac{1-r_0}{\sigma}}
(s\sigma+r_0)  \mathcal{N}(r; r_0, \sigma^2) \, ds \\
&=&
\sigma\left[\mathcal{N}\left(-\frac{r_0}{\sigma};0,1\right)
-\mathcal{N}\left(\frac{1-r_0}{\sigma};0,1\right)\right] \nonumber +\frac{r_0}{2}\left[ \mathrm{erf}\left(\frac{1-r_0}{\sqrt{2}\sigma}\right)- \mathrm{erf}\left(-\frac{r_0}{\sqrt{2}\sigma}\right)\right].
\end{eqnarray}
Finally, to solve \eqref{eq:3-6_row_J} we simplify the integral
\begin{eqnarray}
    \iint_{E} tt^\top\mathcal{N}(t; \mu_j, \Sigma_j)&=&\iint_{\|u\|\le 1} \left((-2\lambda)^{\frac{1}{2}} \Sigma^{\frac{1}{2}}u+\mu\right)\left((-2\lambda)^{\frac{1}{2}} \Sigma^{\frac{1}{2}}u+\mu\right)^\top \mathcal{N}(u; \tilde{\mu}, \tilde{\Sigma}) du \nonumber \\
    &=&\int_{0}^{2\pi} \int_{0}^1 
 r(r^2A+rB+C)\tilde{s} \, \mathcal{N}(r; r_0, \sigma^2) dr \, d\theta
\end{eqnarray}
with
\begin{equation}
    A=(-2\lambda)\Sigma^{\frac{1}{2}}aa^\top(\Sigma^{\frac{1}{2}})^\top
\end{equation}
\begin{equation}
    B=(-2\lambda)^{\frac{1}{2}}\left( \Sigma^{\frac{1}{2}}a\mu^\top+\mu a^\top (\Sigma^{\frac{1}{2}})^\top \right)
\end{equation}
\begin{equation}
    C=\mu\mu^\top.
\end{equation}
The integral can then be expressed in a single dimension as 
\begin{eqnarray}
\iint_{E} tt^\top \mathcal{N}(t; \mu_j, \Sigma_j) &=& \int_{0}^{2\pi} \tilde{s}(\theta) M(\theta)d\theta,
\end{eqnarray}
with
\begin{eqnarray}
M(\theta) &=& \int_{0}^1 
(r^3A+r^2B+rC)\, \mathcal{N}(r; r_0, \sigma^2) dr \nonumber\\
&=& \int_{-\frac{r_0}{\sigma}}^{\frac{1-r_0}{\sigma}}
(s^3\tilde{A}+s^2\tilde{B}+s\,\tilde{C}+\tilde{D})  \mathcal{N}(s; 0, 1) \, ds \nonumber\\
&=& \left[\left(2+\left(-\frac{r_0}{\sigma}\right)^2\right)\tilde{A}-\frac{r_0}{\sigma}\tilde{B}+\tilde{C}\right]\mathcal{N}\left(-\frac{r_0}{\sigma};0,1\right)\nonumber\\
&&-\left[\left(2+\left(\frac{1-r_0}{\sigma}\right)^2\right)\tilde{A}+\frac{1-r_0}{\sigma}\tilde{B}+\tilde{C}\right]\mathcal{N}\left(\frac{1-r_0}{\sigma};0,1\right) \nonumber\\
&&+\frac{1}{2}\left(\tilde{B}+\tilde{D}\right)\left[ \mathrm{erf}\left(\frac{1-r_0}{\sqrt{2}\sigma}\right)- \mathrm{erf}\left(-\frac{r_0}{\sqrt{2}\sigma}\right)\right]
\end{eqnarray}
where
\begin{equation}
    \tilde{A}=\sigma^3A
\end{equation}
\begin{equation}
    \tilde{B}=\sigma^2(3r_0\,A+B)
\end{equation}
\begin{equation}
    \tilde{C}=\sigma(3r_0^2\,A+2r_0\,B+C)
\end{equation}
\begin{equation}
    \tilde{D}=r_0^3\,A+r_0^2\,B+r_0\,C.
\end{equation}

\remove{
\section{Proof of Proposition~\ref{prop:fy_properties} }\label{sec:proof_fy_properties}

The proof adapts that of \citet{Blondel2019AISTATS} 
\andre{maybe replace by \citep{blondel2020learning}} 
when  Fenchel duality is now taken in the infinite-dimensional set $\mathcal{F} \subseteq \mathbb{R}^S$, which endowed with the inner product $\langle f, g\rangle = \int_S f(t) g(t) d\nu(t)$ forms a Hilbert space \citep{Bauschke_Combettes2011}. \andre{point to section?} 
The non-negativity of $L_\Omega$ stems from the Fenchel-Young inequality in Hilbert spaces.
The loss is zero iff $(f_\theta, p)$ is a dual pair, i.e., if $p = \hat{p}_\Omega[f_\theta] = \nabla \Omega^*(f_\theta)$.
The gradient of $L_\Omega$ is
\begin{equation*}
\nabla_\theta L_\Omega(f_\theta; p) = \int_S \frac{\partial L_\Omega(f_\theta; p)}{\partial f_\theta(t)}  \nabla_\theta f_\theta(t) d\mu(t) = \int_S (\hat{p}_\Omega[f_\theta](t) - p(t)) \phi(t).
\end{equation*}
where we used the fact that
$\frac{\partial L_\Omega(f_\theta; p)}{\partial f_\theta(t)} = [\nabla \Omega^*(f_\theta) - p](t)$.
This proves point 1.
To prove point 2, note that we have
\begin{equation*}
\nabla\nabla_\theta L_\Omega(f_\theta; p) = \nabla_\theta \mathbb{E}_{\hat{p}_\Omega[f_\theta]}[\phi(t)] - \underbrace{\nabla_\theta \mathbb{E}_{p}[\phi(t)]}_{=0} = \mathrm{Cov}_{\hat{p}_\Omega[f_\theta], 2-\alpha}[\phi(t)],
\end{equation*}
where we used the result of Prop.~\ref{prop:jacobian_entmax}.
Since for any probability distribution the covariance operation  leads to a positive semi-definite matrix, so does  $\mathrm{Cov}_{\hat{p}_\Omega[f_\theta], 2-\alpha}[\phi(t)]$. Therefore, the Hessian of $L_\Omega$ with respect to $\theta$ is positive semi-definite, i.e., $L_\Omega$ is convex on the canonical parameters, which proves point 2.
Finally, point 3 is an immediate consequence of points 1 and 2: Since
$L_{\Omega_\alpha}$ is convex, any stationary point is a global minimum,
and, from point 1, any stationary point $\hat\theta$ must satisfy
$\mathbb{E}_{\hat{p}_{\Omega}[f_{\hat{\theta}}]}[\phi(t)] =
\mathbb{E}_{p}[\phi(t)]$.

One interpretation of point 3 is that we may fit an $\alpha$-sparse
density to an empirical distribution $p$ by matching the expected statistics.
Point 3 guarantees that the result is optimal in the FY loss sense, analogous to
how, for exponential families, moment matching amounts to maximum likelihood.
Figure~\ref{fig:misfit} illustrates the result of fitting a Gaussian, truncated
parabola, and triangular distribution to samples drawn from each of the three
types of distribution. Kolmogorov-Smirnov test results, reported for each plot,
successfully reject all mismatched fits, confirming adequate fitting.

\begin{figure*}[t]\centering\includegraphics[width=.9\textwidth]{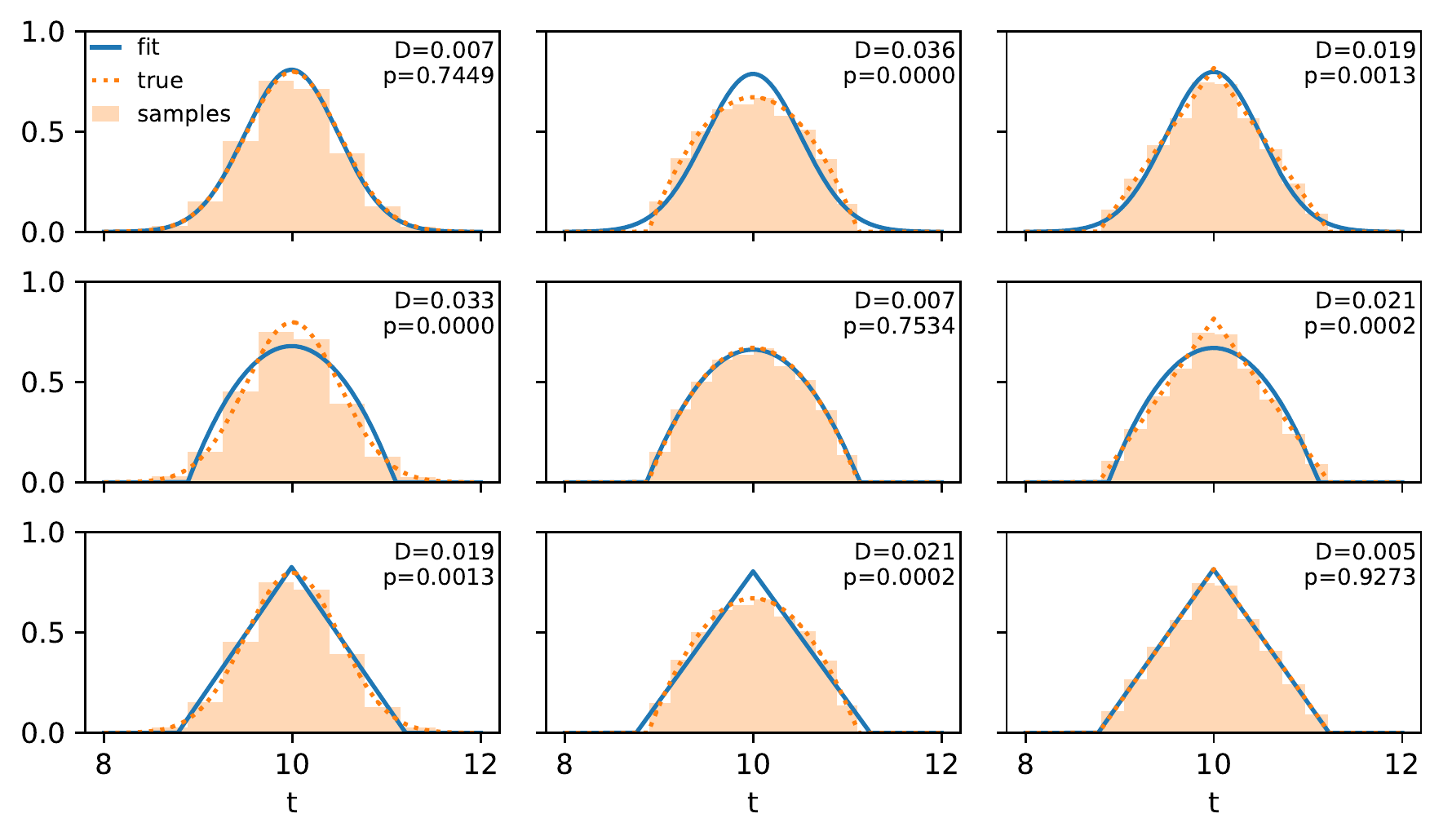}
\caption{\label{fig:misfit}Fitting a Gaussian, truncated parabola, and triangular distribution
(rows, $\downarrow$) to 10,000 samples from a Gaussian, truncated parabola, and
triangular (columns, $\rightarrow$). For each fit, we report the
Kolmogorov-Smirnov distance $D$ and corresponding $p$-value $p$.}
\end{figure*}

\section{Derivations of Fenchel-Young Losses}\label{sec:fy_example_derivations}

We derive in this section closed-form expressions for Fenchel-Young losses between truncated parabolas, triangular distributions, and between Gaussians and truncated parabolas and uniform distributions.

\subsection{FY loss between truncated parabolas}

If $f(t) = -\frac{(t-\mu)^2}{2\sigma^2}$,
we have $\lambda=-\frac{1}{2}\left(\frac{3}{2\sigma} \right)^{2/3}$, so this yields:
\begin{eqnarray}
\Omega_2^*(f) &=& \frac{1}{2} -\frac{3}{10} \left(\frac{3}{2\sigma} \right)^{2/3}.
\end{eqnarray}
Let
$p(t) = [-\lambda_p - \frac{(t-\mu_p)^2}{2\sigma_p^2}]_+$ be another truncated parabola,
with $\lambda_p = -\frac{1}{2}\left( \frac{3}{2\sigma_p}\right)^{2/3}$.
Doing the change of variable $\delta \mu = \mu - \mu_p$,
the Fenchel-Young loss becomes:
\begin{eqnarray}
L_{\Omega_2}(f; p) &=&
\Omega_2^*(f) + \Omega_2(p) - \mathbb{E}_p[f(X)]\nonumber\\
&=&
\frac{1}{2} -\frac{3}{10} \left(\frac{3}{2\sigma} \right)^{2/3}
- \frac{1}{2} +\frac{1}{5} \left(\frac{3}{2\sigma_p} \right)^{2/3}
+ \int_{-a}^a \left( -\lambda_p - \frac{t^2}{2\sigma_p^2} \right) \frac{(t-\delta\mu)^2}{2\sigma^2}\nonumber\\
&=&
-\frac{3}{10} \left(\frac{3}{2\sigma} \right)^{2/3}
+\frac{1}{5} \left(\frac{3}{2\sigma_p} \right)^{2/3}
+ \frac{1}{10\sigma^2} \left(\frac{3\sigma_p^2}{2} \right)^{2/3} + \frac{(\mu-\mu_p)^2}{2\sigma^2}.
\end{eqnarray}
which becomes 0 (and is minimized) when $\mu=\mu_p$ and $\sigma=\sigma_p$.
Note that when $\sigma=\sigma_p$ (\textit{i.e.}, if the two distributions have the same variance), the sparsemax loss between truncated parabolas and the KL divergence between Gaussians are both equal to $\frac{\mu^2}{2\sigma^2}$.

\subsection{FY loss between triangular distributions}

If $f(t) = -\frac{|t-\mu|}{b}$,
we have $\lambda=-b^{-1/2}$
and $\Omega_2(\hat{p}_{\Omega_2}[f]) = -\frac{1}{2} +  \frac{1}{3\sqrt{b}}$, so \eqref{eq:omega_conjugate} yields:
\begin{eqnarray}
\Omega_2^*(f) &=& \frac{1}{2} -\frac{2}{3\sqrt{b}}.
\end{eqnarray}
Let
$p(t) = [-\lambda_p - \frac{|t-\mu_p|}{b_p}]_+$ be another triangular distribution,
with $\lambda_p = -b_p^{-1/2}$.
We can assume without loss of generality that $\mu_p=0$.
To obtain $\mathbb{E}_p[f(t)]$, we need to consider four cases:
\begin{itemize}
\item \framebox{$\mu \ge a_p$}
In this case, $\mathbb{E}_p[f(t)] = -\int_{-a_p}^{a_p} \left(\lambda_p  + \frac{|t|}{b_p} \right) \frac{t-\mu}{b} = -\frac{\mu}{b}$ (after some math).
\item \framebox{$\mu \le -a_p$}
By symmetry, $\mathbb{E}_p[f(t)] = \frac{\mu}{b}$.
\item \framebox{$-a_p \le \mu \le 0$}
We need to split the integral in three intervals $[-a_p, \mu]$, $[\mu, 0]$, and $[0, a_p]$, giving (after some math):
\begin{equation}
\mathbb{E}_p[f(t)] = -\frac{\mu^2}{b_p^{1/2} b} - \frac{\mu^3}{3 b_p b} - \frac{b_p^{1/2}}{3b}.
\end{equation}
\item \framebox{$0 \le \mu \le a_p$}
By symmetry:
\begin{equation}
\mathbb{E}_p[f(t)] = -\frac{\mu^2}{b_p^{1/2} b} + \frac{\mu^3}{3 b_p b} - \frac{b_p^{1/2}}{3b}.
\end{equation}
\end{itemize}
Putting everything together, we obtain:
\begin{equation}
\mathbb{E}_p[f(t)] = \min \left\{-\frac{|\mu|}{b}, \,\, \frac{|\mu|^3}{3 b_p b} -\frac{\mu^2}{b_p^{1/2} b} - \frac{b_p^{1/2}}{3b}\right\},
\end{equation}
and
the Fenchel-Young loss becomes:
\begin{eqnarray}
L_{\Omega_2}(f; p) &=&
\Omega_2^*(f) + \Omega_2(p) - \mathbb{E}_p[f(t)]\nonumber\\
&=&
\frac{1}{3b_p^{1/2}} - \frac{2}{3b^{1/2}}
+ \max \left\{\frac{|\mu|}{b}, \,\, -\frac{|\mu|^3}{3 b_p b} +\frac{\mu^2}{b_p^{1/2} b} + \frac{b_p^{1/2}}{3b}\right\}.
\end{eqnarray}

\subsection{KL  between as uniform and a Gaussian}

Let now $p(t) = u_{[a,b]}(t)$ be a uniform density, and let $f_\theta(t) = \theta^\top \phi(t)$ with $\phi(t) = [t, t^2]$.
The log-partition function of $f_\theta$, expressed in terms of $\mu$ and $\sigma^2$, is
\begin{equation}
\Omega^*(f_\theta) = \frac{\mu^2}{2\sigma^2} + \frac{1}{2}\log(2\pi\sigma^2).
\end{equation}
The Shannon negentropy of a uniform distribution is
\begin{equation}
\Omega(u_{[a,b]}) = -\log(b-a).
\end{equation}
The expectation $\mathbb{E}_{u_{[a,b]}}[f_\theta(t)]$ is
\begin{eqnarray}\label{eq:expectation_under_uniform}
\mathbb{E}_{u_{[a,b]}}[f_\theta(t)] &=& \frac{1}{b-a}\int_a^b \theta^\top \phi(t)\nonumber\\
&=& \theta_1 \frac{b^2 - a^2}{2(b-a)} + \theta_2 \frac{b^3 - a^3}{3(b-a)}\nonumber\\
&=& \theta_1 \frac{a+b}{2} + \theta_2 \frac{a^2 + ab + b^2}{3}\nonumber\\
&=& \frac{\mu}{2\sigma^2} (a+b) - \frac{1}{6\sigma^2} (a^2 + ab + b^2).\nonumber\\
\end{eqnarray}
Putting all together, we get:
\begin{eqnarray}
    L_\Omega(f_\theta; u_{[a,b]}) &=&
    \frac{\mu^2}{2\sigma^2} + \frac{1}{2}\log(2\pi\sigma^2) -\log(b-a)\nonumber\\ && -\frac{\mu}{2\sigma^2} (a+b) + \frac{1}{6\sigma^2} (a^2 + ab + b^2).\nonumber\\
\end{eqnarray}

\subsection{Sparsemax loss between uniform and truncated parabola}

Let now $p(t) = u_{[a,b]}(t)$ be a uniform density, and let $f_\theta(t) = \theta^\top \phi(t)$ with $\phi(t) = [t, t^2]$.
We have
\begin{equation}
\Omega_2^*(f_\theta) = \frac{\mu^2}{2\sigma^2} +
\frac{1}{2} -\frac{3}{10} \left(\frac{3}{2\sigma} \right)^{2/3}.
\end{equation}
The Gini negentropy of a uniform distribution is
\begin{equation}
\Omega_2(u_{[a,b]}) = \frac{1}{2} \int_a^b \frac{1}{(b-a)^2} - 1 = \frac{1}{2}\left(\frac{1}{b-a} - 1\right).
\end{equation}
The expectation $\mathbb{E}_{u_{[a,b]}}[f_\theta(t)]$ is the same as \eqref{eq:expectation_under_uniform}.
Putting all together, we get:
\begin{eqnarray}
    L_{\Omega_2}(f_\theta; u_{[a,b]}) &=&
    \frac{\mu^2}{2\sigma^2}
    -\frac{3}{10} \left(\frac{3}{2\sigma} \right)^{2/3} +
    \frac{1}{2(b-a)}\nonumber\\ && -\frac{\mu}{2\sigma^2} (a+b) + \frac{1}{6\sigma^2} (a^2 + ab + b^2).\nonumber\\
\end{eqnarray}
}

\section{Experimental Details and Model Hyperparameters}\label{sec:model_hyperparams}

\subsection{Document classification}

We used the IMDB movie review dataset \citep{maas2011learning},%
\footnote{\url{https://ai.stanford.edu/~amaas/data/sentiment}} %
which consist of user-written text reviews with binary labels (positive/negative). 
Following \citep{jain2019attention}, we used 25K training documents, 10\% of which for validation, and 25K for testing. The training and test sets are perfectly balanced: 12.5K negative and 12.5K positive examples. 
The documents have 280 words on average.  

Our architecture is the same as \citep{maas2011learning}, a BiLSTM with attention. 
We used pretrained GloVe embeddings from the 840B release,\footnote{\url{http://nlp.stanford.edu/data/glove.840B.300d.zip}} kept frozen. 
We tuned three hyperparameters using the discrete softmax attention baseline:
learning rate within $\{0.003, \mathbf{0.001}, 0.0001\}$; $\ell_2$ within $\{0.01, 0.001, \mathbf{0.0001}, 0\}$; number of epochs within $\{5, \mathbf{10}, 20\}$. 
We picked the best configuration by doing a grid search and by taking into consideration the accuracy on the validation set (selected values in bold).
Table~\ref{tab:table_all_hyperparams} shows the hyperparameters and model configurations used for all document classification experiments.

\begin{table}[t]
    \caption{Hyperparmeters for document classification.}
    \label{tab:table_all_hyperparams}
    \begin{small}
    \begin{center}
    \begin{tabular}{llllll}
        \toprule
        \sc Hyperparameter & \sc Value  \\
        \midrule
        Batch size                  & 16    \\
        Word embeddings size        & 300     \\
        BiLSTM hidden size          & 128     \\
        Merge BiLSTM states         & Concat     \\
        Attention scorer            & \citep{bahdanau2014neural}     \\
        Conv filters                & 128   \\
        Conv kernel size            & 3     \\
        Early stopping patience     & 5     \\
        Number of epochs            & 10     \\
        Optimizer                   & Adam      \\
        $\ell_2$ regularization     & 0.0001     \\
        Learning rate               & 0.001     \\
        \bottomrule
    \end{tabular}
    \end{center}
    \end{small}
    \vskip -0.1in
\end{table}

\subsection{Machine translation}

We used the De$\to$En dataset from the IWSLT 2017 evaluation campaign \citep{cettolo2017overview}, with the standard splits (206K, 9K, and 2K sentence pairs for train/dev/test).%
\footnote{\url{https://wit3.fbk.eu/mt.php?release=2017-01-trnted}} %
We used BPE \citep{sennrich2016neural} with 32K merges to reduce the vocabulary size.
Our implementation is based on Joey-NMT \citep{kreutzer2019joey} and we used the provided configuration script for the baseline, a BiLSTM model with discrete softmax attention%
\footnote{\url{https://github.com/joeynmt/joeynmt/blob/master/configs/iwslt14_deen_bpe.yaml}} %
with the  hyperpameters in Table~\ref{tab:table_all_hyperparams_nmt}. 

\begin{table}[t]
    \caption{Hyperparmeters for neural machine translation.}
    \label{tab:table_all_hyperparams_nmt}
    \begin{small}
    \begin{center}
    \begin{tabular}{llllll}
        \toprule
        \sc Hyperparameter & \sc Value  \\
        \midrule
        Batch size                  & 80    \\
        Word embeddings size        & 620     \\
        BiLSTM hidden size          & 1000     \\
        Attention scorer            & \citep{bahdanau2014neural}     \\
        Early stopping patience & 8 \\
        Number of epochs            & 100     \\
        Optimizer                   & Adam      \\
        $\ell_2$ regularization     & 0     \\
        Dropout & 0.0 \\
        Hidden dropout & 0.2 \\
        Learning rate               & 0.0002     \\
        Scheduling & Plateau \\
        Decrease factor & 0.7 \\
        Lower case & True \\
        Normalization & Tokens \\
        Maximum output length & 80 \\
        Beam size & 5 \\
        RNN type & GRU \\
        RNN layers & 1 \\
        Input feeding & True \\
        Init. hidden & Bridge \\
        \bottomrule
    \end{tabular}
    \end{center}
    \end{small}
    \vskip -0.1in
\end{table}

\subsection{Visual question answering}

We used the VQA-v2 dataset \cite{Goyal2019} with the standard splits (443K, 214K, and 453K question-image pairs for train/dev/test, the latter subdivided into  test-dev, test-standard, test-challenge and test-reserve). We adapted the implementation of \cite{Yu2019},%
\footnote{\url{https://github.com/MILVLG/mcan-vqa}} %
consisting of a Modular Co-Attention Network (MCAN). Our architecture is the same as \cite{Yu2019} except that we represent the image input with grid features generated by a ResNet \cite{He2016} pretrained on ImageNet \cite{Russakovsky2015}, instead of bounding-box features \cite{Anderson2018}. 
The images are resized to $448 \times 448$ before going through the ResNet that outputs a feature map of size $14 \times 14 \times 2048$. To represent the input question words we use 300-dimensional GloVe word embeddings \cite{pennington2014glove}, yielding a question feature matrix representation. Table~\ref{tab:table_hyperparams_VQA} shows the hyperparameters used for all the VQA experiments presented.

All the models we experimented with use the same features and were trained only on the train set without data augmentation.

\paragraph{Examples.} 
Figure~\ref{fig:examples_vqa_skate} illustrates the difficulties that continuous attention models may face when trying to focus on objects that are too far from each other or that seem to have different relative importance to answer the question. Intuitively, in VQA, this becomes a problem when counting objects in those conditions. On the other side, in counting questions that require the understanding of a contiguous region of the image only, continuous attention may perform better (see Figure~\ref{fig:examples_vqa_2birds}).

Figures~\ref{fig:examples_vqa_hat} and \ref{fig:examples_vqa_soccer} show other examples where continuous attention focus on the right region of the image and answers the question correctly. For these cases, discrete attention is more diffuse than its continuous counterpart: in both examples, it attends to two different regions in the image, leading to incorrect answers.

\begin{table}[t]
    \caption{Hyperparmeters for VQA.}
    \label{tab:table_hyperparams_VQA}
    \begin{small}
    \begin{center}
    \begin{tabular}{llllll}
        \toprule
        \sc Hyperparameter & \sc Value  \\
        \midrule
        Batch size                  & 64    \\
        Word embeddings size        & 300     \\
        Input image features size   & 2048 \\
        Input question features size & 512 \\
        Fused multimodal features size & 1024 \\
        Multi-head attention hidden size        & 512 \\
        Number of MCA layers        & 6 \\
        Number of attention heads   & 8 \\
        Dropout rate                & 0.1 \\
        MLP size in flatten layers  & 512 \\
        Optimizer                   & Adam \\
        Base learning rate at epoch $t$ starting from 1 & $\mathrm{min}(2.5 t \cdot 10^{-5}, 1\cdot 10^{-4})$\\
        Learning rate decay ratio at epoch $t\in \{10,12\}$  & 0.2 \\
        Number of epochs            & 13 \\
        
        \bottomrule
    \end{tabular}
    \end{center}
    \end{small}
\end{table}

\begin{figure*}[t]
\centering
\includegraphics[width=0.24\textwidth]{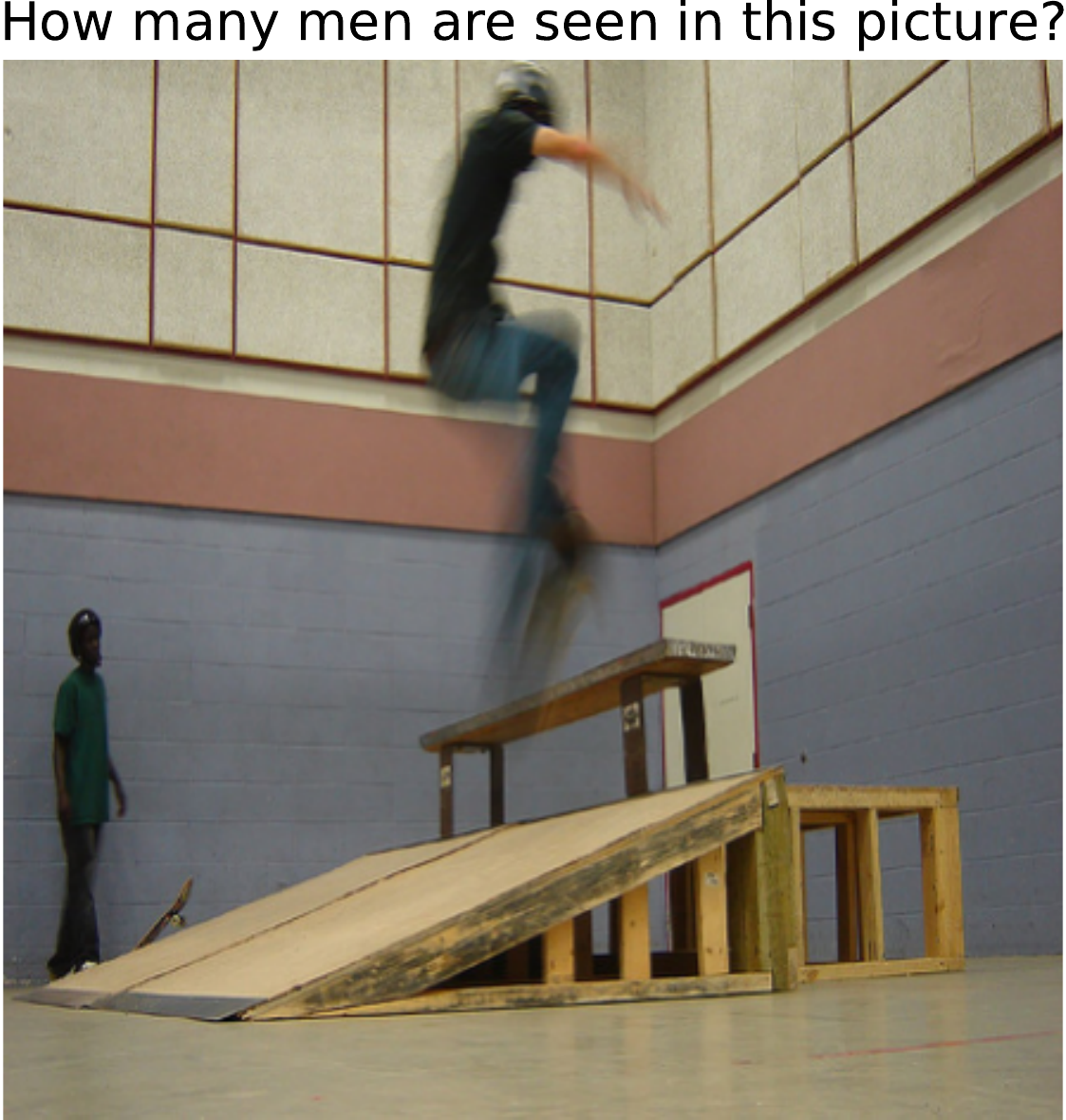}
\includegraphics[width=0.24\textwidth]{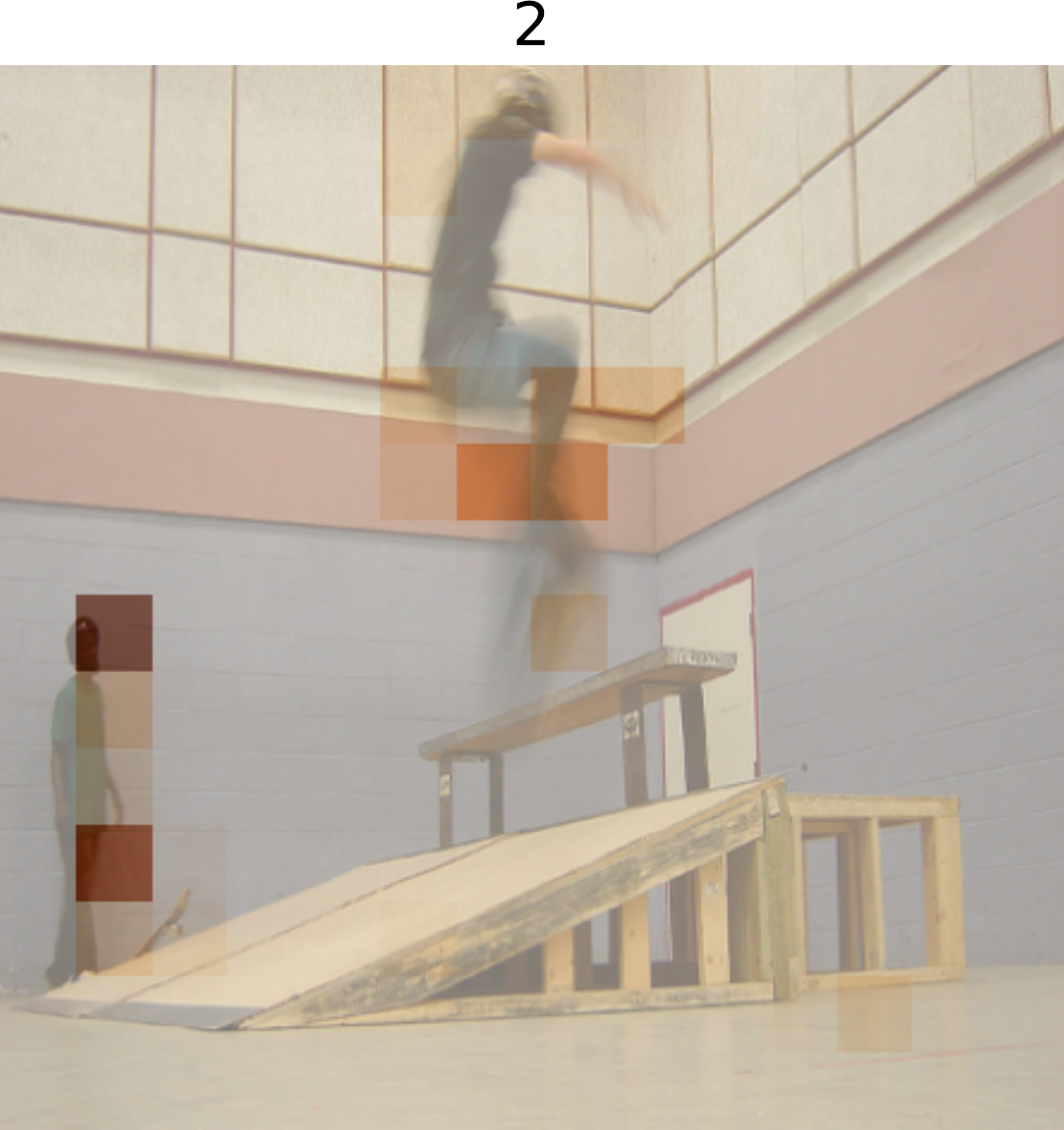}
\includegraphics[width=0.24\textwidth]{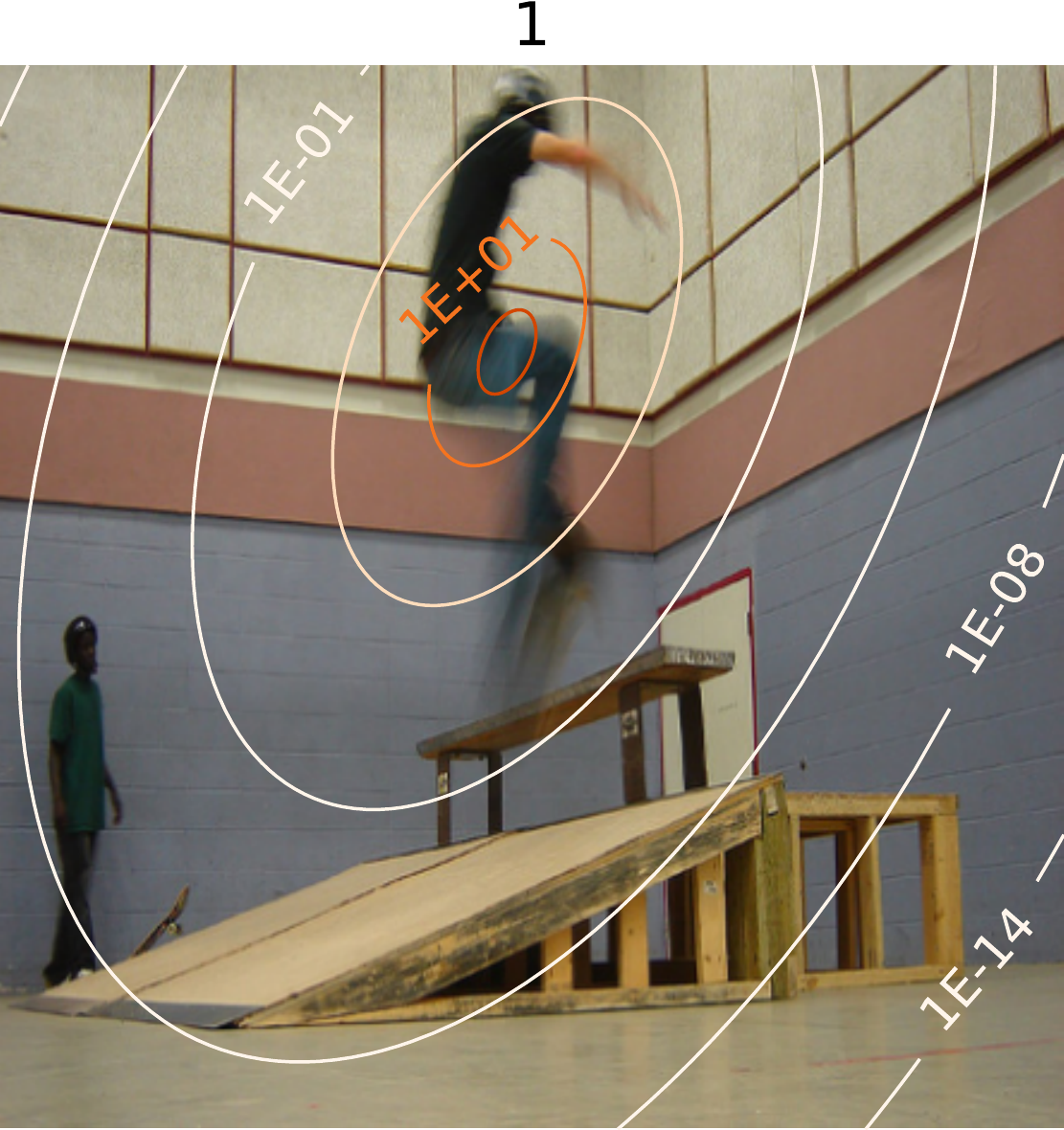}
\includegraphics[width=0.24\textwidth]{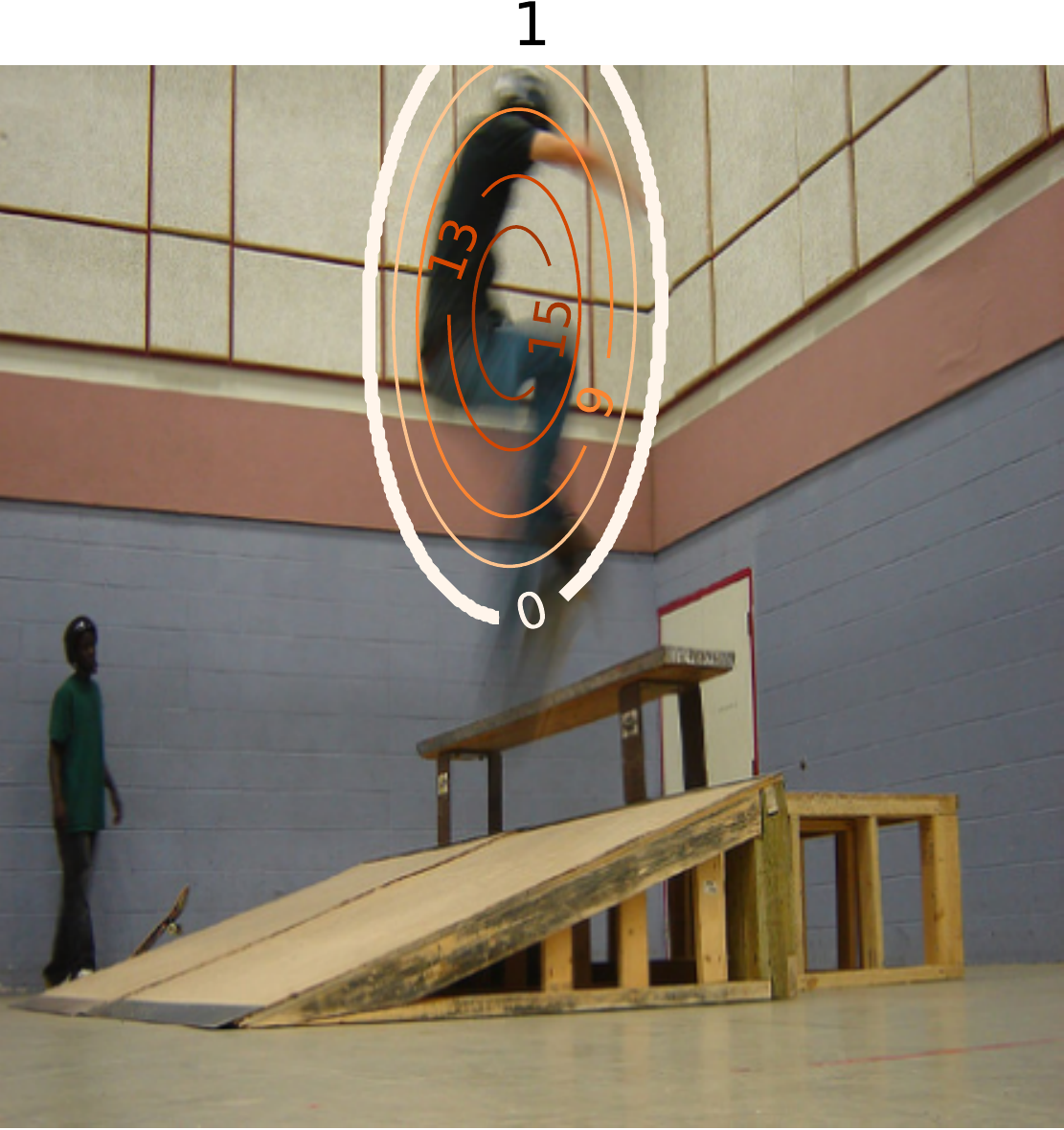}
\caption{\label{fig:examples_vqa_skate}Attention maps for an example in VQA-v2: original image, discrete attention, continuous softmax, and continuous sparsemax.}
\end{figure*}

\begin{figure*}[t]
\centering
\includegraphics[width=0.24\textwidth]{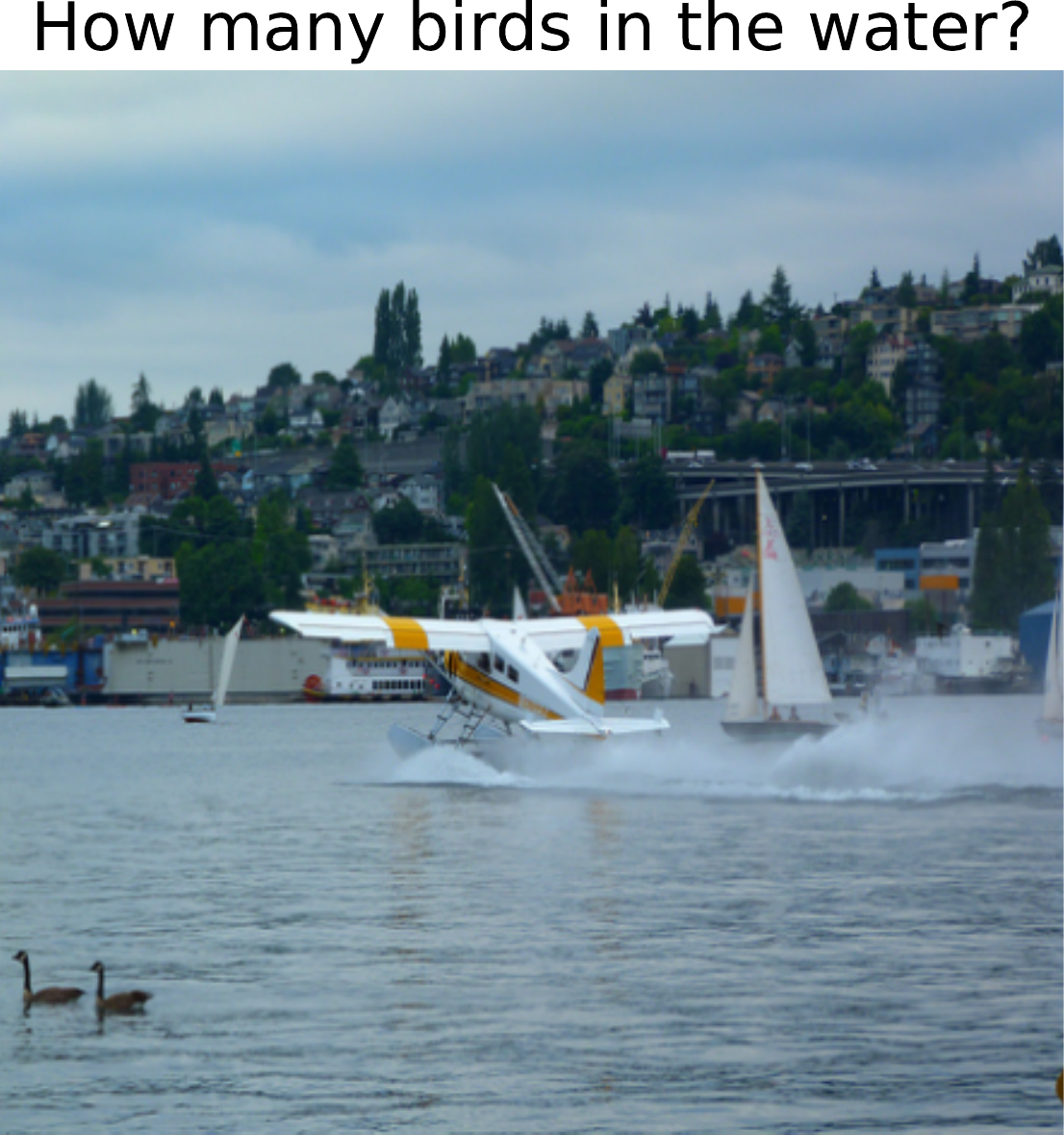}
\includegraphics[width=0.24\textwidth]{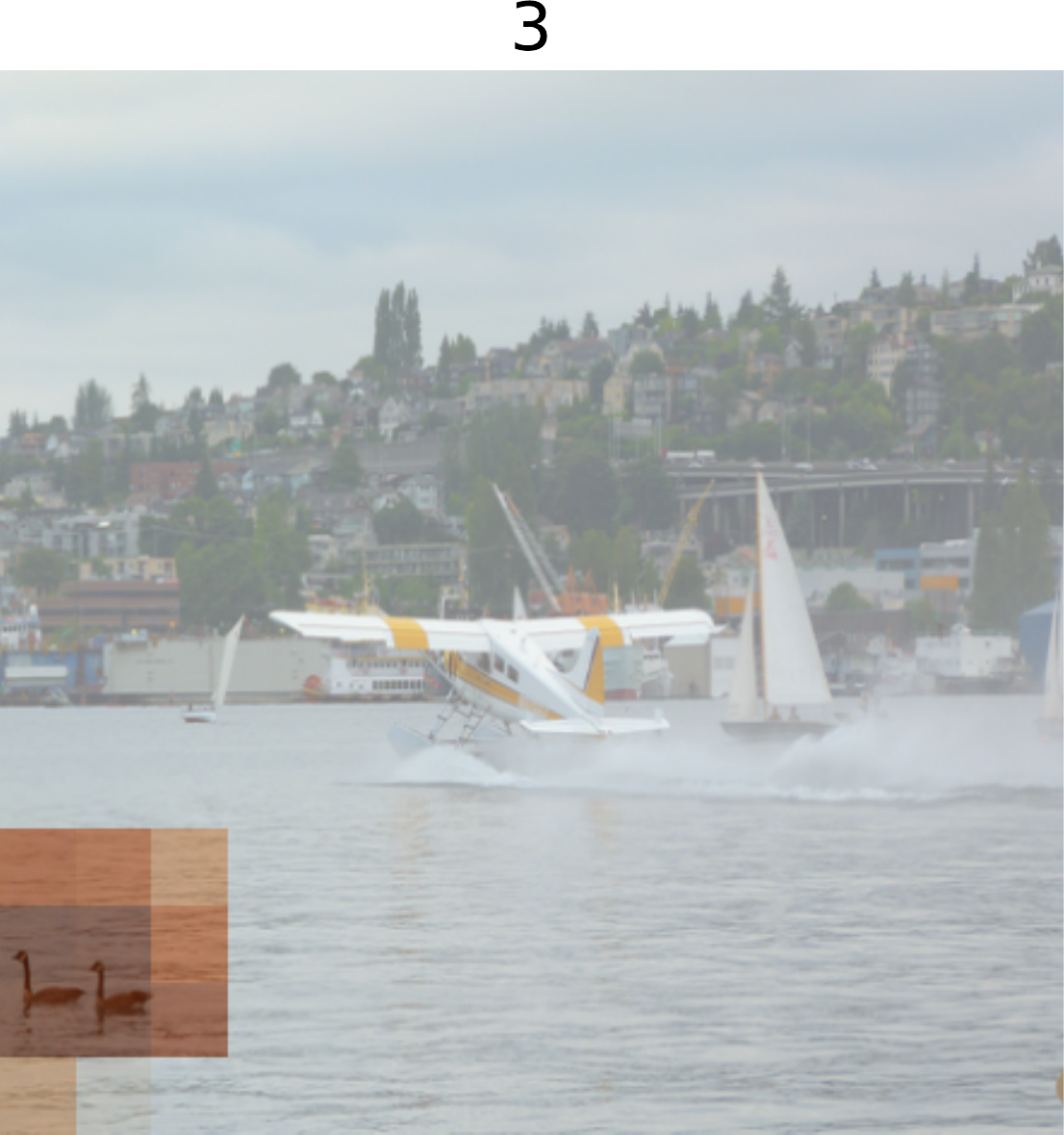}
\includegraphics[width=0.24\textwidth]{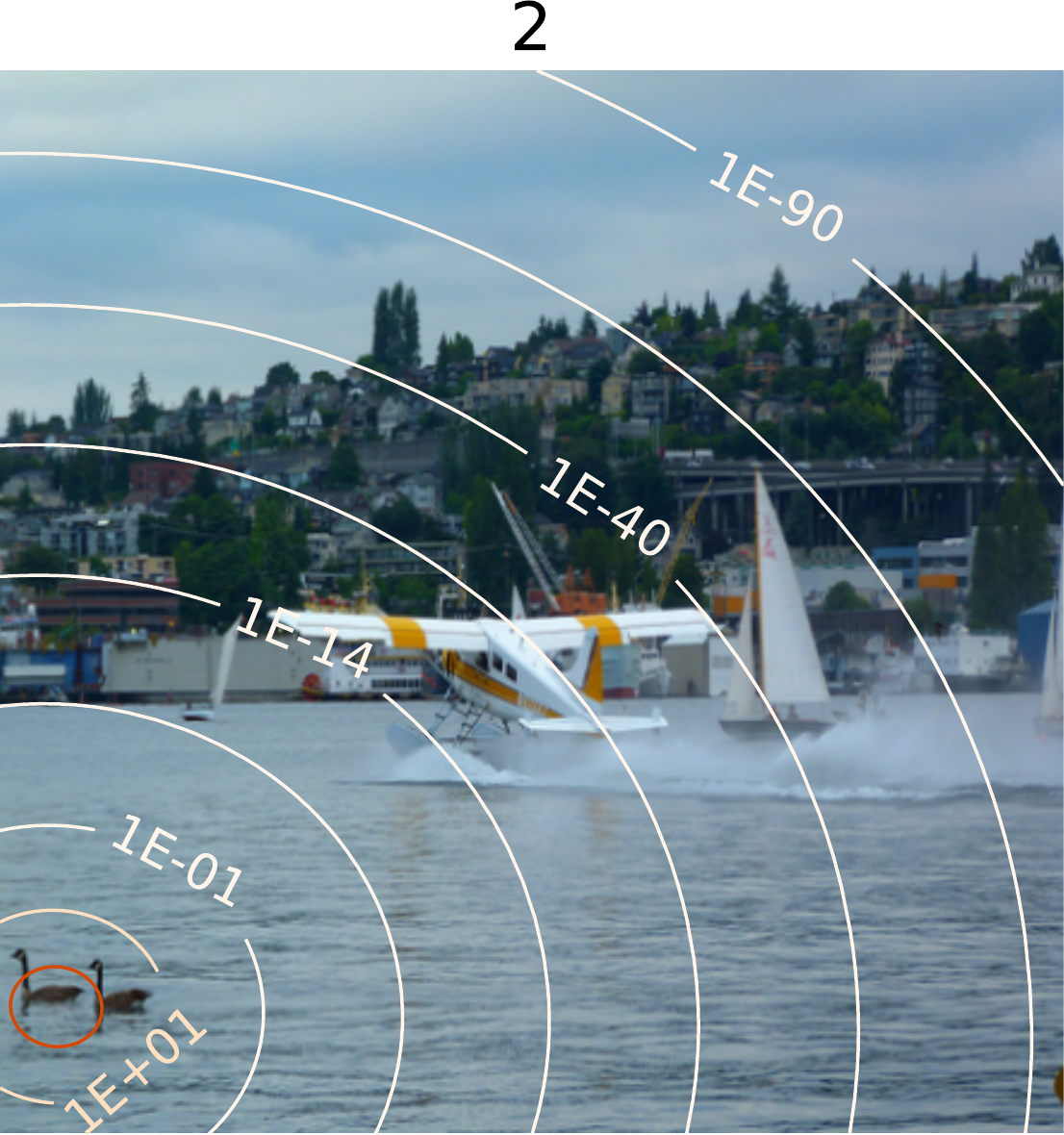}
\includegraphics[width=0.24\textwidth]{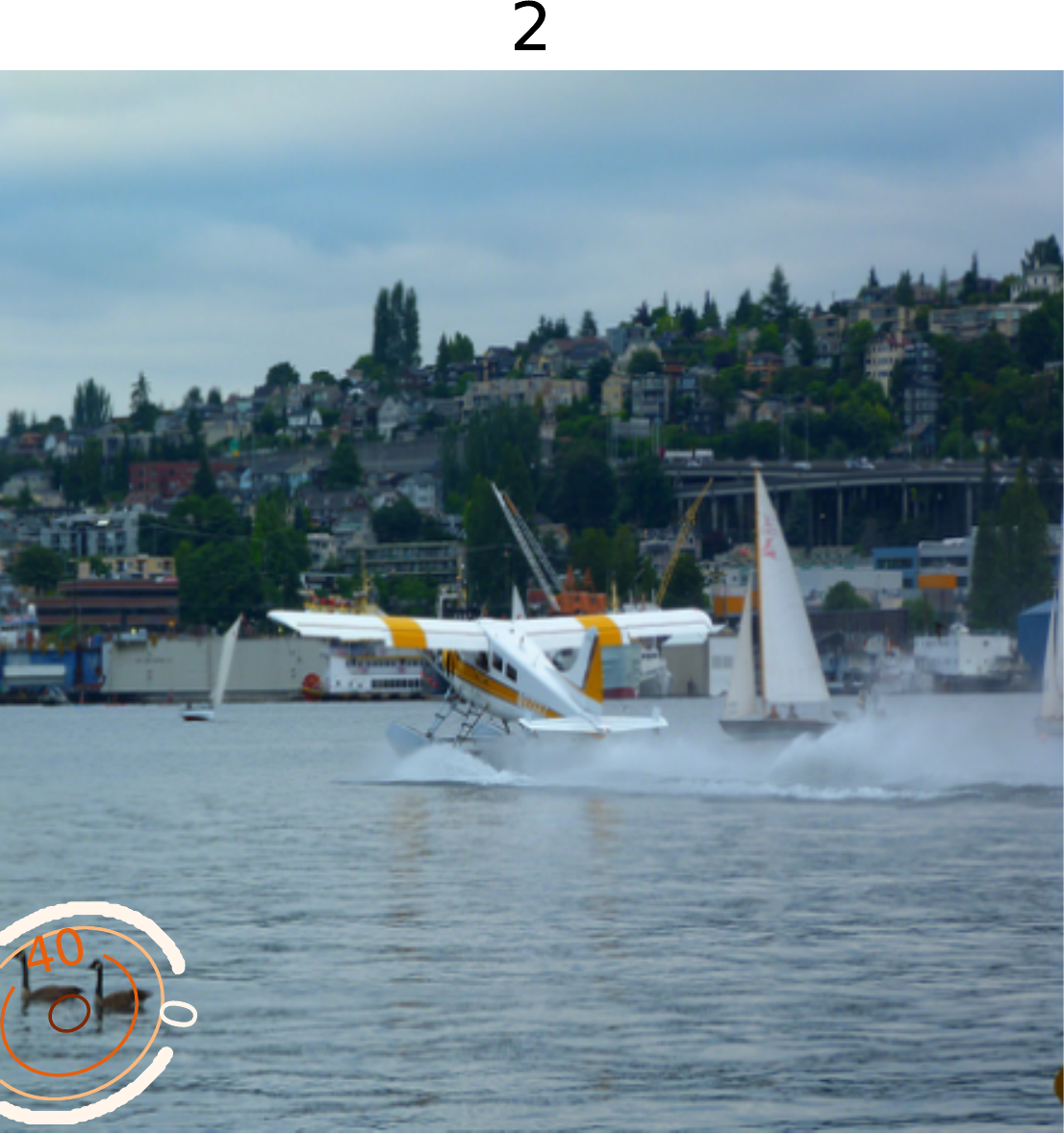}
\caption{\label{fig:examples_vqa_2birds}Attention maps for an example in VQA-v2: original image, discrete attention, continuous softmax, and continuous sparsemax.}
\end{figure*}

\begin{figure*}[t]
\centering
\includegraphics[width=0.24\textwidth]{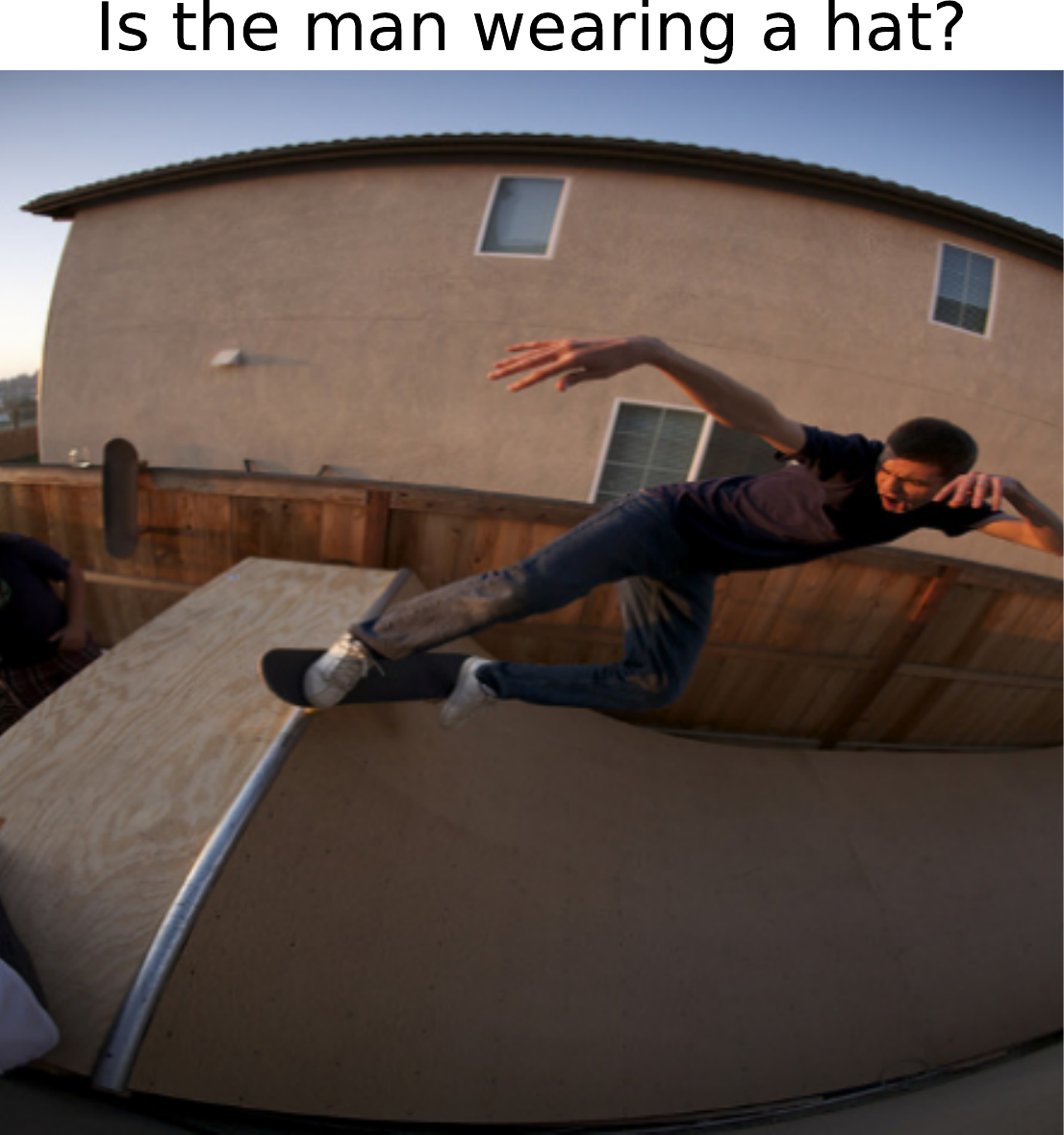}
\includegraphics[width=0.24\textwidth]{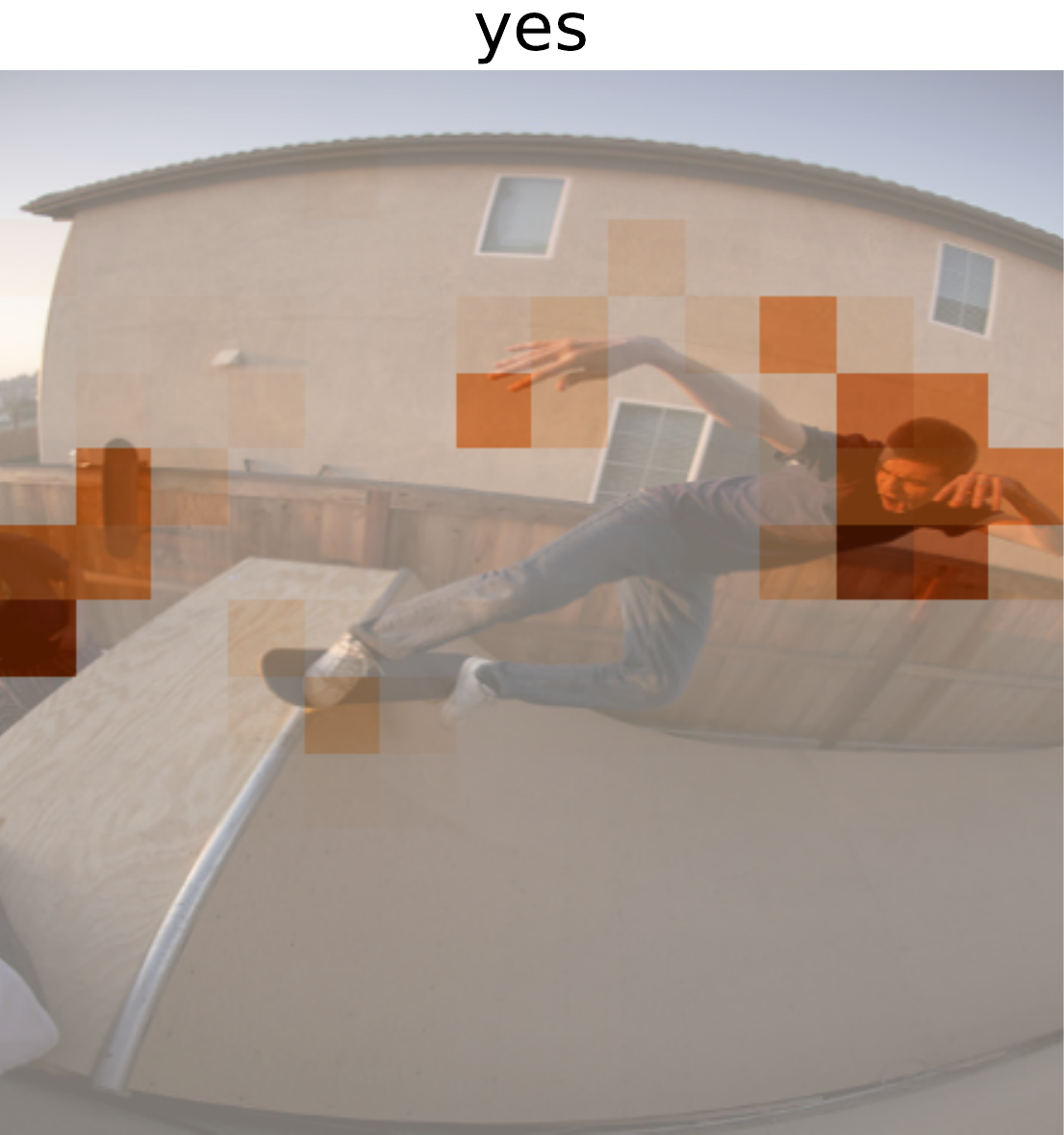}
\includegraphics[width=0.24\textwidth]{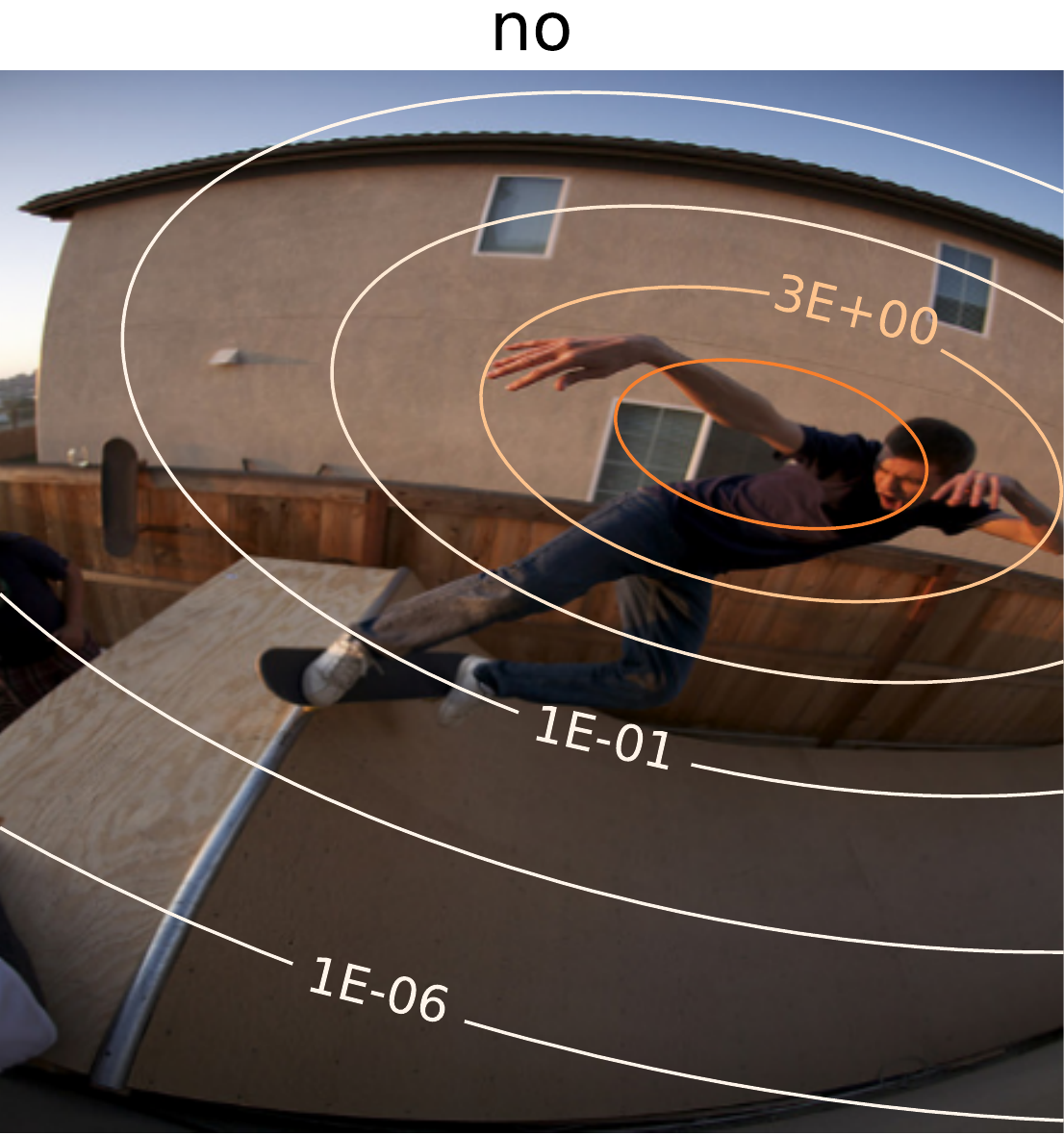}
\includegraphics[width=0.24\textwidth]{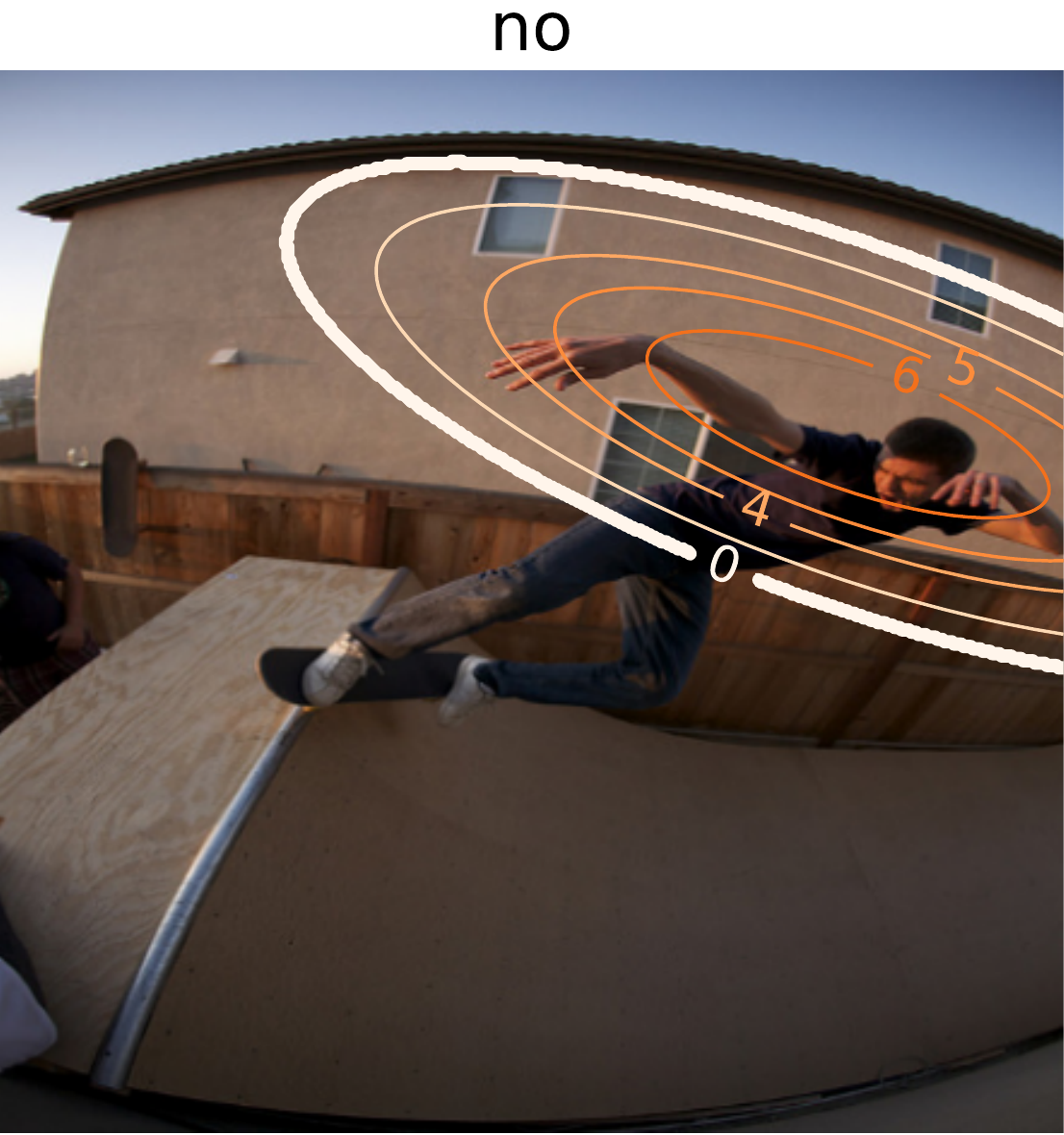}
\caption{\label{fig:examples_vqa_hat}Attention maps for an example in VQA-v2: original image, discrete attention, continuous softmax, and continuous sparsemax.}
\end{figure*}

\begin{figure*}[t]
\centering
\includegraphics[width=0.24\textwidth]{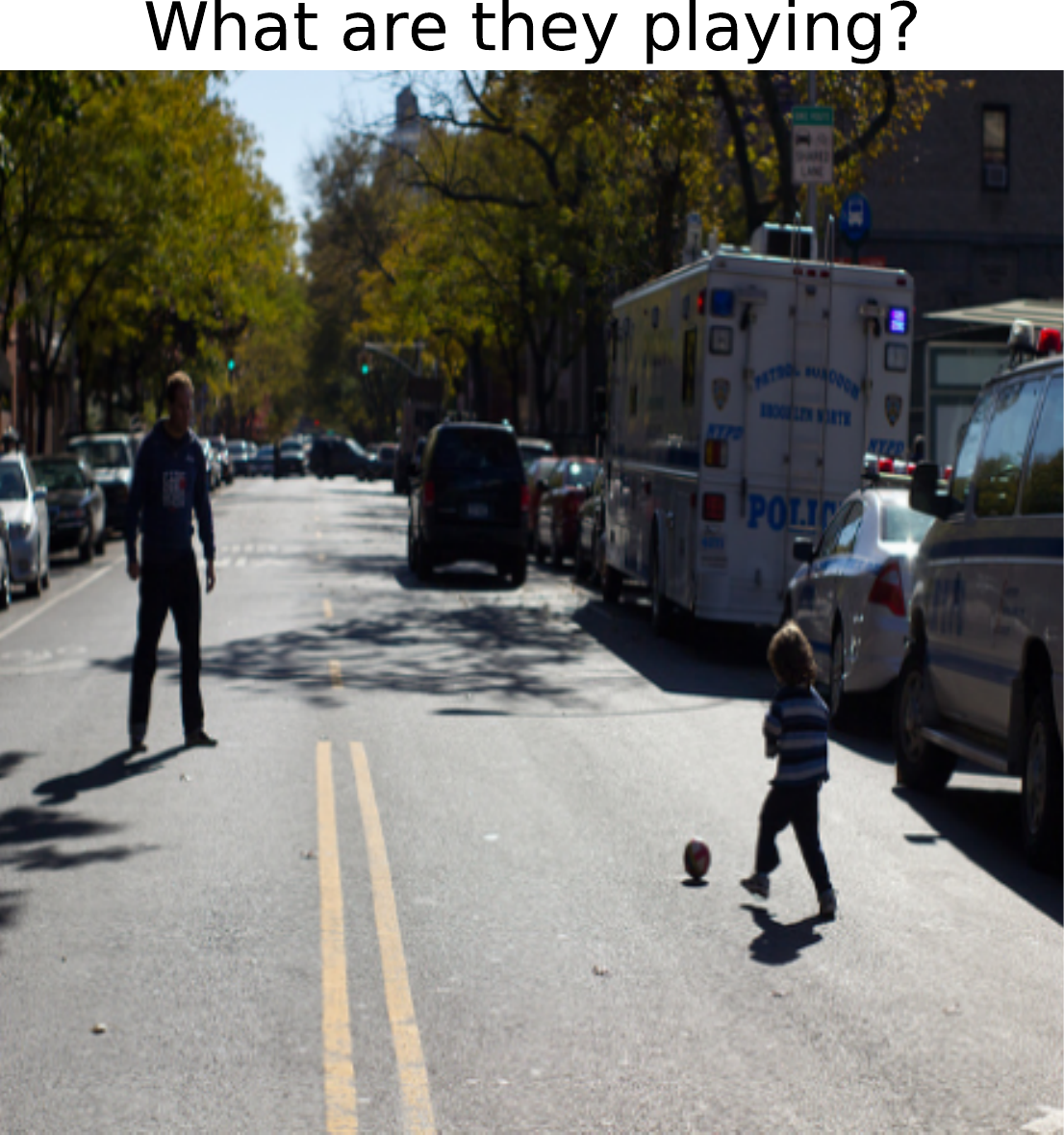}
\includegraphics[width=0.24\textwidth]{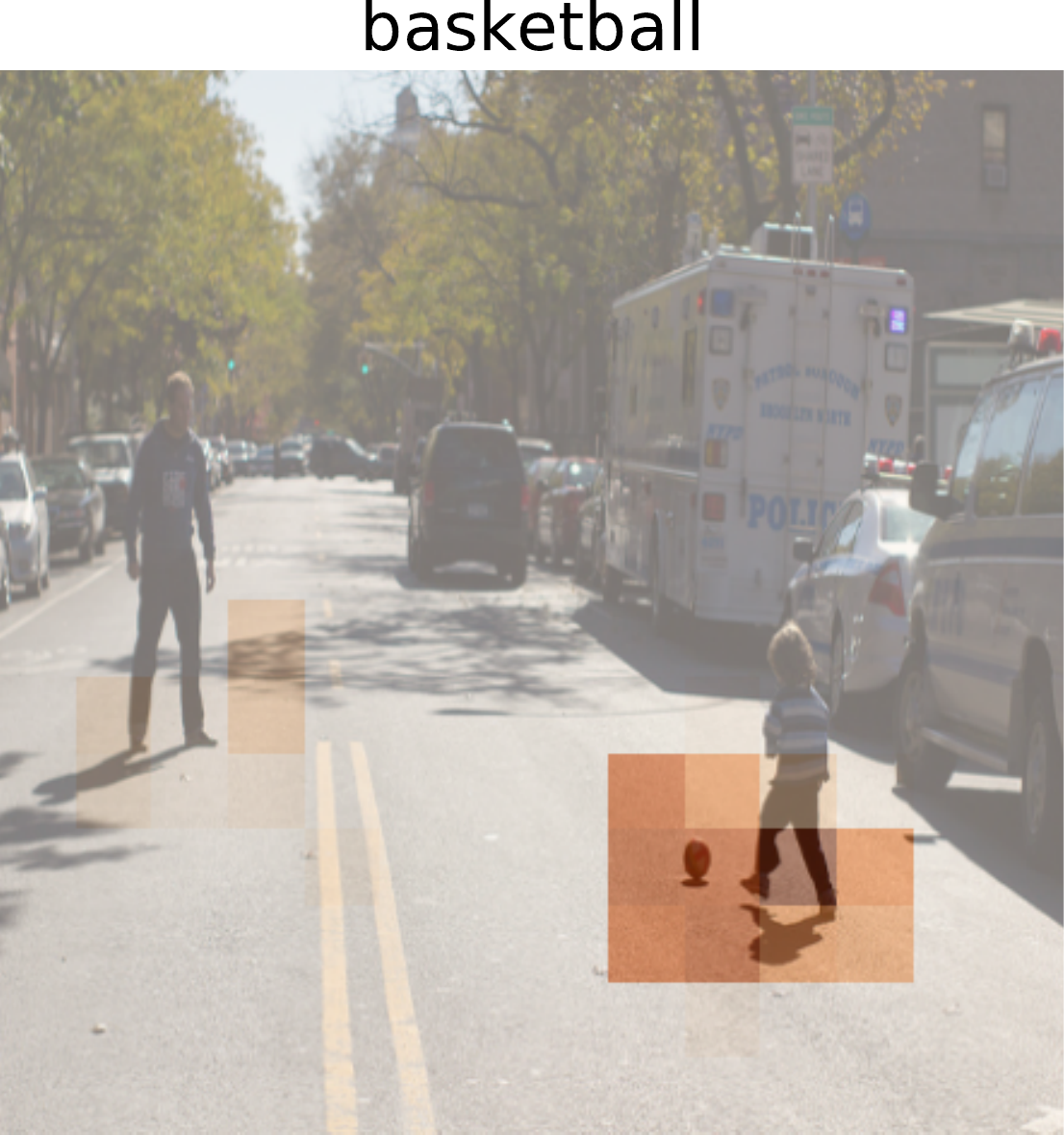}
\includegraphics[width=0.24\textwidth]{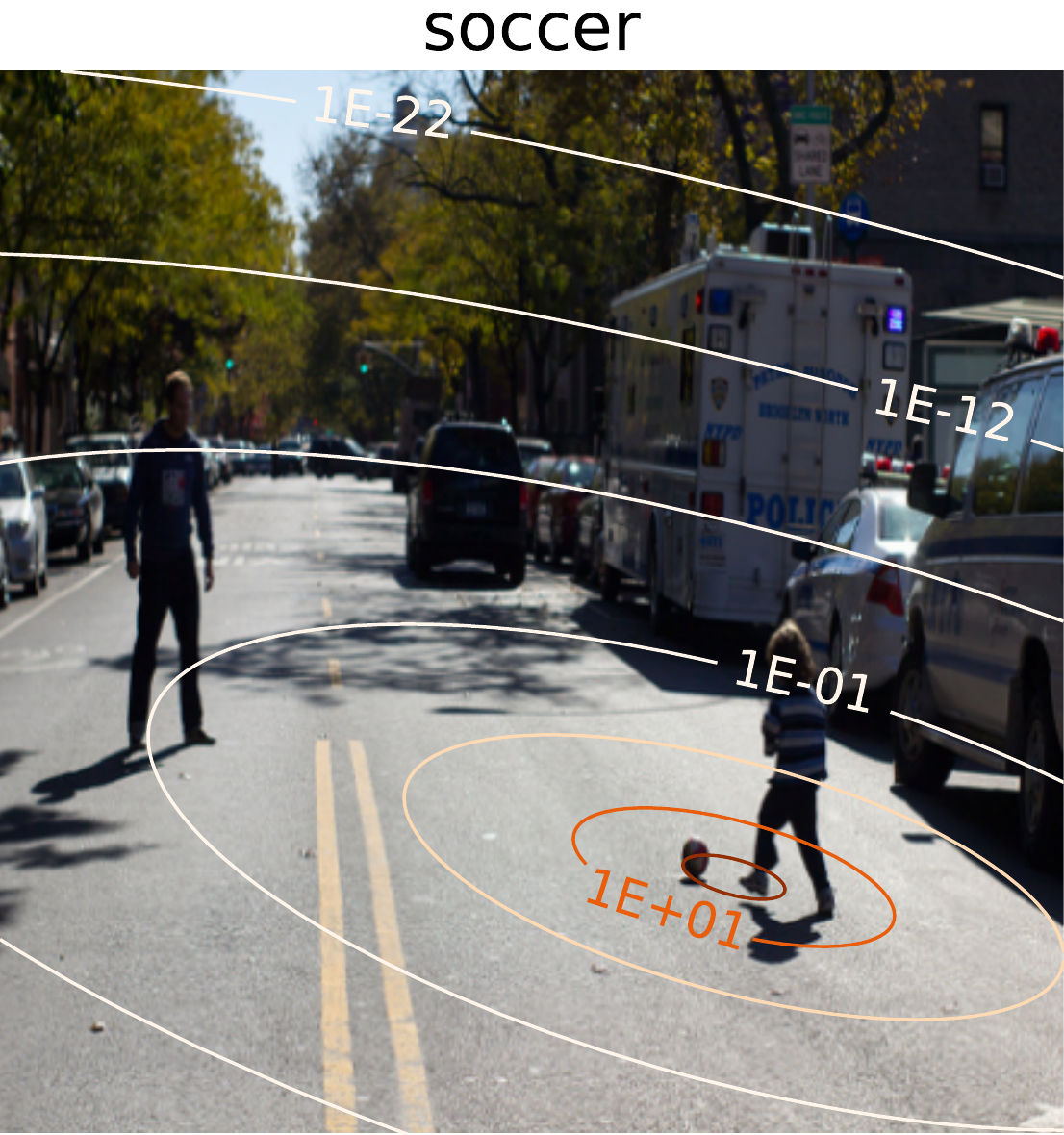}
\includegraphics[width=0.24\textwidth]{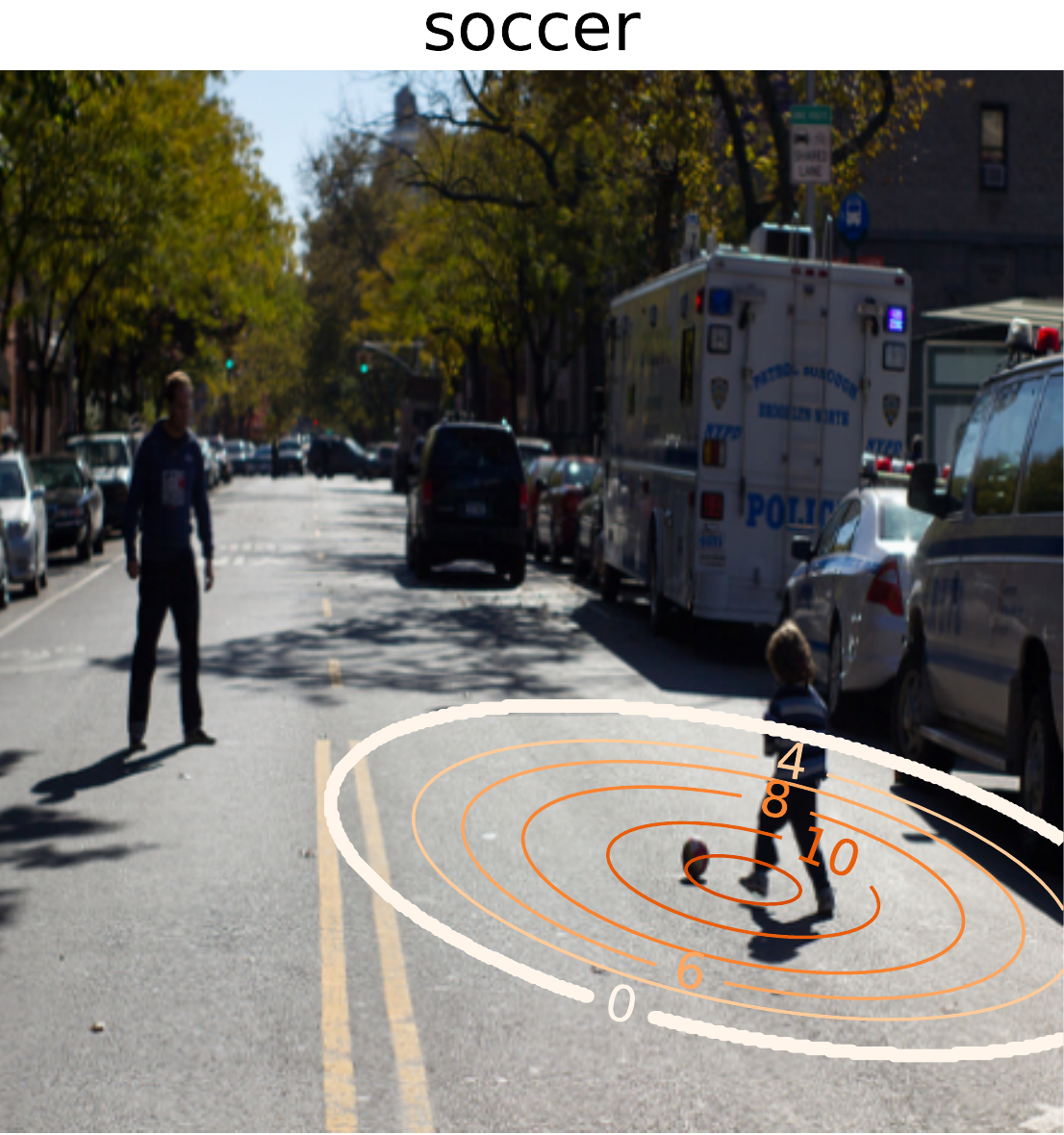}
\caption{\label{fig:examples_vqa_soccer}Attention maps for an example in VQA-v2: original image, discrete attention, continuous softmax, and continuous sparsemax.}
\end{figure*}

\remove{
\subsection{Interval regression.}

The dataset may be downloaded from
\url{https://www.ncdc.noaa.gov/crn/}.
\begin{table}[!htb]
    \caption{Hyperparmeters for interval regression.}
    \label{tab:table_reg_hyperparams}
    \vskip 0.15in
    \begin{small}
    \begin{center}
    \begin{tabular}{llllll}
        \toprule
        \sc Hyperparameter
& \sc MSE
& \sc Gaussian-$\sigma$
& \sc Gaussian-$95\%$
& \sc Truncated parabola
& \sc Triangular \\
\midrule
\emph{(tuned)} & & & & &\\
Learning rate &
0.002 &
0.001 &
0.0005 &
0.001 &
0.001 \\
Hidden layer size &
64 &
256 &
32 &
256 &
256 \\
Dropout probability &
0.2 &
0.1 &
0.2 &
0.1 &
0.05 \\
\emph{(fixed)} & & & & &\\
Optimizer & Adam & & & &\\
Batch& full dataset & & & &\\
N.~epochs& 50,000 & & & &\\
\bottomrule
    \end{tabular}
    \end{center}
    \end{small}
    \vskip -0.1in
\end{table}
}

\end{document}